\documentclass[twoside,11pt]{article}
\usepackage{jmlr2e}

\usepackage{epsf}
\usepackage{epsfig}
\usepackage{amsmath}
\usepackage{subfigure}
\usepackage{amssymb}
\usepackage{algorithm}
\usepackage{multirow}
\usepackage{multicol}
\usepackage{graphicx}
\usepackage{algorithm}
\usepackage{algorithmic}

\usepackage[pagebackref=true,breaklinks=true,letterpaper=true,colorlinks,bookmarks=yes]{hyperref}

\def\A{{\bf A}}
\def\Amp{{\A^{\dag}}}
\def\a{{\bf a}}
\def\B{{\bf B}}
\def\bb{{\bf b}}
\def\C{{\bf C}}
\def\Cmp{{\C^{\dag}}}

\def\D{{\bf D}}
\def\d{{\bf d}}

\def\F{{\bf F}}

\def\I{{\bf I}}

\def\M{{\bf M}}

\def\N{{\bf N}}

\def\Q{{\bf Q}}

\def\R{{\bf R}}
\def\S{{\bf S}}
\def\s{{\bf s}}
\def\T{{\bf T}}

\def\U{{\bf U}}
\def\u{{\bf u}}
\def\V{{\bf V}}
\def\v{{\bf v}}
\def\W{{\bf W}}
\def\w{{\bf w}}
\def\X{{\bf X}}
\def\x{{\bf x}}
\def\Y{{\bf Y}}

\def\0{{\bf 0}}
\def\1{{\bf 1}}

\def\UA{{{\U}_{\A}}}

\def\UAk{{{\U}_{\A, k}}}
\def\UAkp{{{\U}_{\A, k \perp}}}

\def\VA{{{\V}_{\A}}}

\def\VAk{{{\V}_{\A, k}}}

\def\SiA{{{\Si}_{\A}}}

\def\SiAk{{{\Si}_{\A, k}}}
\def\SiAkp{{{\Si}_{\A, k \perp}}}

\def\AM{{\mathcal A}}

\def\OM{{\mathcal O}}
\def\PM{{\mathcal P}}

\def\UM{{\mathcal U}}
\def\VM{{\mathcal V}}

\def\RB{{\mathbb R}}
\def\RBmn{{\RB^{m\times n}}}
\def\EB{{\mathbb E}}
\def\PB{{\mathbb P}}

\def\Pii{\mbox{\boldmath$\Pi$\unboldmath}}

\def\Si{\mbox{\boldmath$\Sigma$\unboldmath}}

\def\Lam{\mbox{\boldmath$\Lambda$\unboldmath}}

\def\spann{\mathrm{span}}
\def\tr{\mathrm{tr}}
\def\rk{\mathrm{rank}}

\def\Diag{\mathsf{Diag}}
\def\blkdiag{\mathsf{BlkDiag}}

\def\CUR{{CUR} }
\def\nystrom{{Nystr\"{o}m} }
\def\ALG{{\AM_{\textrm{col}}}}
\def\TimeMulti{{T_{\mathrm{Multiply}}}}

\jmlrheading{14}{2013}{2549-2589}{10/12; Revised 4/13}{8/13}{Shusen Wang and Zhihua Zhang}
\ShortHeadings{Improving CUR Matrix Decomposition and the Nystr\"{o}m Approximation}{Wang and Zhang}
\firstpageno{1}

\begin{document}
\title{Improving CUR Matrix Decomposition and the Nystr\"{o}m Approximation via Adaptive Sampling}

\author{\name  Shusen Wang  \email wss@zju.edu.cn \\
        \addr College of Computer Science and Technology \\
               Zhejiang University \\
               Hangzhou, Zhejiang 310027, China \\
        \AND
        \name Zhihua Zhang\thanks{Corresponding author.} \email zhihua@sjtu.edu.cn \\
        \addr Department of Computer Science and Engineering \\
                Shanghai Jiao Tong University \\
                800 Dong Chuan Road, Shanghai, China 200240
        }

\editor{Mehryar Mohri}

\maketitle

\begin{abstract}%
The CUR matrix decomposition and the \nystrom approximation are two important low-rank matrix approximation techniques.
The \nystrom method approximates a symmetric positive semidefinite matrix in terms of a small number of its columns,
while CUR approximates an arbitrary data matrix by a small number of its columns and rows. Thus, CUR decomposition
can be regarded as an extension of the \nystrom approximation.

In this paper we establish a more general error bound for the adaptive column/row sampling algorithm,
based on which we propose more accurate CUR and \nystrom algorithms with expected relative-error bounds.
The proposed CUR and \nystrom algorithms also have low time complexity and can avoid maintaining the whole data matrix in RAM.
In addition, we  give theoretical analysis for the lower error bounds of the standard \nystrom method and the ensemble \nystrom method.
The main theoretical results established in this paper are novel,
and our analysis makes no special assumption on the data matrices.
\end{abstract}


\section{Introduction}

Large-scale  matrices emerging from stocks, genomes, web documents, web images and videos
everyday bring new challenges in modern data analysis.
Most efforts have been focused on manipulating, understanding and interpreting large-scale data matrices.
In many cases, matrix factorization methods are employed for constructing parsimonious and
informative representations to facilitate computation and interpretation.
A principled approach is  the truncated singular value decomposition (SVD)
which finds the best low-rank approximation of a  data matrix.
Applications of SVD such as eigenfaces \citep{Sirovich87eigenface,turk1991eigenface} and
latent semantic analysis \citep{deerwester1990lsa} have been illustrated to be very successful.

However, using SVD to find basis vectors and low-rank approximations has its limitations.
As pointed out by \citet{berry2005algorithm},
it is often useful to find a low-rank matrix approximation which posses additional structures such as sparsity or nonnegativity.
Since SVD or the standard QR decomposition for sparse matrices does not preserve sparsity in general,
when the sparse matrix is large, computing or even storing such decompositions becomes challenging.
Therefore it is useful to compute a low-rank matrix decomposition which preserves such structural properties of the original data matrix.

Another limitation of SVD is that the  basis vectors resulting from SVD have little concrete meaning, which
makes it very difficult for us to understand and interpret the data in question.
An example of \citet{drineas2008cur} and \citet{mahoney2009matrix} has well shown this viewpoint; that is,
the vector $[(1/2)\textrm{age} - (1/\sqrt{2})\textrm{height} + (1/2)\textrm{income}]$,
the sum of the significant uncorrelated features from a data set of people's features,
is not particularly informative.
\citet{kuruvilla2002vector} have also claimed:
``it would be interesting to try to find basis vectors for all experiment vectors,
using actual experiment vectors and not artificial bases that offer little insight.''
Therefore, it is of great interest to represent a data matrix in terms of
a small number of actual columns and/or actual rows of the matrix.
{\it Matrix column selection} and the {\it  \CUR matrix decomposition}  provide such techniques.

\subsection{Matrix Column Selection} \label{sec:connection}

Column selection  has been extensively studied in the theoretical computer science (TCS)
and numerical linear algebra (NLA) communities.
The work in TCS mainly focuses on choosing good columns by randomized algorithms with provable error bounds
\citep{frieze2004fast,deshpande2006matrix,drineas2008cur,deshpande2010efficient,boutsidis2011NOC,Guruswami2012optimal}.
The focus in NLA is then on deterministic algorithms, especially the rank-revealing QR factorizations, that select columns by pivoting rules
\citep{foster1986rank,chan1987rank,stewart1999four,bischof1991structure,hong1992rank,chandrasekaran1994rank,gu1996efficient,berry2005algorithm}.
In this paper we focus on randomized algorithms for column selection.

Given a matrix $\A \in \RBmn$,
column selection algorithms aim to choose $c$
columns of $\A$ to construct a matrix $\C \in \RB^{m\times c}$
such that $\|\A - \C \Cmp \A \|_\xi$
achieves the minimum. Here ``$\xi=2$," ``$\xi=F$,"  and ``$\xi=*$'' respectively represent the matrix spectral norm, the matrix
Frobenius norm, and the matrix nuclear norm, and $\Cmp$ denotes the Moore-Penrose inverse of $\C$.
Since there are $(^n_c)$ possible choices of constructing $\C$,
selecting the best subset is a hard problem.

In recent years,
many polynomial-time approximate algorithms have been proposed.
Among them we are especially interested in those algorithms with {\it multiplicative upper bounds};
that is, there exists a polynomial function $f(m,n,k,c)$ such that
with $c$ $(\geq k)$ columns selected from $\A$ the following inequality holds
\[
\|\A - \C \Cmp \A \|_\xi \;\leq\; f(m,n,k,c) \, \| \A - \A_k \|_\xi
\]
with high probability (w.h.p.)\ or in expectation w.r.t.\ $\C$.
We call $f$ the {\it approximation factor}.
The bounds are strong when $f = 1+\epsilon$ for an error parameter $\epsilon$---they are known as {\it relative-error bounds}.
Particularly, the bounds are called {\it constant-factor bounds} when $f$ does not depend on $m$ and $n$ \citep{mahoney2011ramdomized}.
The relative-error bounds and constant-factor bounds of the CUR matrix decomposition and the \nystrom approximation are similarly defined.

However, the column selection method, also known as the $\A \approx \C \X$ decomposition in some applications, has its limitations.
For a large sparse matrix $\A$, its submatrix $\C$ is sparse,
but the coefficient matrix $\X \in \RB^{c\times n}$ is not sparse in general.
The $\C\X$ decomposition suffices when $m \gg n$, because  $\X$ is small in size.
However, when $m$ and $n$ are near equal, computing and storing the dense matrix $\X$ in RAM becomes infeasible.
In such an occasion the CUR matrix decomposition is a very useful alternative.

\subsection{The CUR Matrix Decomposition} \label{sec:introduction:cur}

The CUR matrix decomposition problem has been widely discussed in the literature
\citep{goreinov1997pseudoskeleton,goreinov1997maximalvolume,stewart1999four,tyrtyshnikov2000incompletecross,
berry2005algorithm,drineas2005nystrom,mahoney2008tensor,bien2010cur},
and it has been shown to be very useful in high dimensional data analysis.
Particularly,
a \CUR decomposition algorithm seeks to find a subset of $c$ columns of $\A$ to form a matrix $\C \in \RB^{m{\times} c}$,
a subset of $r$ rows to form a matrix $\R \in \RB^{r{\times} n}$, and an intersection matrix $\U \in \RB^{c{\times}r}$ such that $\|\A - \C \U \R\|_\xi$  is small.
Accordingly, we use $\tilde{\A} = \C \U \R$ to approximate $\A$.

\citet{drineas04fastmonte} proposed a \CUR algorithm with additive-error bound.
Later on, \citet{drineas2008cur} devised a randomized \CUR algorithm
which has relative-error bound w.h.p.\
if sufficiently many columns and rows are sampled.
\citet{mackey2011divide} established a divide-and-conquer method which solves the \CUR problem in parallel.
The \CUR algorithms guaranteed by relative-error bounds are of great interest.

Unfortunately, the existing \CUR algorithms usually require a large number of columns and rows to be chosen.
For example, for an $m{\times} n$ matrix $\A$ and a target rank $k \ll \min\{m,n\}$,
{\it the subspace sampling algorithm} \citep{drineas2008cur}---a classical \CUR algorithm---requires
$\OM(k \epsilon^{-2} \log k)$ columns and $\OM(k \epsilon^{-4} \log^2 k)$ rows to achieve relative-error bound w.h.p.
The subspace sampling algorithm selects columns/rows according to the statistical leverage scores,
so the computational cost of this algorithm is at least
equal to the cost of the truncated SVD of $\A$, that is, $\OM(m n k)$ in general.
However, maintaining a large scale matrix in RAM is often impractical,
not to mention performing SVD.
Recently, \cite{drineas2012fast} devised fast approximation to statistical leverage scores
which can be used to speedup the subspace sampling algorithm heuristically---yet no theoretical results have been reported that
the leverage scores approximation can give provably efficient subspace sampling algorithm.

The \CUR  matrix decomposition problem has a close connection with the column selection problem.
Especially, most \CUR algorithms such as those of \citet{drineas2003pass,drineas04fastmonte,drineas2008cur}
work in a two-stage manner where the first stage is a standard column selection procedure.
Despite their strong resemblance, CUR is a harder problem than column selection
because  ``one can get good columns or rows separately" does not mean that one can get good columns and rows together.
If the second stage is na\"ively solved by a column selection algorithm on $\A^T$,
then the approximation factor will trivially be $\sqrt{2}f$\footnote{It is because $\| \A - \C \U \R\|_F^2
= \|\A - \C \C^\dag \A + \C \C^\dag \A - \C \C^\dag \A \R^\dag \R \|_F^2
= \|(\I - \C \C^\dag) \A  \|_F^2 + \|\C \C^\dag (\A - \A \R^\dag \R) \|_F^2
\leq \|\A - \C \C^\dag \A  \|_F^2 + \|\A - \A \R^\dag \R \|_F^2
\leq 2 f^2 \| \A - \A_k \|_F^2$,
where the second equality follows from $(\I-\C\C^\dag)^T \C\C^\dag=0$.} \citep{mahoney2009matrix}.
Thus, more sophisticated error analysis techniques for the second stage
are indispensable in order to achieve relative-error bound.

\subsection{The \nystrom Methods} \label{sec:nystrom}

The \nystrom approximation is closely related to CUR,
and it can potentially benefit from the advances in CUR techniques.
Different from CUR, the \nystrom methods are used for approximating symmetric positive semidefinite (SPSD)\ matrices.
The  methods approximate an SPSD matrix only using a subset of its columns,
so they can alleviate computation and storage costs when the SPSD matrix in question is large in size.
In fact, the \nystrom methods have been extensively used in the machine learning community.
For example, they have been applied to Gaussian processes \citep{williams2001using},
kernel SVMs \citep{zhang2008improved}, spectral clustering \citep{fowlkes2004spectral},
kernel PCA \citep{talwalkar2008large,zhang2008improved,zhang2010clustered}, etc.

The \nystrom methods approximate any SPSD matrix in terms of a subset of its columns.
Specifically, given an $m{\times} m$ SPSD matrix ${\A}$, they require sampling
$c$ ($<m$) columns of ${\A}$ to construct an $m\times c$ matrix ${\C}$. Since
there exists an $m{\times} m$ permutation matrix $\Pii$ such that $\Pii{\C}$
consists of the first $c$ columns of $\Pii {\A} \Pii^T$, we always assume that ${\C}$
consists of the first $c$ columns of $\A$ without loss of generality. We
partition $\A$ and $\C$ as
\begin{equation} 
\A \;=\;  \begin{bmatrix}
             \W & \A_{2 1}^T \\
             \A_{2 1} & \A_{2 2}
           \end{bmatrix}
          \quad \textrm{ and }
\quad
\C \;=\; \begin{bmatrix}
             \W  \\
             \A_{2 1}
           \end{bmatrix}\textrm{,} \nonumber
\end{equation}
where $\W$ and $\A_{2 1}$ are of sizes $c\times c$ and $(m{-}c) \times c$, respectively.
There are three models which are defined as follows.
\begin{itemize}
\item   {\bf The Standard \nystrom Method}.
    The standard \nystrom approximation to $\A$ is
    \begin{equation} \label{eq:nystrom_approx}
    \tilde{\A}_c^{\textrm{nys}} \; = \; \C \W^\dag \C^T
    \; = \; \left[
        \begin{array}{cc}
          \W & \A_{2 1}^T \\
          \A_{2 1} & \A_{2 1} \W^\dag \A_{2 1}^T \\
        \end{array}
      \right]
      \textrm{.}
    \end{equation}
    Here $\W^\dag$ is  called the {\it intersection matrix}.
    The matrix $(\W_k)^\dag$, where $k\leq c$ and $\W_k$ is the best $k$-rank approximation to $\W$, is also used as an intersection matrix for constructing
    approximations with even lower rank.
    But using $\W^\dag$ results in a tighter approximation than using $(\W_k)^\dag$ usually.
\item   {\bf The Ensemble \nystrom Method} \citep{kumar2009ensemble}.
    It selects a collection of $t$ samples,
    each sample ${\C^{(i)}}$, ($i=1, \cdots, t$), containing $c$ columns of $\A$.
    Then the ensemble method combines the samples to construct an approximation in the form of
    \begin{equation} \label{eq:ensemble_nystrom_approx}
    \tilde{\A}_{t,c}^{\textrm{ens}}
    \; = \; \sum_{i=1}^t \mu^{(i)} {\C^{(i)}} {\W^{(i)}}^\dag {\C^{(i)}}^T \textrm{,}
    \end{equation}
    where $\mu^{(i)}$ are the weights of the samples.
    Typically, the ensemble \nystrom method seeks to find out the weights by minimizing
    $\|\A - \tilde{\A}_{t,c}^{\textrm{ens}} \|_F$ or $\|\A - \tilde{\A}_{t,c}^{\textrm{ens}} \|_2$.
    A simple but effective strategy is to set the weights as $\mu^{(1)}=\cdots=\mu^{(t)}=\frac{1}{t}$.
\item   {\bf The Modified \nystrom Method} (proposed in this paper).
    It is defined as
    \[
        \tilde{\A}_c^{\textrm{mod}} \;=\; \C \big( \C^\dag \A (\C^\dag)^T \big) \C^T .
    \]
    This model is not strictly the \nystrom method because it uses a quite different intersection matrix $\C^\dag \A (\C^\dag)^T$.
    It costs $\OM(m c^2)$ time to compute the Moore-Penrose inverse $\C^\dag$ and $m^2 c$ flops to compute matrix multiplications.
    The matrix multiplications can be executed very efficiently in multi-processor environment,
    so ideally computing the intersection matrix costs time only linear in $m$.
    This model is more accurate (which will be justified in Section~\ref{sec:improved_nystrom} and \ref{sec:lower_bounds})
    but more costly than the conventional ones, so there is a trade-off between time and accuracy when deciding which model to use.
\end{itemize}
Here and later, we call those which use intersection matrix $\W^\dag$ or $(\W_k)^\dag$ {\it the conventional \nystrom methods},
including the standard \nystrom and the ensemble Nystr\"om.

To generate effective approximations,
much work has been built on the upper error bounds of the sampling techniques for the \nystrom method.
Most of the work, for example, \citet{drineas2005nystrom}, \citet{li2010making}, \citet{kumar2009ensemble}, \citet{jin2011improved}, and \citet{kumar2012sampling},
studied the additive-error bound.
With assumptions on matrix coherence,
better additive-error bounds were obtained by \cite{talwalkar2010matrix}, \cite{jin2011improved}, and \cite{mackey2011divide}.
However, as stated by \citet{mahoney2011ramdomized},
additive-error bounds are less compelling than relative-error bounds.
In one recent work, \cite{gittens2013revisiting} provided a relative-error bound for the first time, where the bound is in nuclear norm.

However, the error bounds of the previous \nystrom methods are much weaker than those of the existing CUR algorithms,
especially the relative-error bounds in which we are more interested \citep{mahoney2011ramdomized}.
Actually, as will be proved in this paper, the lower error bounds of the standard \nystrom method and the ensemble \nystrom method
are even much worse than the upper bounds of some existing CUR algorithms.
This  motivates us to improve the \nystrom method by borrowing the techniques in CUR matrix decomposition.

\subsection{Contributions and Outline} \label{sec:summary_results}

The main technical contribution of this work is the adaptive sampling bound in Theorem~\ref{thm:adaptive_bound},
which is an extension of Theorem~2.1 of \citet{deshpande2006matrix}.
Theorem~2.1 of \citet{deshpande2006matrix} bounds the error incurred by projection onto column or row space,
while our Theorem~\ref{thm:adaptive_bound} bounds the error incurred by the projection simultaneously onto column space and row space.
We also show that Theorem~2.1 of \citet{deshpande2006matrix} can be regarded as a special case of Theorem~\ref{thm:adaptive_bound}.

More importantly, our adaptive sampling bound provides an approach for improving CUR and the \nystrom approximation:
no matter which relative-error column selection algorithm is employed,
Theorem~\ref{thm:adaptive_bound} ensures relative-error bounds for CUR and the \nystrom approximation.
We present the results in Corollary~\ref{cor:adaptive_improved}.

Based on the adaptive sampling bound in Theorem~\ref{thm:adaptive_bound} and its corollary~\ref{cor:adaptive_improved},
we provide a concrete \CUR algorithm which beats the best existing algorithm---the subspace sampling algorithm---both theoretically
and empirically.
The CUR algorithm is described in Algorithm~\ref{alg:fast_cur} and analyzed in Theorem~\ref{cor:fast_cur}.
In Table~\ref{tab:comparison} we present a comparison between our proposed CUR algorithm and the subspace sampling algorithm.
As we see, our algorithm requires much fewer columns and rows to achieve relative-error bound.
Our method is more scalable for it works on only a few columns or rows of the data matrix in question;
in contrast, the subspace sampling algorithm maintains the whole data matrix in RAM to implement SVD.

\begin{table}[t]\setlength{\tabcolsep}{0.3pt}
\begin{center}
\begin{footnotesize}
\begin{tabular}{c c c c c }
\hline
      	         &      ~~{\bf \#column ($c$)}~~   &	  ~~{\bf \#row ($r$)}~~   &		{\bf time}   &       {\bf space}	\\
\hline
Adaptive    &$\frac{2 k}{\epsilon}\big(1+o(1)\big)$&$\frac{c}{\epsilon}\big(1+\epsilon \big)$
                    & Roughly $\OM\big( n k^2 \epsilon^{-4} \big) + \TimeMulti\big(m n k \epsilon^{-1}\big)$~~  & $\OM\big(\max\{m c , n r\}\big) $   \\
Subspace    &$\OM\Big(\frac{k\log k}{\epsilon^2}\Big)$ & $\OM\Big(\frac{c\log c}{\epsilon^2}\Big)$
                    &  $\OM\big( m n k \big)$   &   $\OM(m n)$    \\
\hline
\end{tabular}
\end{footnotesize}
\end{center}
\caption{Comparisons between our {\it adaptive sampling} based CUR algorithm and
the best existing algorithm---the {\it subspace sampling} algorithm of \cite{drineas2008cur}.}
\label{tab:comparison}
\end{table}

Another important application of the adaptive sampling bound
is to yield an algorithm for the modified \nystrom method.
The algorithm has a strong relative-error upper bound: for a target rank $k$,
by sampling $\frac{2k}{\epsilon^2} \big(1 + o(1)\big)$ columns it achieves relative-error bound in expectation.
The results are shown in Theorem~\ref{thm:nystrom_bound}.

Finally, we establish a collection of lower error bounds of the standard \nystrom and the ensemble \nystrom that use $\W^\dag$ as the intersection matrix.
We show the lower bounds in Theorem~\ref{thm:lower_bounds_nystrom} and Table~\ref{tab:nystrom_lower_bound_detail};
here Table~\ref{tab:nystrom_lower_bound} briefly summarizes the lower bounds in Table~\ref{tab:nystrom_lower_bound_detail}.
From the table we can see that the upper error bound of our adaptive sampling algorithm for the modified \nystrom method is
even better than the lower bounds of the conventional \nystrom methods.\footnote{This can be valid because the lower bounds
in Table~\ref{tab:nystrom_lower_bound} do not hold when the intersection matrix is not $\W^\dag$.}

\begin{table}
\begin{center}
\begin{tabular}{c| c c c c c c}
\hline
    &$\frac{\|\A-\tilde{\A}\|_F}{\max_{i,j} |a_{i j}|}$& $\frac{\|\A-\tilde{\A}\|_2}{\max_{i,j} |a_{i j}|}$ &
     $\frac{\|\A-\tilde{\A}\|_*}{\max_{i,j} |a_{i j}|}$
        &$\frac{\|\A-\tilde{\A}\|_F}{\|\A-\A_k\|_F}$    & $\frac{\|\A-\tilde{\A}\|_2}{\|\A-\A_k\|_2}$
        & $\frac{\|\A-\tilde{\A}\|_*}{\|\A-\A_k\|_*}$ \\
\hline
Standard     &$\Omega\big( \frac{m\sqrt{k}}{c} \big)$
                    & $\Omega\big( \frac{{m}}{c} \big)$
                    & $\Omega \big( m-c \big)$
                    & $\Omega\Big(\sqrt{1+\frac{m k}{c^2}} \Big)$
                    & $\Omega\big( \frac{m}{c} \big)$
                    & $\Omega \big(1 + \frac{k}{c} \big) $ \\
Ensemble &$\Omega\big( \frac{{m\sqrt{k}}}{c} \big)$
                    & --
                    & $\Omega \big( m-c \big)$
                    &  $\Omega\Big(\sqrt{1+\frac{m k}{c^2}} \Big)$
                    & --
                    & $\Omega \big( 1 + \frac{k}{c} \big) $ \\
\hline
\end{tabular}
\end{center}
\caption{Lower bounds of the standard \nystrom method and the ensemble \nystrom method.
        The blanks indicate the lower bounds are unknown to us.
        Here $m$ denotes the column/row number of the SPSD matrix, $c$ denotes the number of selected columns, and $k$ denotes the target rank.}
        \label{tab:nystrom_lower_bound}
\end{table}

The remainder of the paper is organized as follows.
In Section~\ref{sec:notation} we give the notation  that will be used in this paper.
In Section~\ref{sec:related_work} we survey the previous work on
the randomized column selection, CUR matrix decomposition, and \nystrom approximation.
In Section~\ref{sec:fast_cur} we present our theoretical results and corresponding algorithms.
In Section~\ref{sec:experiments} we empirically evaluate our proposed CUR and \nystrom algorithms. Finally, we conclude
our work in Section~\ref{sec:concl}.
All proofs are deferred to the appendices.

\section{Notation}  \label{sec:notation}

First of all, we present the notation and notion that are used here and later.
We let $\I_m$ denote the $m\times m$ identity matrix,
$\1_m$ denote the $m{\times}1$  vector of ones,
and $\0$ denote a zero vector or matrix with appropriate size.
For a matrix $\A=[a_{ij}] \in \RB^{m\times n}$, we
let $\a^{(i)}$ be its $i$-th row, $\a_j$ be its $j$-th column,
and $\A_{i:j}$ be a submatrix consisting of its $i$ to $j$-th columns ($i\leq j$).

Let $\rho = \rk(\A) \leq \min\{m, n\}$ and $k \leq \rho$. The singular value decomposition (SVD) of $\A$ can be written as
\[
\A = \sum_{i=1}^{\rho} \sigma_{\A,i} \u_{\A,i} \v^T_{\A,i} = \UA \SiA \V_{\A}^T
= \left[
    \begin{array}{cc}
      \UAk & \UAkp \\
    \end{array}
  \right]
  \left[
    \begin{array}{cc}
      \SiAk & \0 \\
      \0 & \SiAkp \\
    \end{array}
  \right]
  \left[
    \begin{array}{c}
      \V_{\A, k}^T \\
      \V_{\A, k\perp}^T \\
    \end{array}
  \right]
  \textrm{,}
\]
where $\UAk$ ($m{\times}k$), $\SiAk$ ($k{\times}k$), and $\VAk$ ($n{\times}k$) correspond to the top $k$ singular values.
We denote $\A_k = \UAk \SiAk \V_{\A, k}^T$ which is the best (or closest) rank-$k$ approximation to $\A$.
We also use $\sigma_i (\A) = \sigma_{\A,i}$ to denote the $i$-th largest singular value.
When $\A$ is SPSD, the SVD is identical to the eigenvalue decomposition, in which case
we have $\UA = \VA$.

We define the matrix norms as follows.
Let $\|\A\|_1 = \sum_{i,j} |a_{i j}|$ be the $\ell_1$-norm,
$\|\A\|_F= (\sum_{i, j} a_{ij}^2)^{1/2}= (\sum_{i} \sigma^2_{\A,i})^{1/2}$ be the Frobenius norm,
$\|\A\|_{2} = \max_{\x \in \RB^{n},  \|\x\|_2=1} \|\A \x \|_2 = \sigma_{\A,1}$ be the spectral norm,
and $\|\A\|_* = \sum_i \sigma_{\A,i}$ be the nuclear norm.
We always use $\|\cdot\|_{\xi}$ to represent  $\|\cdot\|_{2}$, $\|\cdot\|_{F}$, or $\|\cdot\|_{*}$.

Based on SVD, the {\it statistical leverage scores} of the columns of $\A$ relative to the best rank-$k$ approximation to $\A$
is defined as
\begin{equation} \label{eq:leverage_scores}
\ell_j^{[k]} = \big\| \v_{\A,k}^{(j)} \big\|_2^2 \textrm{, } \quad j = 1 , \cdots , n \textrm{.}
\end{equation}
We have that $\sum_{j=1}^n \ell_j^{[k]} = k$.
The leverage scores of the rows of $\A$ are defined according to $\U_{\A,k}$.
The leverage scores play an important role in low-rank matrix approximation.
Informally speaking, the columns (or rows) with high leverage scores have greater influence in rank-$k$ approximation
than those with low leverage scores.

Additionally, let $\Amp = \V_{\A,\rho} \Si_{\A,\rho}^{-1} \U_{\A,\rho}^T$ be the Moore-Penrose inverse of $\A$ \citep{adi2003inverse}.
When $\A$ is nonsingular, the Moore-Penrose inverse is identical to the matrix inverse.
Given matrices $\A \in \RBmn$, $\X \in \RB^{m\times p}$, and $\Y \in \RB^{q\times n}$,
$\X \X^\dag \A = \U_\X \U_{\X}^T \A \in \RBmn$ is the projection of $\A$ onto the column space of $\X$,
and  $\A \Y^\dag \Y = \A \V_\Y \V_{\Y}^T \in \RBmn$ is the projection of $\A$ onto the row space of $\Y$.

Finally, we discuss the computational costs of the matrix operations mentioned above.
For an $m{\times} n$ general matrix $\A$ (assume $m \geq n$),
it takes $\OM(m n^2)$ flops to compute the full SVD
and $\OM(m n k)$ flops to compute the truncated SVD of rank $k$ ($< n$).
The computation of  $\A^\dag$ also takes $\OM(m n^2)$ flops.
It is worth mentioning that, although multiplying an $m{\times} n$ matrix by an $n{\times} p$ matrix runs in $m n p$ flops,
it can be easily performed in parallel \citep{halko2011ramdom}.
In contrast, implementing operations like SVD and QR decomposition in parallel is much more difficult.
So we denote the time complexity of such a matrix multiplication by $\TimeMulti(m n p)$,
which can be tremendously smaller than $\OM(m n p)$ in practice.

\section{Previous Work} \label{sec:related_work}

In Section~\ref{sec:adaptive_sampling_related} we present an adaptive sampling algorithm and
its relative-error bound established by \cite{deshpande2006matrix}.
In Section~\ref{sec:col_selection} we highlight the near-optimal column selection algorithm of \citet{boutsidis2011NOC}
which we will use in our CUR and \nystrom algorithms for column/row sampling.
In Section~\ref{sec:previous_work_cur} we introduce two important \CUR algorithms.
In Section~\ref{sec:previous_work_nystrom} we introduce the only known relative-error algorithm for the standard \nystrom method.

\subsection{The Adaptive Sampling Algorithm} \label{sec:adaptive_sampling_related}

Adaptive sampling is an effective and efficient column sampling algorithm for reducing the error incurred by the first round of sampling.
After one has selected a small subset of columns (denoted $\C_1$),
an adaptive sampling method is used to further select a proportion of columns according to the residual of the first round,
that is, $\A - \C_1 \C_1^\dag \A$.
The approximation error is guaranteed to be decreasing by a factor after the adaptive sampling \citep{deshpande2006matrix}.
We show the result of \citet{deshpande2006matrix} in the following lemma.

\begin{lemma}  [The Adaptive Sampling Algorithm] \emph{\citep{deshpande2006matrix}} \;
\label{lem:ada_sampling}
Given a matrix $\A \in \RBmn$, we let $\C_1\in \RB^{m\times c_1}$ consist of $c_1$ columns of $\A$,
and define the residual $\B = \A - \C_1 \C_1^\dag \A$. Additionally,
for $i = 1,\cdots, n$, we define
\[
p_i \;=\; \|\bb_i\|_2^2 / \|\B\|_F^2.
\]
We further  sample  $c_2$ columns  i.i.d.\ from $\A$,
 in each trial of which the $i$-th column is chosen with probability $p_i$.
Let $\C_2 \in \RB^{m\times c_2}$ contain the $c_2$ sampled columns
and let $\C = [\C_1,\C_2] \in \RB^{m\times (c_1+c_2)}$.
Then, for any integer $k > 0$, the following inequality holds:
\[
\EB \|\A - \C \Cmp \A\|_F^2 \; \leq \; \|\A - \A_k \|_F^2 + \frac{k}{c_2} \|\A - \C_1 \C_1^\dag \A \|_F^2,
\]
where the expectation is taken w.r.t.\ $\C_2$.
\end{lemma}

We will establish in Theorem~\ref{thm:adaptive_bound} a more general and more useful error bound for this adaptive sampling algorithm.
It can be shown that Lemma~\ref{lem:ada_sampling} is a special case of Theorem~\ref{thm:adaptive_bound}.

\subsection{The Near-Optimal Column Selection Algorithm} \label{sec:col_selection}

\citet{boutsidis2011NOC} proposed a relative-error column selection algorithm which requires only
$c = { 2k \epsilon^{-1}}(1 {+} o(1))$ columns get selected.
\citet{boutsidis2011NOC} also proved the lower bound of the column selection problem which shows that
no column selection algorithm can achieve relative-error bound by selecting less than $c = k\epsilon^{-1}$ columns.
Thus this algorithm is near optimal.
Though an optimal algorithm recently proposed by \citet{Guruswami2012optimal} attains the the lower bound,
this  algorithm is quite inefficient in comparison with the near-optimal algorithm.
So we prefer to use the near-optimal algorithm in our CUR and \nystrom algorithms for column/row sampling.

The near-optimal algorithm consists of three steps:
the approximate SVD via random projection \citep{boutsidis2011NOC,halko2011ramdom},
the dual set sparsification algorithm \citep{boutsidis2011NOC},
and the adaptive sampling algorithm \citep{deshpande2006matrix}.
We describe the near-optimal algorithm in Algorithm~\ref{alg:near_optimal_col} and present the theoretical analysis  in Lemma~\ref{prop:fast_column_selection}.

\begin{lemma}[The Near-Optimal Column Selection Algorithm] \label{prop:fast_column_selection}
Given a matrix $\A \in \RBmn$ of rank $\rho$,
a target rank $k$ $(2 \leq k < \rho)$,
and $0 < \epsilon < 1$.
Algorithm~\ref{alg:near_optimal_col} selects
\begin{equation}
c \;=\; \frac{2k}{\epsilon} \Big( 1 + o(1) \Big) \nonumber
\end{equation}
columns of $\A$ to form a matrix $\C \in \RB^{m \times c}$,
then the following inequality holds:
\begin{equation}
\EB	\|\A - \C \Cmp \A \|_F^2
\;\leq\; (1 + \epsilon) \, \|\A - \A_k \|_F^2 \textrm{,} \nonumber
\end{equation}
where the expectation is taken w.r.t.\ $\C$.
Furthermore, the matrix $\C$ can be obtained in
$\OM\big(m k^2 \epsilon^{-4/3} + n k^3 \epsilon^{-2/3}\big) + \TimeMulti \big(m n k \epsilon^{-2/3} \big)$ time.
\end{lemma}

\begin{algorithm}[tb]
   \caption{The Near-Optimal Column Selection Algorithm of \cite{boutsidis2011NOC}.}
   \label{alg:near_optimal_col}
\algsetup{indent=2em}
\begin{small}
\begin{algorithmic}[1]
   \STATE {\bf Input:} a real matrix $\A \in \RBmn$, target rank $k$, error parameter $\epsilon \in (0, 1]$,
   			target column number $c = \frac{2k}{\epsilon} \big(1+o(1)\big)$;
   \STATE Compute approximate truncated SVD via random projection such that $\A_k \approx \tilde{\U}_k \tilde{\Si}_k \tilde{\V}_k$;
   \STATE Construct $\UM \leftarrow$ columns of $(\A - \tilde{\U}_k \tilde{\Si}_k \tilde{\V}_k)$;
          $\quad \VM \leftarrow$ columns of $\tilde{\V}_k^T$;
   \STATE Compute $\s \leftarrow$ Dual Set Spectral-Frobenius Sparsification Algorithm ($\UM$, $\VM$, $c - 2k/\epsilon$);
   \STATE Construct $\C_1 \leftarrow \A \Diag(\s)$, and then delete the all-zero columns;
   \STATE Residual matrix $\D \leftarrow \A - \C_1 \C_1^\dag \A$;
   \STATE Compute sampling probabilities: $p_i = \|\d_i\|_2^2 / \|\D\|_F^2$, $i = 1, \cdots, n$;
   \STATE Sampling $c_2 = 2k/\epsilon$ columns from $\A$ with probability $\{p_1,\cdots,p_n\}$ to construct $\C_2$;
   \RETURN $\C = [\C_1 , \C_2]$.
\end{algorithmic}
\end{small}
\end{algorithm}

This algorithm has the merits of low time complexity and space complexity.
None of the three steps---the randomized SVD, the dual set sparsification algorithm,
and the adaptive sampling---requires loading the whole of $\A$ into RAM.
All of the three steps can work on only a small subset of the columns of $\A$.
Though a relative-error algorithm recently proposed by \citet{Guruswami2012optimal} requires even fewer columns,
it is less efficient than the near-optimal algorithm.

\subsection{Previous Work in \CUR Matrix Decomposition}\label{sec:previous_work_cur}

We introduce in this section two highly effective \CUR algorithms: one is deterministic and the other is randomized.

\subsubsection{The Sparse Column-Row Approximation (SCRA)}\label{sec:scra}

\citet{stewart1999four} 
proposed a  deterministic CUR algorithm and called it the sparse column-row approximation (SCRA).
SCRA is based on the truncated pivoted QR decomposition via a quasi Gram-Schmidt algorithm.
Given a matrix $\A \in \RB^{m\times n}$, the truncated pivoted QR decomposition procedure deterministically
finds a set of columns $\C\in \RB^{m\times c}$ by column pivoting,
whose span approximates the column space of $\A$,
and computes an upper triangular matrix $\T_\C \in \RB^{c\times c}$ that orthogonalizes those columns.
SCRA runs the same procedure again on $\A^T$ to select a set of rows $\R \in \RB^{r\times n}$
and computes the corresponding upper triangular matrix $\T_\R \in \RB^{r\times r}$.
Let $\C = \Q_\C \T_\C$ and $\R^T = \Q_\R \T_\R$ denote the resulting truncated pivoted QR decomposition.
The intersection matrix is computed by $\U = (\T_\C^T \T_\C)^{-1} \C^T \A \R^T (\T_\R^T \T_\R)^{-1}$.
According to our experiments, this algorithm is quite effective but very time expensive, especially when $c$ and $r$ are large.
Moreover, this algorithm does not have data-independent error bound.

\subsubsection{The Subspace Sampling \CUR Algorithm} \label{sec:stat_leverage}

\cite{drineas2008cur} proposed a two-stage randomized \CUR algorithm
which has a relative-error bound with high probability (w.h.p.).
In the first stage the algorithm samples $c$ columns of $\A$ to construct $\C$,
and in the second stage it samples $r$ rows from $\A$ and $\C$ simultaneously to construct $\R$ and $\W$ and let $\U = \W^\dag$.
The sampling probabilities in the two stages are proportional to the leverage scores of $\A$ and $\C$, respectively.
That is, in the first stage the sampling probabilities are proportional to the squared $\ell_2$-norm of the rows of $\V_{\A,k}$;
in the second stage the sampling probabilities are proportional to the squared $\ell_2$-norm of the rows of $\U_{\C}$.
That is why it is called the \emph{subspace sampling algorithm}.
Here we show the main results of the subspace sampling algorithm in the following lemma.

\begin{lemma}[Subspace Sampling for \CUR]
Given an ${m\times n}$ matrix $\A$ and a target rank $k \ll \min\{m,n\}$,
the subspace sampling algorithm selects $c = \OM(k \epsilon^{-2} \log k \log(1/\delta))$ columns
and $r =$ \linebreak[4] $\OM\big(c \epsilon^{-2} \log c \log(1/\delta)\big)$ rows without replacement.
Then
\begin{equation}
\|\A - \C \U \R\|_F
\; = \; \big\| \A - \C \W^\dag \R \big\|_F
\; \leq \; (1+\epsilon) \|\A - \A_k\|_F \textrm{,} \nonumber
\end{equation}
holds with probability at least $1-\delta$,
where $\W$ contains the rows of $\C$ with scaling.
The running time is dominated by the truncated SVD of $\A$, that is, $\OM(m n k)$.
\end{lemma}

\subsection{Previous Work in the \nystrom Approximation}\label{sec:previous_work_nystrom}

In a very recent work, \cite{gittens2013revisiting} established a framework for analyzing errors incurred by the standard \nystrom method.
Especially, the authors provided the first and the only known relative-error (in nuclear norm) algorithm for the standard \nystrom method.
The algorithm is described as follows and, its bound is shown in Lemma~\ref{lem:subspace_nystrom}.

Like the CUR algorithm in Section~\ref{sec:stat_leverage},
the \nystrom algorithm also samples columns by the subspace sampling of \cite{drineas2008cur}.
Each column is selected with probability $p_j = \frac{1}{k} \ell_j^{[k]}$ with replacement,
where $\ell_1^{[k]} , \cdots , \ell_m^{[k]}$ are leverage scores defined in (\ref{eq:leverage_scores}).
After column sampling, $\C$ and $\W$ are obtained by scaling the selected columns, that is,
\[
\C = \A (\S \D)
\quad \textrm{ and } \quad
\W = (\S \D)^T \A (\S \D).
\]
Here $\S \in \RB^{m\times c}$ is a column selection matrix that
$s_{i j} = 1$ if the $i$-th column of $\A$ is the $j$-th column selected,
and $\D\in \RB^{c\times c}$ is a diagonal scaling matrix satisfying
$d_{j j} = \frac{1}{\sqrt{c p_i}}$ if $s_{i j} = 1$.

\begin{lemma}[Subspace Sampling for the \nystrom Approximation] \label{lem:subspace_nystrom}
Given an $m\times m$ SPSD matrix $\A$ and a target rank $k \ll m$,
the subspace sampling algorithm selects
\[
c = 3200 \epsilon^{-1} k \log (16k/\delta)
\]
columns without replacement and constructs $\C$ and $\W$ by scaling the selected columns.
Then the inequality
\begin{equation}
\big\|\A - \C \W^\dag \C^T \big\|_*
\; \leq \; (1+\epsilon) \|\A - \A_k\|_* \textrm{,} \nonumber
\end{equation}
holds with probability at least $0.6-\delta$.
\end{lemma}

\section{Main Results} \label{sec:fast_cur}

We now present our main results.
We establish a new error bound for the adaptive sampling algorithm in Section~\ref{sec:adaptive_sampling}.
We apply adaptive sampling to the CUR and  modified \nystrom problems, obtaining effective and efficient CUR and \nystrom algorithms
in Section~\ref{sec:improved_cur} and Section~\ref{sec:improved_nystrom} respectively.
In Section~\ref{sec:lower_bounds} we study lower bounds of the conventional \nystrom methods
to demonstrate the advantages of our approach.
Finally, in Section~\ref{sec:whp_bound} we show that our expected bounds can extend to with high probability (w.h.p.) bounds.

\subsection{Adaptive Sampling} \label{sec:adaptive_sampling}

The relative-error adaptive sampling algorithm is originally established in Theorem~2.1 of \citet{deshpande2006matrix} (see also Lemma~\ref{lem:ada_sampling} in Section~\ref{sec:adaptive_sampling_related}).
The algorithm is based on the following idea:
after selecting a proportion of columns from $\A$ to form $\C_1$ by an arbitrary algorithm,
the algorithm randomly samples additional $c_2$ columns according to the residual $\A - \C_1 \C_1^\dag \A$.
Here we prove a new and more general error bound for the same adaptive sampling algorithm.

\begin{theorem} [The Adaptive Sampling Algorithm] \label{thm:adaptive_bound}
Given a matrix $\A \in \RBmn$ and a matrix $\C \in \RB^{m\times c}$
such that $\rk(\C) = \rk(\C \C^\dag \A) = \rho$ $(\rho \leq c \leq n)$.
We let $\R_1 \in \RB^{r_1 \times n}$ consist of $r_1$ rows of $\A$,
and define the residual $\B = \A - \A \R_1^\dag \R_1$. Additionally,
for $i = 1,\cdots, m$, we define
\[
p_i \;=\; \|\bb^{(i)}\|_2^2 / \|\B\|_F^2.
\]
We further  sample  $r_2$ rows  i.i.d.\ from $\A$,
 in each trial of which the $i$-th row is chosen with probability $p_i$.
Let $\R_2 \in \RB^{r_2\times n}$ contain the $r_2$ sampled rows
and let $\R = [\R_1^T,\R_2^T]^T \in \RB^{(r_1+r_2)\times n}$.
Then we have
\begin{equation}
\EB \|\A - \C \Cmp \A \R^\dag \R \|_F^2 \; \leq \; \|\A - \C \C^\dag \A \|_F^2 + \frac{\rho}{r_2} \|\A - \A \R_1^\dag \R_1\|_F^2 \textrm{,} \nonumber
\end{equation}
where the expectation is taken w.r.t.\ $\R_2$.
\end{theorem}

\begin{remark}
This theorem shows a more general bound for adaptive sampling than the original one in Theorem~2.1 of \citet{deshpande2006matrix}.
The original one bounds the error incurred by projection onto the column space of $\C$,
while Theorem~\ref{thm:adaptive_bound} bounds the error incurred by projection onto
the column space of $\C$ and row space of $\R$ simultaneously---such situation rises in problems such as CUR and the \nystrom approximation.
It is worth pointing out that Theorem~2.1 of \cite{deshpande2006matrix} is a direct corollary of this theorem
when $\C = \A_k$ (i.e., $c = n$, $\rho = k$, and $\C \C^\dag \A = \A_k$).
\end{remark}

As  discussed in Section~\ref{sec:introduction:cur},
selecting good columns or rows separately does not ensure good columns and rows together for CUR and the \nystrom approximation.
Theorem~\ref{thm:adaptive_bound} is thereby important
for it guarantees the combined effect column and row selection.
Guaranteed by Theorem~\ref{thm:adaptive_bound}, any column selection algorithm with
relative-error bound can be applied to CUR and the \nystrom approximation.
We show the result in the following corollary.

\begin{corollary}[Adaptive Sampling for CUR and the \nystrom Approximation] \label{cor:adaptive_improved}
Given a matrix $\A\in \RBmn$, a target rank $k$ $(\ll m,n)$, and a column selection algorithm $\ALG$
which achieves relative-error upper bound by selecting $c \geq C(k,\epsilon)$ columns.
Then we have the following results for CUR and the \nystrom approximation.
\begin{enumerate}
\item[\emph{(1)}]
By selecting $c \geq C(k,\epsilon)$ columns of $\A$ to construct $\C$
and $r_1 = c$ rows to construct $\R_1$, both using algorithm $\ALG$,
followed by selecting additional $r_2 = c/\epsilon$ rows using the adaptive sampling algorithm to construct $\R_2$,
the CUR matrix decomposition achieves relative-error upper bound in expectation:
\[
\EB \big\| \A - \C \U \R \big\|_F
\;\leq \; (1+\epsilon) \big\| \A - \A_k \big\|_F ,
\]
where $\R=\big[ \R_1^T , \R_2^T\big]^T$ and $\U = \C^\dag \A \R^\dag$.
\item[\emph{(2)}]
Suppose $\A$ is an $m\times m$ symmetric matrix.
By selecting $c_1 \geq C(k,\epsilon)$ columns of $\A$ to construct $\C_1$ using $\ALG$
and selecting $c_2 = c_1 / \epsilon$ columns of $\A$ to construct $\C_2$ using the adaptive sampling algorithm,
the modified \nystrom method achieves relative-error upper bound in expectation:
\[
\EB \big\| \A - \C \U \C^T \big\|_F
\;\leq \; (1+\epsilon) \big\| \A - \A_k \big\|_F,
\]
where $\C=\big[ \C_1  , \C_2 \big]$ and $\U = \C^\dag \A \big(\C^\dag\big)^T$.
\end{enumerate}
\end{corollary}

Based on Corollary~\ref{cor:adaptive_improved}, we attempt to solve CUR and the \nystrom by adaptive sampling algorithms.
We present concrete algorithms
in Section~\ref{sec:improved_cur} and \ref{sec:improved_nystrom}.

\subsection{Adaptive Sampling for \CUR Matrix Decomposition} \label{sec:improved_cur}

Guaranteed by the novel adaptive sampling bound in Theorem~\ref{thm:adaptive_bound},
we combine the near-optimal column selection algorithm of \cite{boutsidis2011NOC} and the adaptive sampling algorithm for solving the CUR problem,
giving rise to an algorithm with a much tighter theoretical bound than existing algorithms.
The algorithm is described in Algorithm~\ref{alg:fast_cur} and its analysis is given in Theorem~\ref{cor:fast_cur}.
Theorem~\ref{cor:fast_cur} follows immediately from Lemma~\ref{prop:fast_column_selection} and Corollary~\ref{cor:adaptive_improved}.

\begin{theorem} [Adaptive Sampling for CUR] \label{cor:fast_cur}
Given a matrix $\A \in \RBmn$ and a positive integer $k \ll \min\{m,n\}$,
the \CUR algorithm described in Algorithm~\ref{alg:fast_cur} randomly selects
$c = \frac{2k}{\epsilon} (1 {+} o(1))$ columns of $\A$ to construct $\C\in \RB^{m{\times} c}$,
and then selects $r=\frac{c}{\epsilon} (1 {+} \epsilon)$ rows of $\A$ to construct $\R \in \RB^{r{\times} n}$.
Then we have
\begin{equation}
\EB \| \A - \C \U \R \|_F
\; =\; \EB \| \A - \C (\C^\dag \A \R^\dag) \R \|_F
\; \leq \; (1+ \epsilon) \|\A - \A_k\|_F \textrm{.} \nonumber
\end{equation}
The algorithm costs time
$\OM \big( (m + n)k^3 \epsilon^{-2/3} + m k^2 \epsilon^{-2} + nk^2\epsilon^{-4} \big) + \TimeMulti\big( m n k \epsilon^{-1} \big)$
to compute matrices $\C$, $\U$ and $\R$.
\end{theorem}

When the algorithm is executed in a single-core processor,
the time complexity of the \CUR algorithm is linear in $m n$;
when executed in multi-processor environment where matrix multiplication is performed in parallel,
ideally the algorithm costs time only linear in $m{+}n$.
Another advantage of this algorithm is that it avoids loading the whole $m{\times} n$ data matrix $\A$ into RAM.
Neither the near-optimal column selection algorithm nor the adaptive sampling algorithm requires loading the whole of $\A$ into RAM.
The most space-expensive operation throughout this algorithm is  computation of the Moore-Penrose inverses of $\C$ and $\R$,
which requires maintaining an $m{\times} c$ matrix or an $r {\times} n$ matrix in RAM.
To compute the intersection matrix $\C^\dag \A \R^\dag$, the algorithm needs to visit each entry of $\A$,
but it is not RAM expensive because the multiplication can be done by computing $\C^\dag \a_j$ for $j=1,\cdots,n$ separately.
The above analysis is also valid for the \nystrom algorithm in Theorem~\ref{thm:nystrom_bound}.

\begin{remark}
If we replace the near-optimal column selection algorithm in Theorem~\ref{cor:fast_cur} by the optimal algorithm of \cite{Guruswami2012optimal}, it suffices to select  $c=k \epsilon^{-1} (1+o(1))$ columns and $r = c \epsilon^{-1} (1+\epsilon)$ rows totally.
But the optimal algorithm is less efficient than the near-optimal algorithm.
\end{remark}

\begin{algorithm}[tb]
   \caption{Adaptive Sampling for CUR.}
   \label{alg:fast_cur}
\algsetup{indent=2em}
\begin{small}
\begin{algorithmic}[1]
   \STATE {\bf Input:} a real matrix $\A \in \RBmn$, target rank $k$, $\epsilon \in (0, 1]$,
   			target column number $c = \frac{2k}{\epsilon} \big(1+o(1)\big)$, target row number $r = \frac{c}{\epsilon} (1+\epsilon)$;
   \STATE Select $c= \frac{2k}{\epsilon} \big(1+o(1)\big)$ columns of $\A$ to construct $\C \in \RB^{m\times c}$ using Algorithm~\ref{alg:near_optimal_col};
   \STATE Select $r_1= c$ rows of $\A$ to construct $\R_1 \in \RB^{r_1\times n}$ using Algorithm~\ref{alg:near_optimal_col};
   \STATE Adaptively sample $r_2 = c/\epsilon$ rows from $\A$ according to the residual $\A - \A \R_1^\dag \R_1$;
   \RETURN $\C$, $\R = [\R^T_1 , \R_2^T]^T$, and $\U = \C^\dag \A \R^\dag$.
\end{algorithmic}
\end{small}
\end{algorithm}

\subsection{Adaptive Sampling for the \nystrom Approximation} \label{sec:improved_nystrom}

Theorem~\ref{thm:adaptive_bound} provides an approach for bounding the approximation errors
incurred by projection simultaneously onto column space and row space.
Thus this approach can be applied to solve the modified \nystrom method.
The following theorem follows directly from Lemma~\ref{prop:fast_column_selection} and Corollary~\ref{cor:adaptive_improved}.

\begin{theorem}[Adaptive Sampling for the Modified \nystrom Method] \label{thm:nystrom_bound}
Given a symmetric matrix $\A \in \RB^{m\times m}$ and a target rank $k$,
with $c_1=\frac{2k}{\epsilon} \big(1 + o(1)\big)$ columns sampled by Algorithm~\ref{alg:near_optimal_col}
and $c_2 = c_1/\epsilon$ columns sampled by the adaptive sampling algorithm,
that is, with totally $c =\frac{2k}{\epsilon^2} \big(1 + o(1)\big)$ columns being sampled,
the approximation error incurred by the modified \nystrom method is upper bounded by
\begin{equation}
\EB \big\|\A - \C \U \C^T \big\|_F
\;\leq \; \EB \Big\|\A - \C \Big( \Cmp \A (\C^\dag)^T \Big) \C^T \Big\|_F
\;\leq\; (1+\epsilon) \|\A - \A_k \|_F \textrm{.} \nonumber
\end{equation}
The algorithm costs time
$\OM\big( m k^2 \epsilon^{-4} + m k^3 \epsilon^{-2/3} \big) + \TimeMulti\big( m^2 k \epsilon^{-2} \big)$ in
computing $\C$ and $\U$.
\end{theorem}

\begin{remark}
The error bound in Theorem~\ref{thm:nystrom_bound} is the only Frobenius norm relative-error bound
for the \nystrom approximation at present,
and it is also a constant-factor bound.
If one uses the optimal column selection algorithm of \cite{Guruswami2012optimal}, which is less efficient,
the error bound is further improved: only $c = \frac{k}{\epsilon^2} (1 + o(1) )$ columns are required.
Furthermore, the theorem requires the matrix $\A$ to be symmetric, which is milder than the
SPSD requirement made in the previous work.
\end{remark}

This is yet the strongest result for the \nystrom approximation problem---much stronger than the best
possible algorithms for the conventional \nystrom method.
We will illustrate this point by revealing the lower error bounds of the conventional \nystrom methods.

\begin{table}
\begin{center}
\begin{small}
\begin{tabular}{c| c c c}
\hline
                        &$\frac{\|\A-\tilde{\A}\|_F}{\max_{i,j} |a_{i j}|}$& $\frac{\|\A-\tilde{\A}\|_2}{\max_{i,j} |a_{i j}|}$ &   $\frac{\|\A-\tilde{\A}\|_*}{\max_{i,j} |a_{i j}|}$\\
\hline
{\it Standard}&$0.99 \sqrt{m-c-k + k \big(\frac{m+99k}{c+99k}\big)^2 } $
                        & $\frac{0.99(m+99)}{c+99}$
                        & $0.99(m-c)\big(1+\frac{k}{c+99k}\big)$ \\
{\it Ensemble}&$0.99 \sqrt{(m-2c + \frac{c}{t} -k)  + k \big( \frac{ m - c + \frac{c}{t} + 99 k }{ c + 99k } \big)^2 }$
                        & --
                        & $0.99(m-c)\big(1+\frac{k}{c+99k}\big)$  \\
\hline
\end{tabular}

\vspace{2mm}

\begin{tabular}{c| c c c}
\hline
                    &$\frac{\|\A-\tilde{\A}\|_F}{\|\A-\A_k\|_F}$
                    & $\frac{\|\A-\tilde{\A}\|_2}{\|\A-\A_k\|_2}$
                    & $\frac{\|\A-\tilde{\A}\|_*}{\|\A-\A_k\|_*}$ \\
\hline
{\it Standard}&$\sqrt{ 1 + \frac{ m^2 k - c^3 }{c^2 (m-k)} }$
                    & ${\frac{m}{c}}$
                    & ${\frac{m - c}{m-k} \big( 1 + \frac{k}{c} \big)} $\\
{\it Ensemble}& $\sqrt{\frac{m - 2c + {c}/{t} - k}{m-k} \Big( 1+ \frac{k(m - 2c + c/t)}{c^2}  \Big)}$
                    & --
                    & ${\frac{m - c}{m-k} \big( 1 + \frac{k}{c} \big)} $ \\
\hline
\end{tabular}
\end{small}
\end{center}
\caption{Lower bounds of the standard \nystrom method and the ensemble \nystrom method.
        The blanks indicate the lower bounds are unknown to us.
        Here $m$ denotes the column/row number of the SPSD matrix, $c$ denotes the number of selected columns, and $k$ denotes the target rank.}
        \label{tab:nystrom_lower_bound_detail}
\end{table}

\subsection{Lower Error Bounds of the Conventional \nystrom Methods} \label{sec:lower_bounds}

We now demonstrate to what an extent our modified \nystrom method is superior over the conventional \nystrom methods
(namely the standard \nystrom defined in (\ref{eq:nystrom_approx}) and the ensemble Nystr\"{o}m in (\ref{eq:ensemble_nystrom_approx}))
by showing the lower error bounds of the conventional \nystrom methods.
The conventional \nystrom methods work no better than the lower error bounds unless additional assumptions are made on the original matrix $\A$.
We show in Theorem~\ref{thm:lower_bounds_nystrom} the lower error bounds of the conventional \nystrom methods;
the results are briefly summarized previously in Table~\ref{tab:nystrom_lower_bound}.

To derive lower error bounds, we construct two adversarial cases for the \nystrom methods.
To derive the spectral norm lower bounds, we use an SPSD matrix $\B$ whose diagonal entries equal to $1$ and off-diagonal entries equal to $\alpha \in [0,1)$.
For the Frobenius norm and nuclear norm bounds,
we construct an $m\times m$ block diagonal matrix $\A$ which has $k$ diagonal blocks,
each of which  is $\frac{m}{k} \times \frac{m}{k}$ in size and constructed in the same way as $\B$.
For the lower bounds on $\frac{\|\A-\tilde{\A}\|_\xi}{\max_{i,j} |a_{i j}|}$, $\alpha$ is set to be constant;
for the bounds on $\frac{\|\A-\tilde{\A}\|_\xi}{\|\A-\A_k\|_\xi}$, $\alpha$ is set to be $\alpha \rightarrow 1$.
The detailed proof of Theorem~\ref{thm:lower_bounds_nystrom} is deferred to Appendix~\ref{sec:proof_lower_bounds}.

\begin{theorem}[Lower Error Bounds of the \nystrom Methods] \label{thm:lower_bounds_nystrom}
Assume we are given an SPSD matrix $\A\in \RB^{m\times m}$ and a target rank $k$.
Let $\A_k$ denote the best rank-$k$ approximation to $\A$.
Let $\tilde{\A}$ denote either the rank-$c$ approximation to $\A$ constructed by the standard \nystrom method in (\ref{eq:nystrom_approx}),
or the approximation constructed by the ensemble \nystrom method in (\ref{eq:ensemble_nystrom_approx}) with $t$ non-overlapping samples, each of which contains $c$ columns of $\A$.
Then there exists an SPSD matrix such that
for any sampling strategy the approximation errors of the conventional \nystrom methods,
that is,  $\|\A - \tilde{\A}\|_\xi$, \emph{($\xi = 2$, $F$, or ``$*$")},
are lower bounded by some factors which are shown in Table~\ref{tab:nystrom_lower_bound_detail}.
\end{theorem}

\begin{remark}
The lower bounds in Table~\ref{tab:nystrom_lower_bound_detail} (or Table~\ref{tab:nystrom_lower_bound}) show
the conventional \nystrom methods
can be sometimes very ineffective. The spectral norm and Frobenius norm bounds even depend on $m$,
so such bounds are not constant-factor bounds.
Notice that the lower error bounds do not meet if $\W^\dag$ is replaced by $\C^\dag \A (\C^\dag)^T$,
so our modified \nystrom method is not limited by such lower bounds.
\end{remark}

\subsection{Discussions of the Expected Relative-Error Bounds} \label{sec:whp_bound}

The upper error bounds established in this paper all hold in expectation.
Now we show that the expected error bounds immediately extend to w.h.p.\ bounds using Markov's inequality.
Let the random variable $X = \|\A - \tilde{\A} \|_F / \|\A - \A_k\|_F$ denote the error ratio, where
\[
\tilde{\A} = \C \U \R \; \textrm{ or } \; \C \U \C^T .
\]
Then we have $\EB (X) \leq 1+\epsilon$ by the preceding theorems.
By applying Markov's inequality we have that
\[
\PB \big(X > 1+s\epsilon \big)
\;<\; \frac{\EB (X)}{1+ s \epsilon}
\;<\; \frac{1+ \epsilon}{1+ s \epsilon},
\]
where $s$ is an arbitrary constant greater than $1$.
Repeating the sampling procedure for $t$ times and letting $X_{(i)}$ correspond to the error ratio of the $i$-th sample,
we obtain an upper bound on the failure probability:
\begin{equation} \label{eq:whp_bound1}
\PB \Big( \min_i \{X_{(i)}\} > 1+s\epsilon \Big)
\; = \; \PB \Big( X_{(i)} > 1+s\epsilon \; \forall i = 1, \cdots, t \Big)
\; < \; \Big(\frac{1+ \epsilon}{1+ s \epsilon} \Big)^t
\; \triangleq \; \delta ,
\end{equation}
which decays exponentially with $t$.
Therefore, by repeating the sampling procedure multiple times and choosing the best sample,
our CUR and \nystrom algorithms are also guaranteed with w.h.p.\ relative-error bounds.
It follows directly from (\ref{eq:whp_bound1}) that, by repeating the sampling procedure for
\[
t \;\geq \; \frac{1+\epsilon}{(s-1)\epsilon} \log \Big(\frac{1}{\delta}\Big)
\]
times, the inequality
\[
\| \A - \tilde{\A} \|_F \;\leq \; (1+ s \epsilon) \: \|\A - \A_k\|_F
\]
holds with probability at least $1-\delta$.

For instance, we let $s = 1 + \log ({1}/{\delta})$, then by repeating the sampling procedure for $t \geq 1 + 1/{\epsilon} $ times, the inequality
\[
\| \A - \tilde{\A} \|_F \;\leq \; \Big(1+ \epsilon + \epsilon \log ({1}/{\delta}) \Big) \: \|\A - \A_k\|_F
\]
holds with probability at least $1-\delta$.

For another instance, we let $s=2$, then by repeating the sampling procedure for $t \geq (1 + 1/{\epsilon}) \log (1/\delta) $ times, the inequality
\[
\| \A - \tilde{\A} \|_F \;\leq \; (1+ 2 \epsilon ) \: \|\A - \A_k\|_F
\]
holds with probability at least $1-\delta$.

\section{Empirical Analysis} \label{sec:experiments}

In Section~\ref{sec:experiments_cur} we empirical evaluate our \CUR algorithms in comparison
with the algorithms introduced in Section~\ref{sec:previous_work_cur}.
In Section~\ref{sec:experiments_nystrom} we conduct empirical comparisons between
the standard \nystrom and our modified Nystr\"{o}m,
and comparisons among three sampling algorithms.
We report the approximation error incurred by each algorithm on each data set.
The error ratio is defined by
\[
\textrm{Error Ratio}
\;=\;  \frac{\| \A - \tilde{\A}\|_F}{\|\A - \A_k\|_F}\textrm{,}
\]
where $\tilde{\A}=\C \U \R$ for the CUR matrix decomposition,
$\tilde{\A}= \C \W^\dag \C^T$ for the standard \nystrom method, and
$\tilde{\A}= \C \big(\C^\dag \A (\C^\dag)^T\big) \C^T$ for the modified \nystrom method.

We conduct experiments on a workstation with two Intel Xeon $2.40$GHz CPUs,
$24$GB RAM, and $64$bit Windows Server 2008 system.
We implement  the algorithms in MATLAB R2011b,
and  use the MATLAB function `$\mathrm{svds}$' for truncated SVD.
To compare the running time, all the computations are carried out in a single thread by
setting `$\mathrm{maxNumCompThreads(1)}$' in MATLAB.

\subsection{Comparison among the CUR Algorithms} \label{sec:experiments_cur}

In this section we empirically compare our adaptive sampling based CUR algorithm (Algorithm~\ref{alg:fast_cur})
with the subspace sampling algorithm of \citet{drineas2008cur}
and the deterministic sparse column-row approximation (SCRA) algorithm of \citet{stewart1999four}.
For SCRA, we use the MATLAB code released by \citet{stewart1999four}.
As for the subspace sampling algorithm, we compute the leverages scores exactly via the truncated SVD.
Although the fast approximation to leverage scores \citep{drineas2012fast} can significantly speedup subspace sampling,
we do not use it because the approximation has no theoretical guarantee when applied to subspace sampling.

\begin{table}[!ht]\setlength{\tabcolsep}{0.3pt}
\begin{center}
\begin{footnotesize}
\begin{tabular}{c c c c c c}
\hline
	{\bf Data Set}	& 	 {\bf Type}	  &			{\bf ~~Size}     &	{\bf \#Nonzero Entries} & {\bf Source}	\\
\hline
    Enron Emails    &   ~~text~~  &   ~~$39,861\times 28,102$~~   &  ~~$3,710,420$ & Bag-of-words, UCI \\
	Dexter			&	~text~    &   $20,000\times 2,600$	    &	~$248,616$	 & \cite{guyon2004result}\\
    Farm Ads		&	   ~text~       &   $54,877\times 4,143$	    &	$821,284$    & \cite{mesterharm2011active}	\\
    Gisette		    &handwritten digit~~&   $13,500\times 5,000$	    &	$8,770,559$	 & \cite{guyon2004result}\\
\hline
\end{tabular}
\end{footnotesize}
\end{center}
\caption{A summary of the data sets for CUR matrix decomposition.}
\label{tab:datasets}
\end{table}

\begin{figure*}
\subfigtopskip = 0pt
\begin{center}
\centering
{\includegraphics[width=60mm,height=40mm]{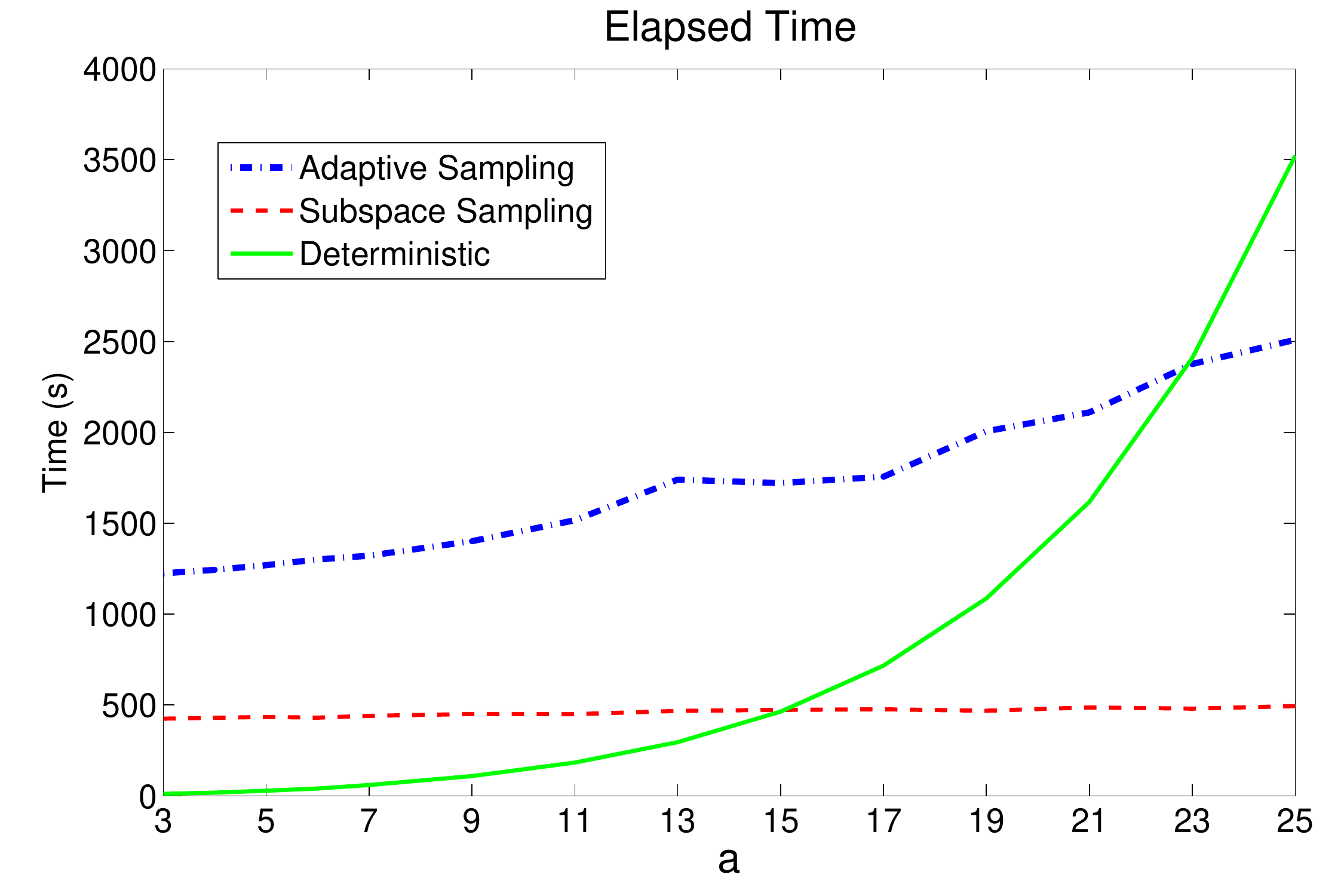}}~
{\includegraphics[width=60mm,height=40mm]{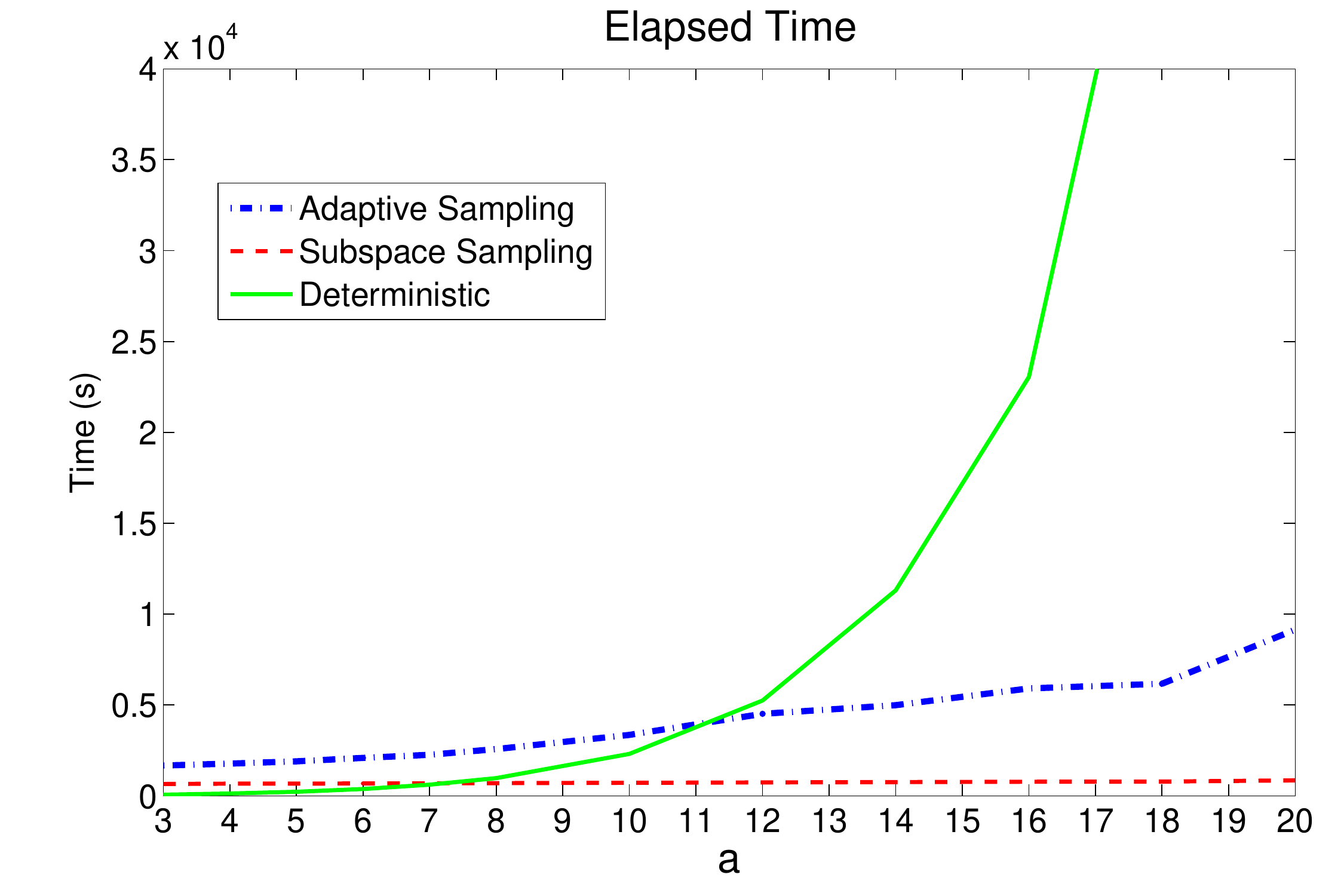}}\\
\subfigure[\textsf{$k = 10$, $c=a k$, and $r=a c$.}]{\includegraphics[width=60mm,height=40mm]{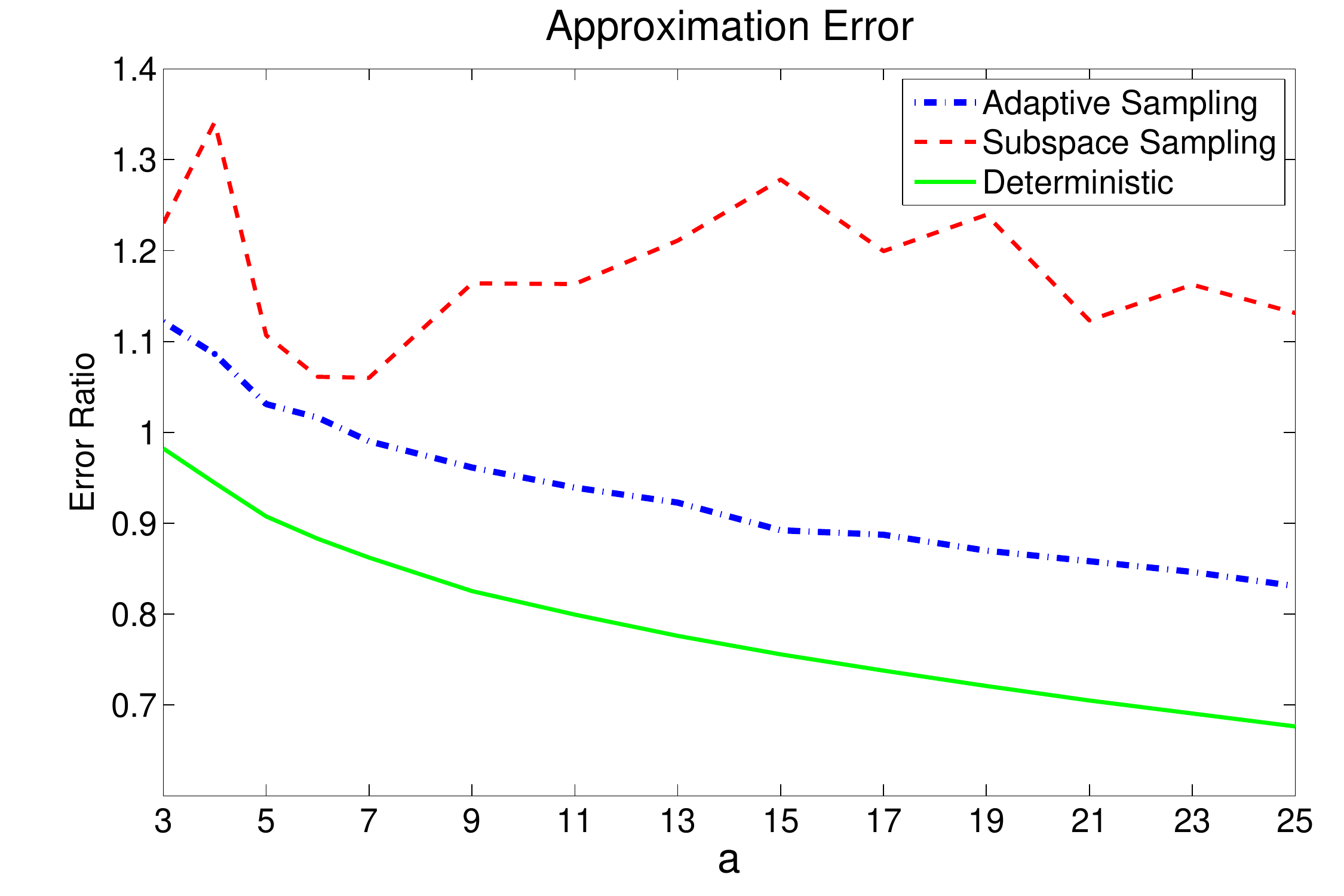}}~
\subfigure[\textsf{$k = 50$, $c=a k$, and $r=a c$.}]{\includegraphics[width=60mm,height=40mm]{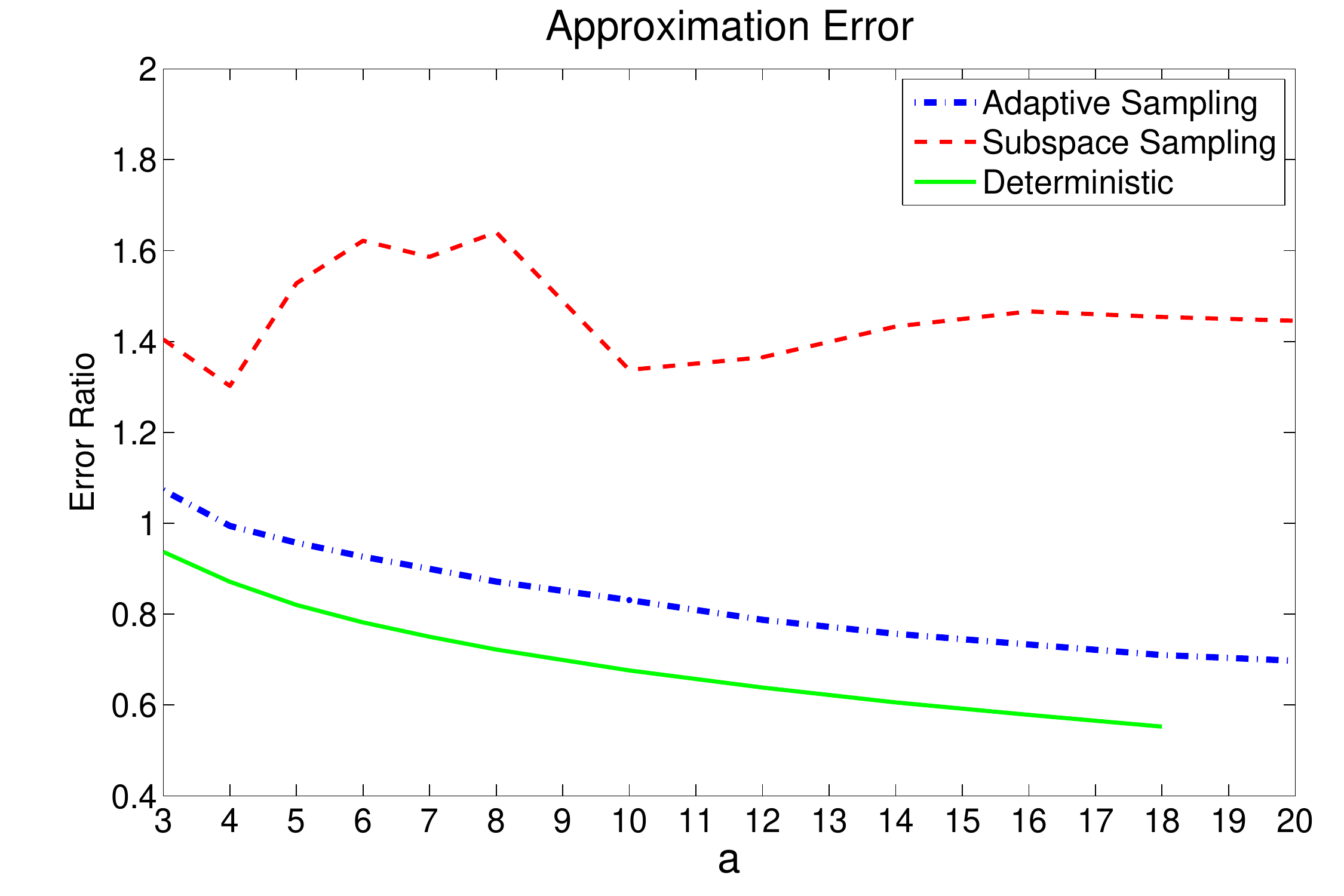}}
\end{center}
   \caption{Results of the CUR algorithms on the Enron data set.}
\label{fig:enron}
\end{figure*}

\begin{figure*}
\subfigtopskip = 0pt
\begin{center}
\centering
{\includegraphics[width=60mm,height=40mm]{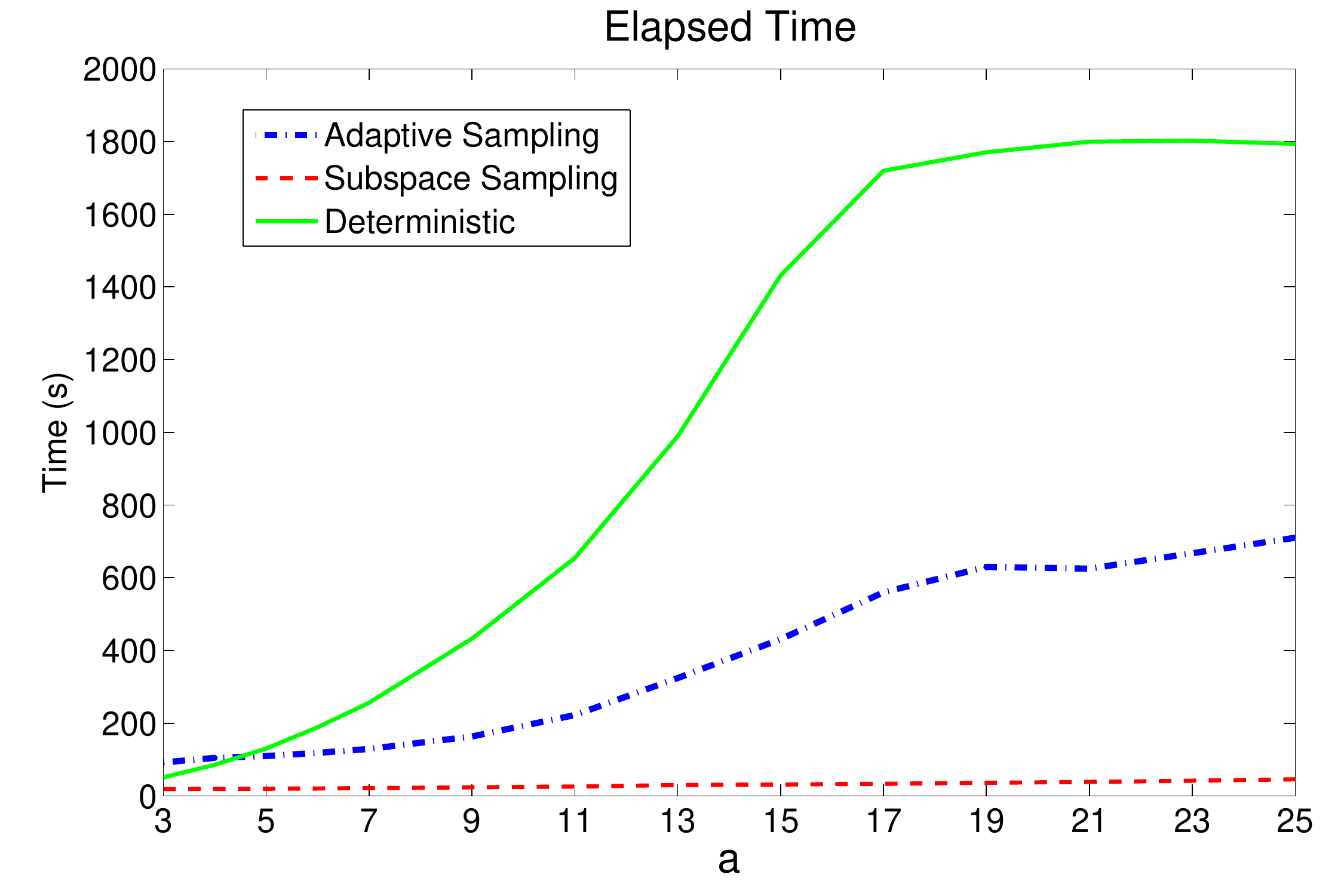}}~
{\includegraphics[width=60mm,height=40mm]{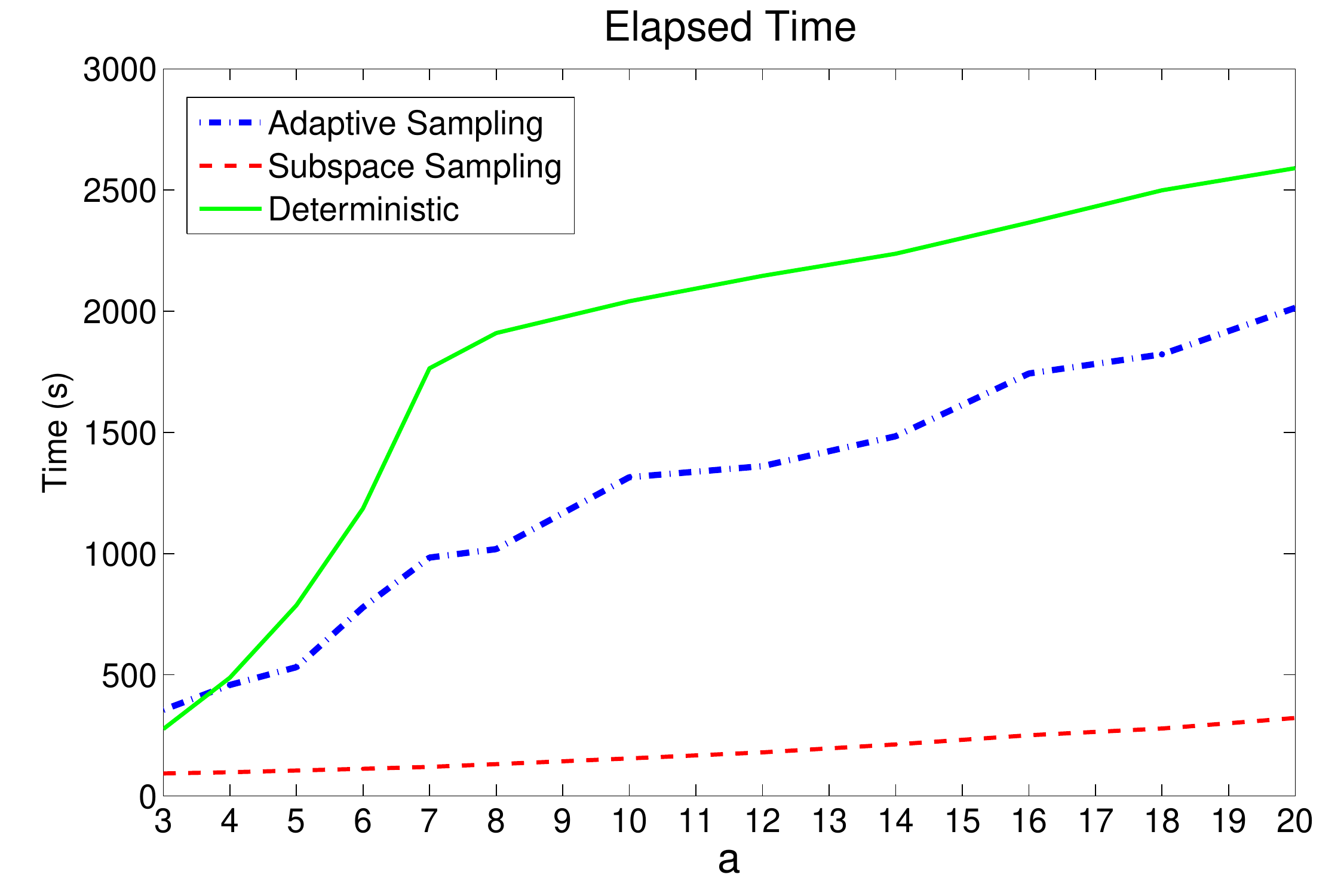}}\\
\subfigure[\textsf{$k = 10$, $c=a k$, and $r=a c$.}]{\includegraphics[width=60mm,height=40mm]{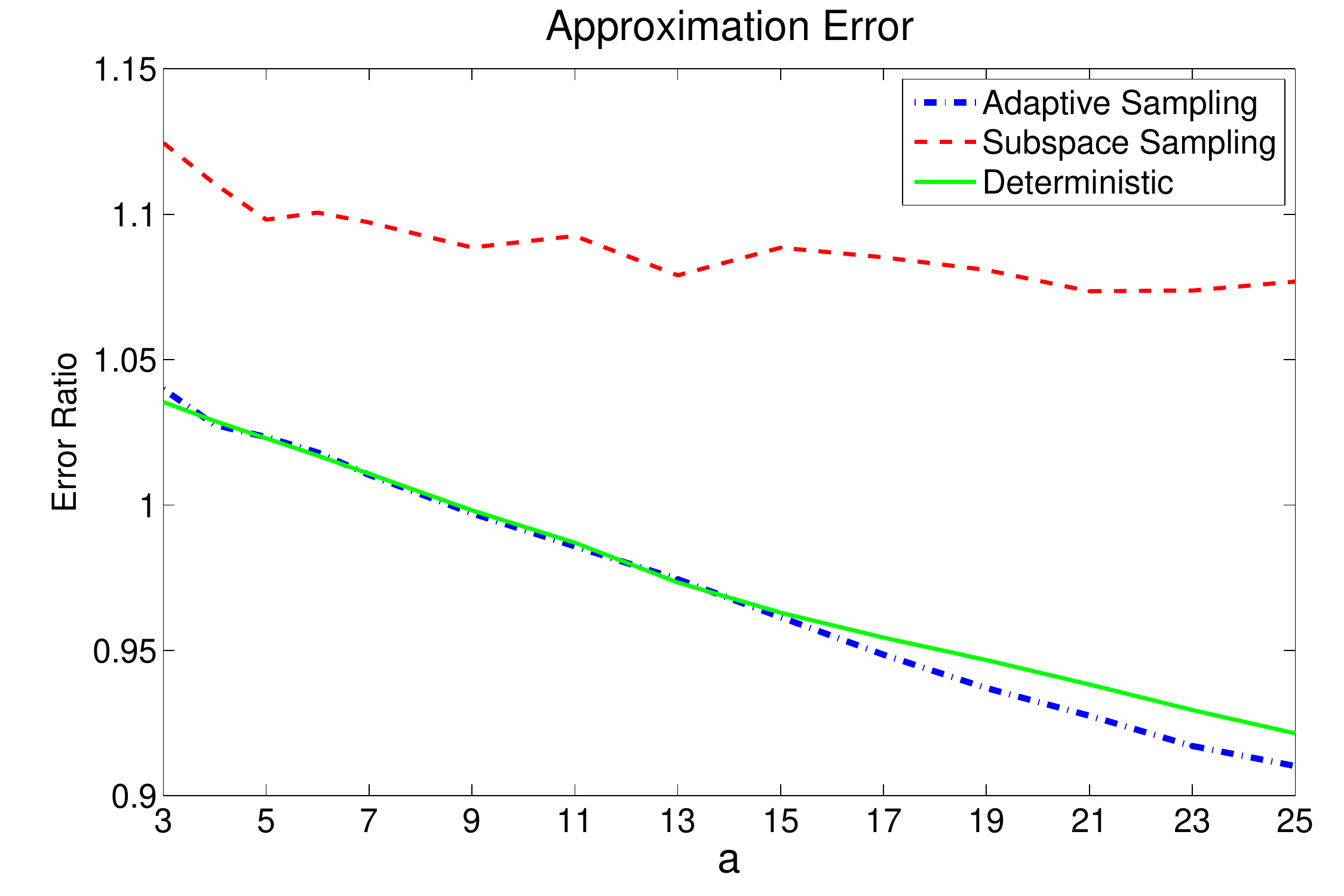}}~
\subfigure[\textsf{$k = 50$, $c=a k$, and $r=a c$.}]{\includegraphics[width=60mm,height=40mm]{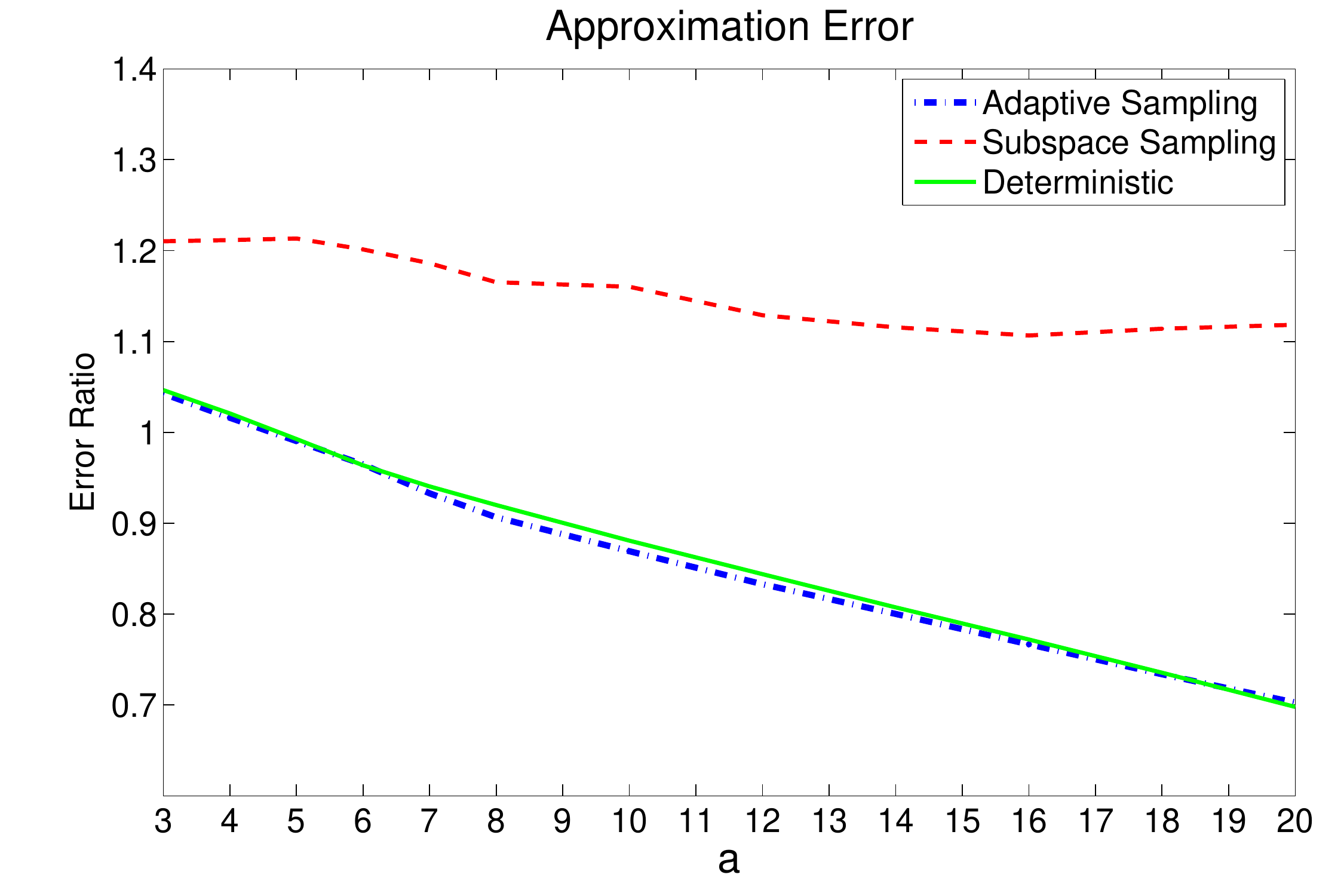}}
\end{center}
   \caption{Results of the CUR algorithms on the Dexter data set.}
\label{fig:dexter}
\end{figure*}

\begin{figure*}
\subfigtopskip = 0pt
\begin{center}
\centering
{\includegraphics[width=60mm,height=40mm]{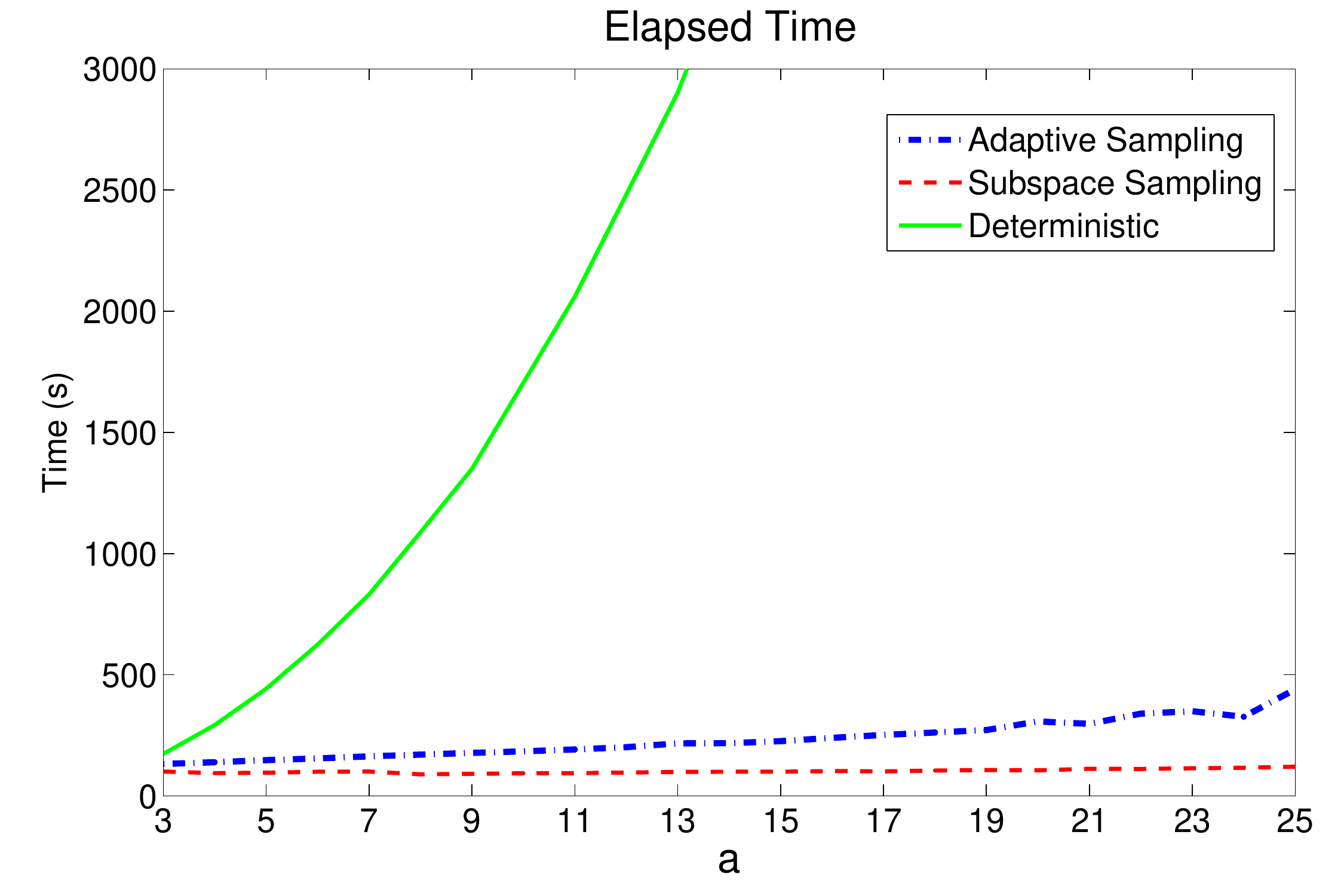}}~
{\includegraphics[width=60mm,height=40mm]{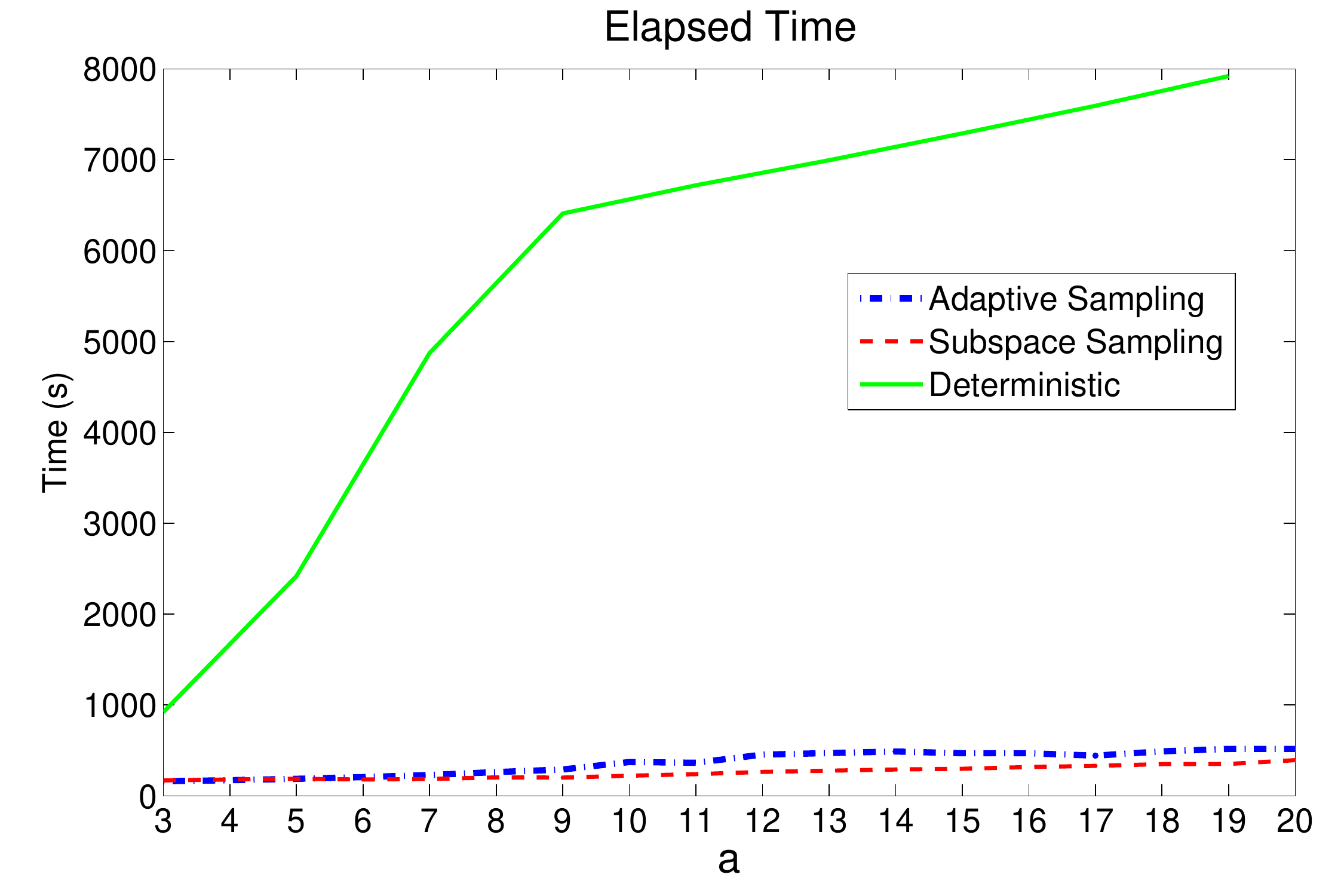}}\\
\subfigure[\textsf{$k = 10$, $c=a k$, and $r=a c$.}]{\includegraphics[width=60mm,height=40mm]{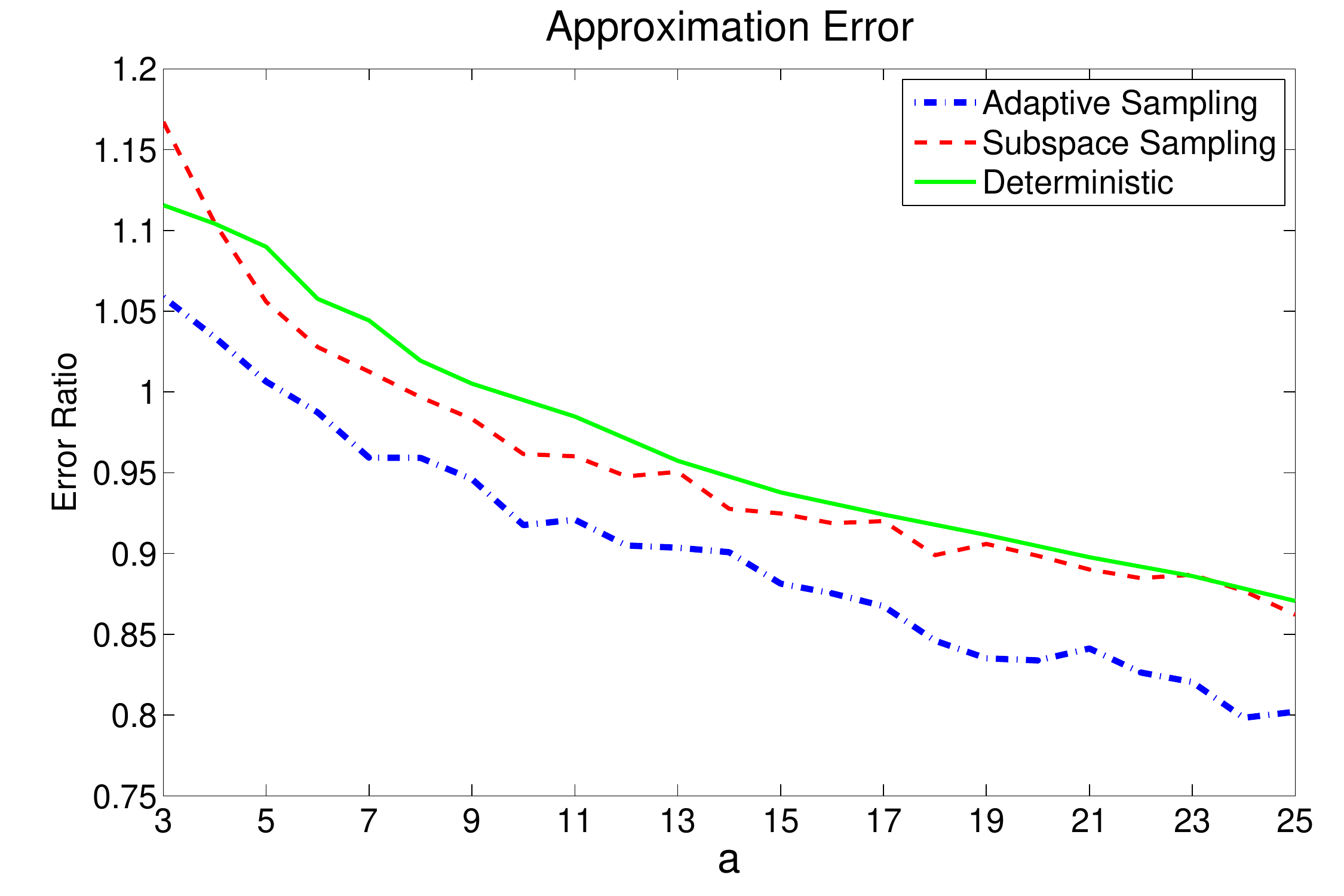}}~
\subfigure[\textsf{$k = 50$, $c=a k$, and $r=a c$.}]{\includegraphics[width=60mm,height=40mm]{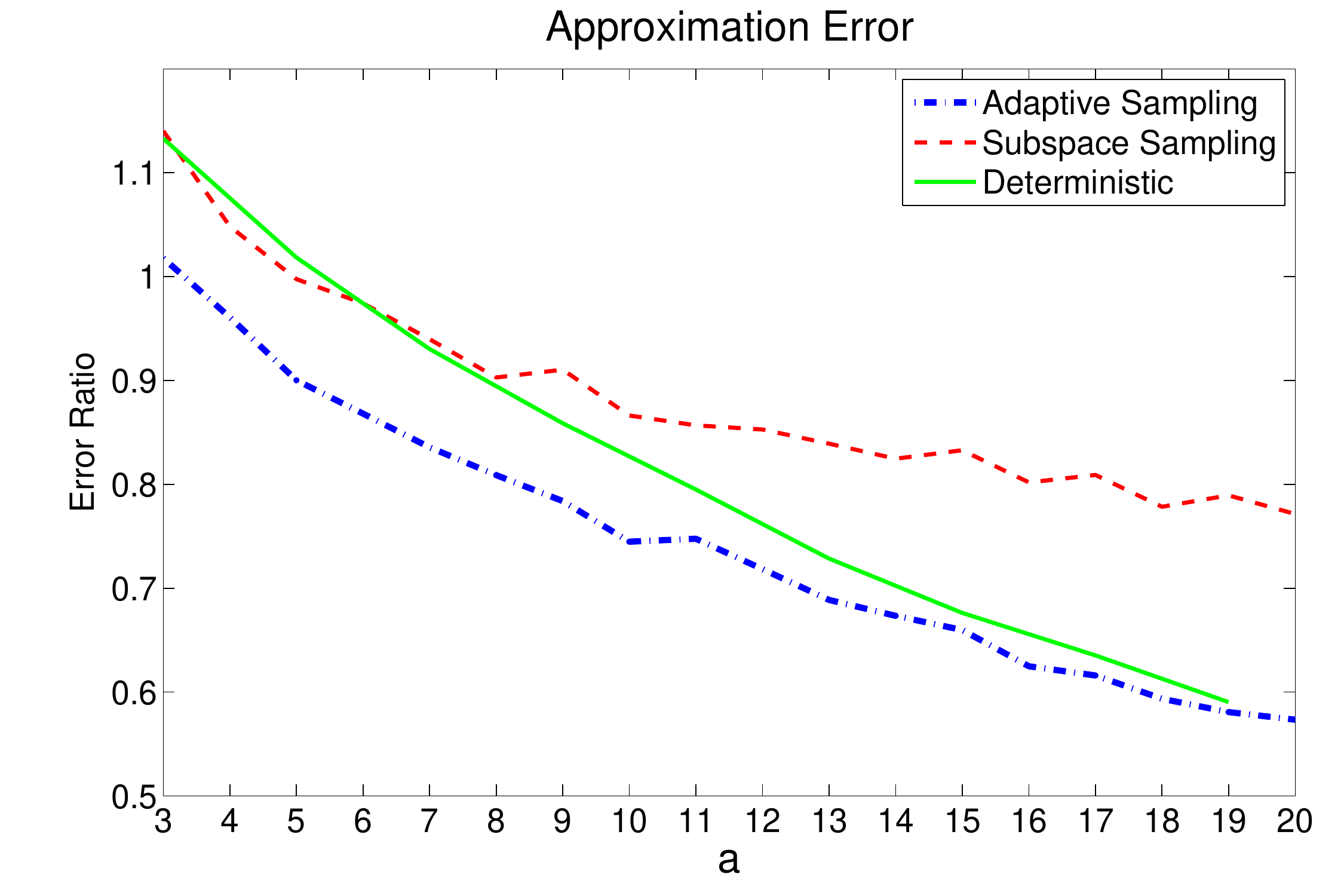}}
\end{center}
   \caption{Results of the CUR algorithms on the Farm Ads data set.}
\label{fig:farmads}
\end{figure*}

\begin{figure*}
\subfigtopskip = 0pt
\begin{center}
\centering
{\includegraphics[width=60mm,height=40mm]{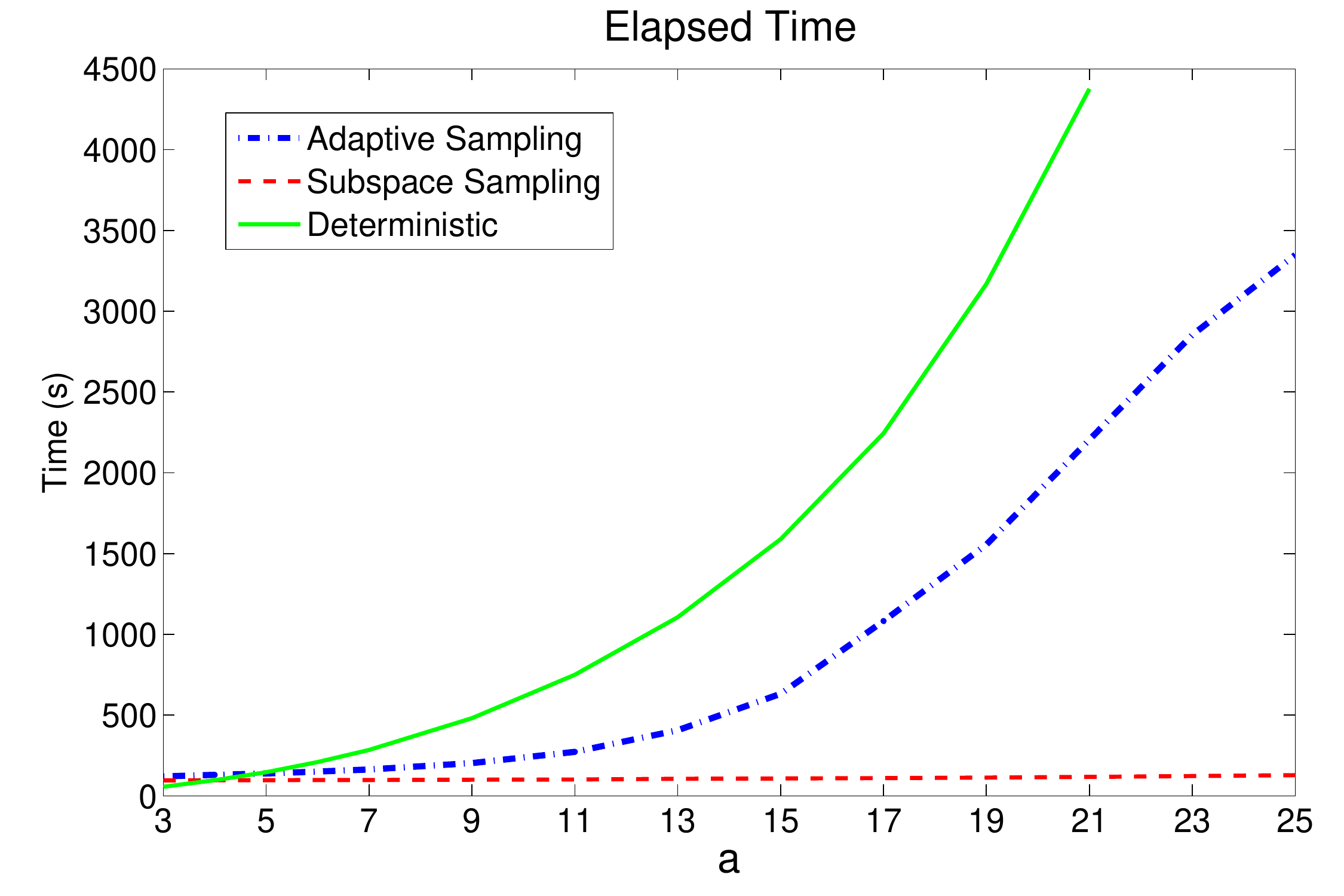}}~
{\includegraphics[width=60mm,height=40mm]{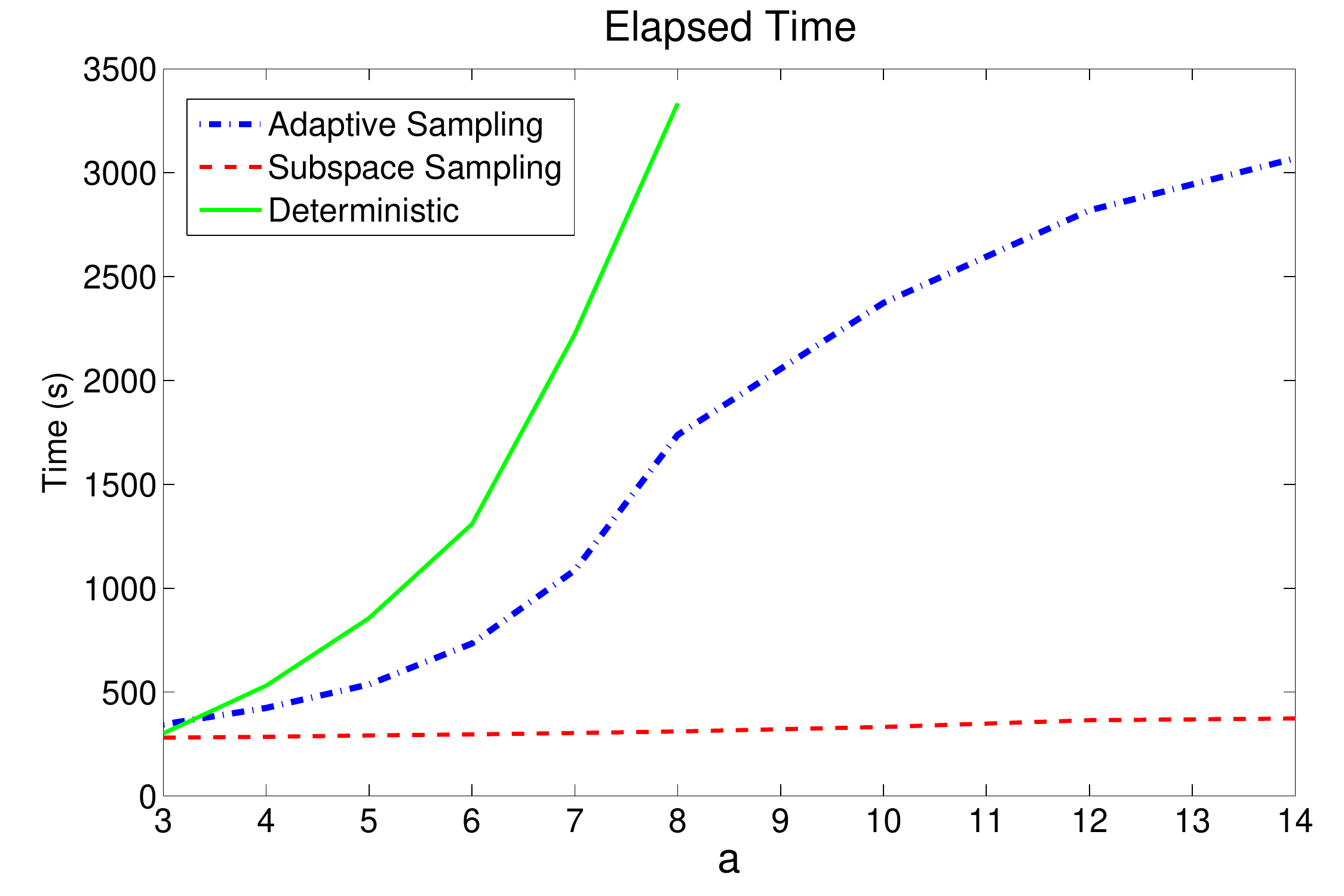}}\\
\subfigure[\textsf{$k = 10$, $c=a k$, and $r=a c$.}]{\includegraphics[width=60mm,height=40mm]{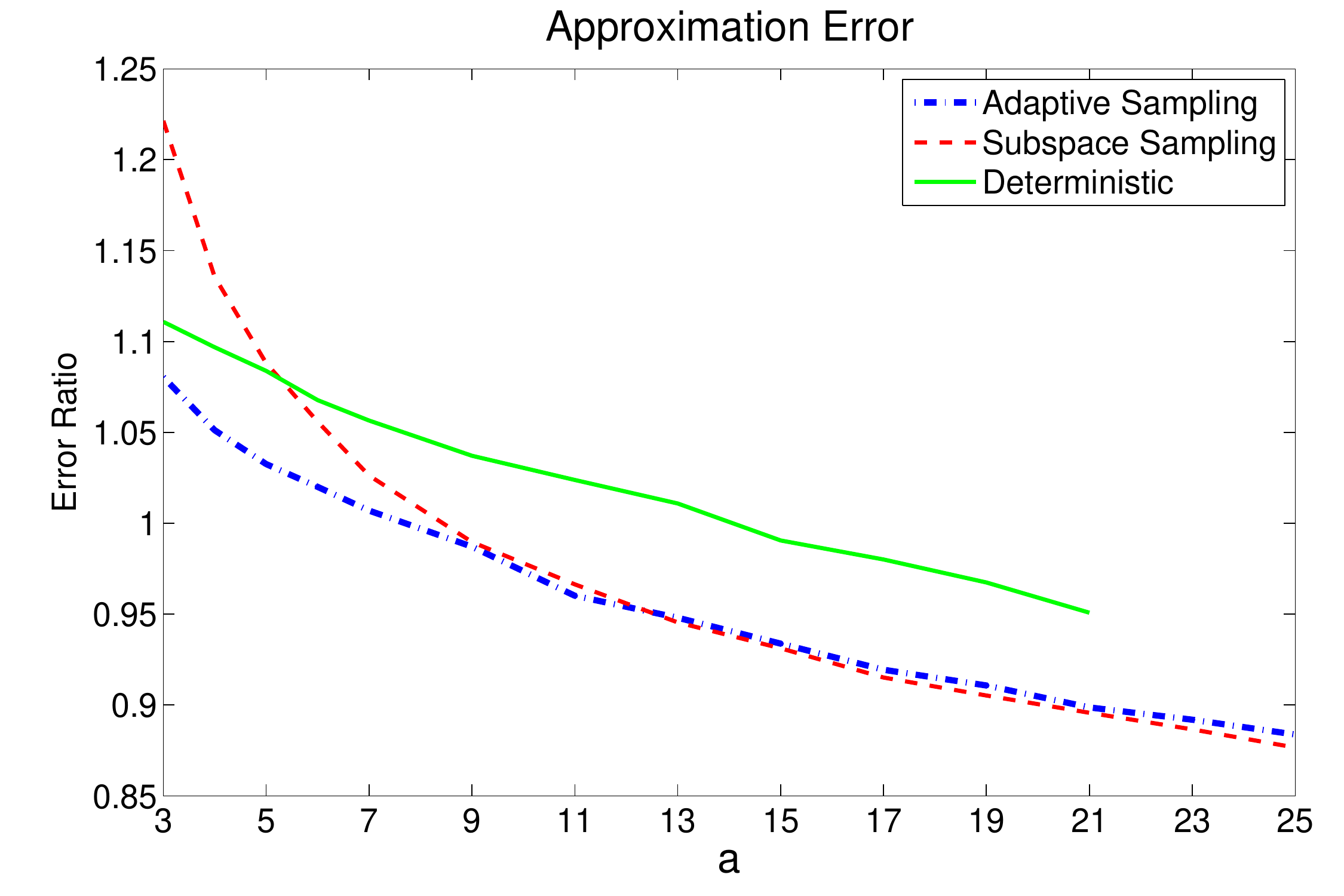}}~
\subfigure[\textsf{$k = 50$, $c=a k$, and $r=a c$.}]{\includegraphics[width=60mm,height=40mm]{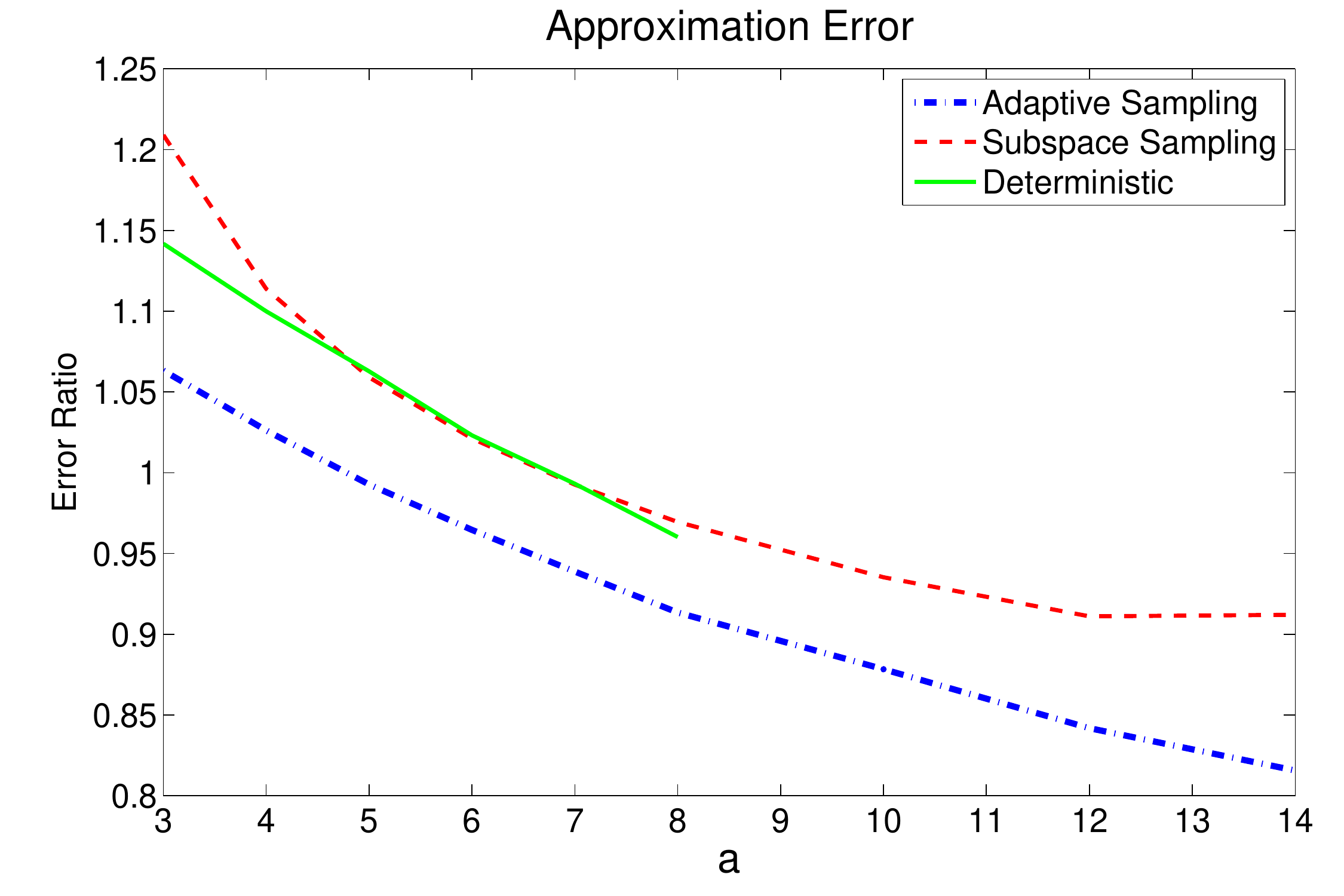}}
\end{center}
   \caption{Results of the CUR algorithms on the Gisette data set.}
\label{fig:gisette}
\end{figure*}

We conduct experiments on four UCI data sets \citep{uci2010} which are summarized in Table~\ref{tab:datasets}.
Each data set is represented as a data matrix, upon which we apply the \CUR algorithms.
According to our analysis, the target rank $k$ should be far less than $m$ and $n$,
and the column number $c$ and row number $r$ should be strictly greater than $k$.
For each data set and each algorithm, we set $k = 10$ or $50$,
and $c = ak$, $r =ac$, where $a$ ranges in each set of experiments.
We repeat each of the two randomized algorithms $10$ times, and report the minimum error ratio
and the total elapsed time of the $10$ rounds.
We depict the error ratios and the elapsed time of the three CUR matrix decomposition algorithms
in Figures \ref{fig:enron}, \ref{fig:dexter}, \ref{fig:farmads}, and \ref{fig:gisette}.

We can see from Figures \ref{fig:enron}, \ref{fig:dexter}, \ref{fig:farmads}, and \ref{fig:gisette}
that our adaptive sampling based \CUR algorithm
has much lower approximation error than the subspace sampling algorithm in all cases.
Our adaptive sampling based algorithm is better than the deterministic SCRA on the Farm Ads data set
and the Gisette data set, worse than SCRA on the Enron data set, and comparable to SCRA
on the Dexter data set.
In addition, the experimental results  match our theoretical analysis in Section~\ref{sec:fast_cur} very well.
The empirical results all obey the theoretical relative-error upper bound
\[
\frac{\|\A - \C \U \R\|_F }{ \|\A - \A_k\|_F}
\; \leq \; 1 + \frac{2 k}{c} \big(1+o(1)\big)
\;=\; 1 + \frac{2}{a} \big(1+o(1)\big) \textrm{.}
\]

As for the running time, the subspace sampling algorithm and our adaptive sampling based algorithm
are much more efficient than SCRA, especially when $c$ and $r$ are large.
Our adaptive sampling based algorithm is comparable to the subspace sampling algorithm
when $c$ and $r$ are small;
however, our algorithm becomes less efficient when $c$ and $r$ are large.
This is due to the following reasons.
First, the computational cost of the subspace sampling algorithm is dominated by the truncated SVD of $\A$,
which is determined by the target rank $k$ and the size and sparsity of the data matrix.
However, the cost of our algorithm grows with $c$ and $r$.
Thus, our algorithm becomes less efficient when $c$ and $r$ are large.
Second, the truncated SVD operation in MATLAB, that is, the `$\mathrm{svds}$' function, gains from sparsity,
but our algorithm does not.
The four data sets are all very sparse, so the subspace sampling algorithm has advantages.
Third, the truncated SVD functions are very well implemented by MATLAB (not in MATLAB language but in Fortran/C).
In contrast, our algorithm is implemented in MATLAB language, which is usually less efficient than Fortran/C.

\subsection{Comparison among the \nystrom Algorithms} \label{sec:experiments_nystrom}

In this section we empirically compare our adaptive sampling algorithm (in Theorem~\ref{thm:nystrom_bound}) with some other
sampling algorithms including the subspace sampling of \citet{drineas2008cur} and the uniform sampling, both without replacement.
We also conduct comparison between the standard \nystrom and our modified Nystr\"{o}m,
both use the three sampling algorithms to select columns.

We test the algorithms on three data sets which are summarized in Table~\ref{tab:datasets_nystrom}.
The experiment setting follows \cite{gittens2013revisiting}.
For each data set we generate a radial basis function (RBF) kernel matrix $\A$ which is defined by
\[
a_{i j} = \exp \bigg( -\frac{\|\x_i - \x_j \|_2^2}{2\sigma^2} \bigg),
\]
where $\x_i$ and $\x_j$ are data instances and $\sigma$ is a scale parameter.
Notice that the RBF kernel is dense in general.
We set $\sigma = 0.2$ or $1$ in our experiments.
For each data set with different settings of $\sigma$, we fix a target rank $k=10$, $20$ or $50$
and vary $c$ in a very large range.
We will discuss the choice of $\sigma$ and $k$ in the following two paragraphs.
We run each algorithm for $10$ times, and report the the minimum error ratio as well as the total elapsed time of the $10$ repeats.
The results are shown in Figures~\ref{fig:abalone}, \ref{fig:wine}, and \ref{fig:letter}.

\begin{table}[t]\setlength{\tabcolsep}{0.3pt}
\begin{center}
\begin{footnotesize}
\begin{tabular}{c c c c c}
\hline
	{\bf Data Set}	&  {\bf ~~\#Instances~~}  &	{\bf ~~\#Attributes~~}   & {\bf ~~Source}	\\
\hline
    Abalone         &   $4,177$  &   $8$   & UCI \citep{uci2010} \\
    Wine Quality    &   $4,898$  &   $12$  & UCI \citep{cortez2009modeling} \\
    Letters         &   $5,000$  &   $16$  & Statlog \citep{michie1994machine} \\
\hline
\end{tabular}

\vspace{4mm}

\begin{tabular}{c | c c c | c c c}
\hline
                                & \multicolumn{3}{|c|}{~~${\|\A - \A_{k}\|_F}/{\|\A\|_F}$~~ }
                                & \multicolumn{3}{c}{ $\frac{m}{k}\mathsf{std}\big(\mathbf{\ell}^{[k]}\big)$ }	\\
                                & ~~$k=10$~~ & ~~$k=20$~~ & ~~$k=50$~~ & ~~$k=10$~~ & ~~$k=20$~~ & ~~$k=50$~~ \\
\hline
Abalone ($\sigma=0.2$)      &  $0.4689$     &  $0.3144$     &  $0.1812$     & $0.8194$ & $0.6717$ & $0.4894$  \\
Abalone ($\sigma=1.0$)      &  $0.0387$    &   $0.0122$    &   $0.0023$     & $0.5879$ & $0.8415$ & $1.3830$  \\
Wine Quality ($\sigma=0.2$) &   $0.8463$    &   $0.7930$    &   $0.7086$    & $1.8703$ & $1.6490$ & $1.3715$  \\
Wine Quality ($\sigma=1.0$) &   $0.0504$   &   $0.0245$    &   $0.0084$     & $0.3052$ & $0.5124$ & $0.8067$  \\
Letters ($\sigma=0.2$)      &   $0.9546$    &   $0.9324$    &   $0.8877$    & $5.4929$ & $3.9346$ & $2.6210$  \\
Letters ($\sigma=1.0$)      &   $0.1254$   &   $0.0735$    &   $0.0319$     & $0.2481$ & $0.2938$ & $0.3833$  \\
\hline
\end{tabular}
\end{footnotesize}
\end{center}
\caption{A summary of the data sets for the \nystrom approximation.
        In the second tabular $\mathsf{std}\big(\mathbf{\ell}^{[k]}\big)$ denotes the standard deviation of the statistical leverage scores
        of $\A$ relative to the best rank-$k$ approximation to $\A$. We use the normalization factor $\frac{m}{k}$ because
        $\frac{m}{k}\mathsf{mean}\big(\mathbf{\ell}^{[k]}\big)=1$.}
\label{tab:datasets_nystrom}
\end{table}

Table~\ref{tab:datasets_nystrom} provides useful implications on choosing the target rank $k$.
In Table~\ref{tab:datasets_nystrom}, $\frac{\|\A - \A_{k}\|_F}{\|\A\|_F}$ denotes ratio that is not captured by the best rank-$k$ approximation to the RBF kernel,
and the parameter $\sigma$ has an influence on the ratio ${\|\A - \A_{k}\|_F}/{\|\A\|_F}$.
When $\sigma$ is large, the RBF kernel can be well approximated by a low-rank matrix,
which implies that (i) a small $k$ suffices when $\sigma$ is large, and (ii) $k$ should be set large when $\sigma$ is small.
So the settings ($\sigma = 1$, $k=10$) and ($\sigma=0.2$, $k=50$) are more reasonable than the rest.
Let us take the RBF kernel in the Abalone data set as an example.
When $\sigma=1$, the rank-$10$ approximation well captures the kernel, so $k$ can be safely set as small as $10$;
when $\sigma=0.2$, the target rank $k$ should be set large, say larger than $50$, otherwise the approximation is rough.

The standard deviation of the leverage scores reflects whether the advanced importance sampling techniques
such as the subspace sampling and adaptive sampling are useful.
Figures \ref{fig:abalone}, \ref{fig:wine}, and \ref{fig:letter} show that the advantage of the subspace sampling and adaptive sampling over
the uniform sampling is significant whenever the standard deviation of the leverage scores is large (see Table~\ref{tab:datasets_nystrom}),
and vise versa.
Actually, as reflected in Table~\ref{tab:datasets_nystrom}, the parameter $\sigma$ influences the homogeneity/heterogeneity
of the leverage scores.
Usually, when $\sigma$ is small, the leverage scores become heterogeneous, and the effect of choosing ``good" columns is significant.

\begin{figure*}
\subfigtopskip = 0pt
\begin{center}
\centering
\includegraphics[width=50mm,height=45mm]{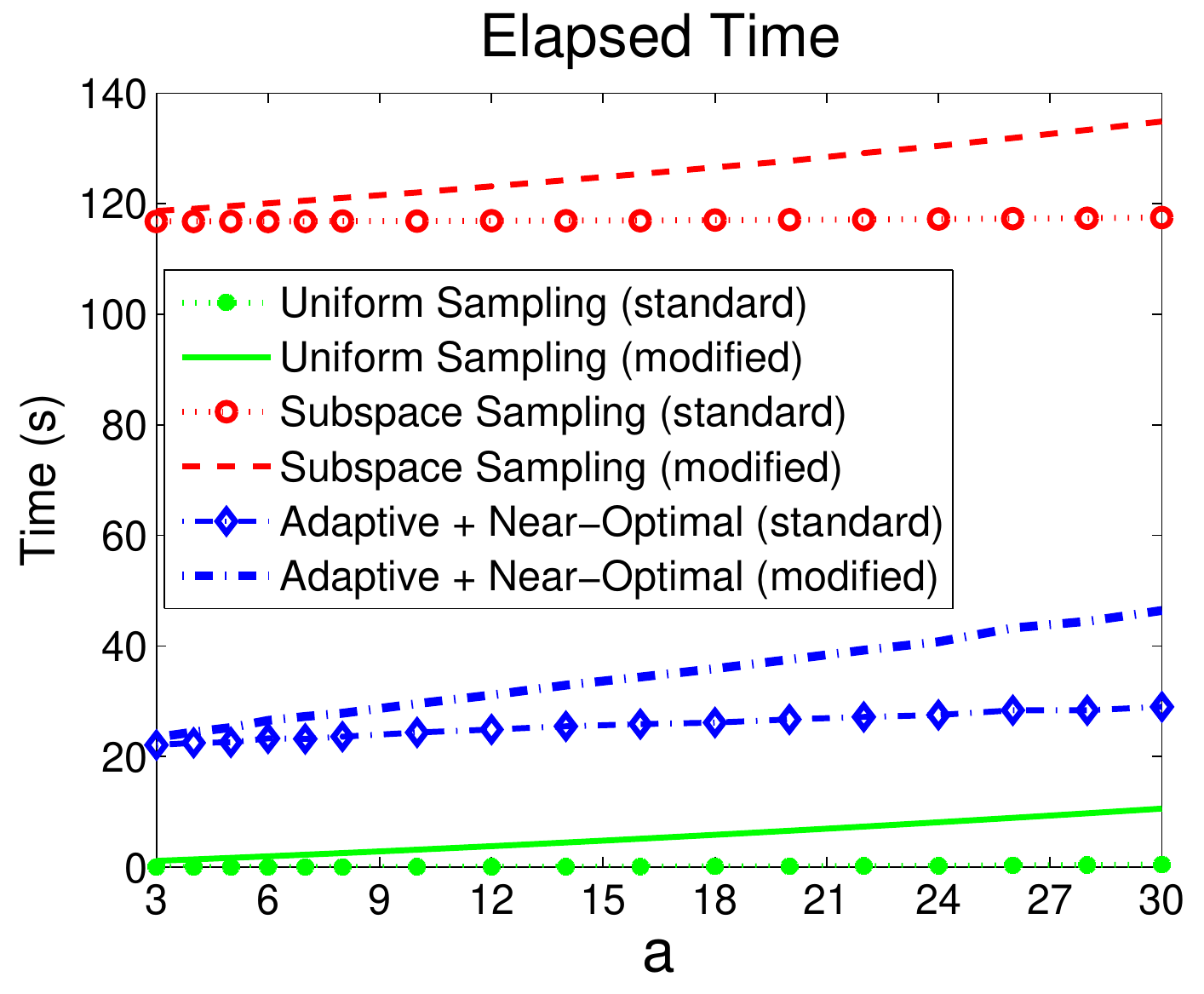}
\includegraphics[width=50mm,height=45mm]{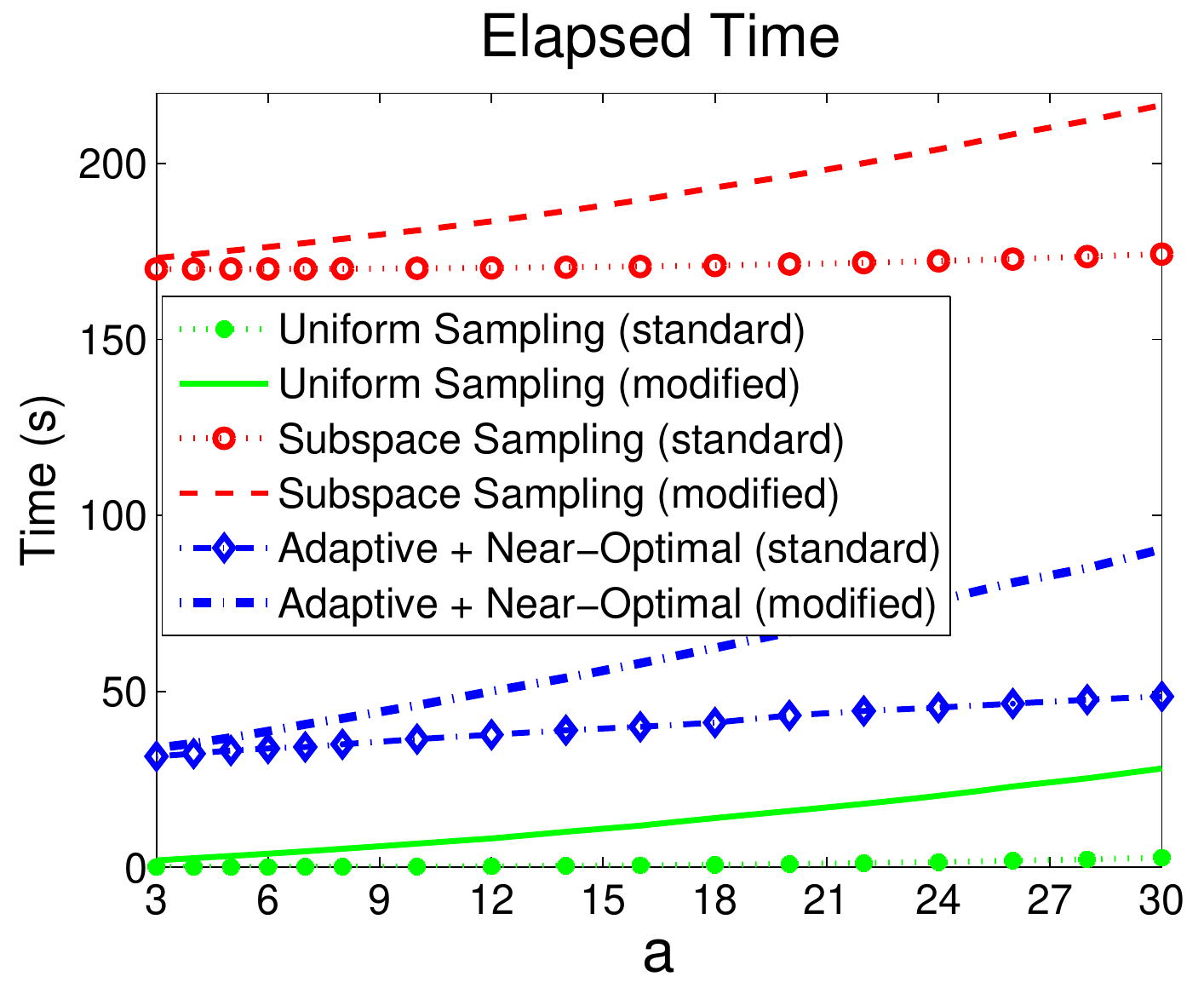}
\includegraphics[width=50mm,height=45mm]{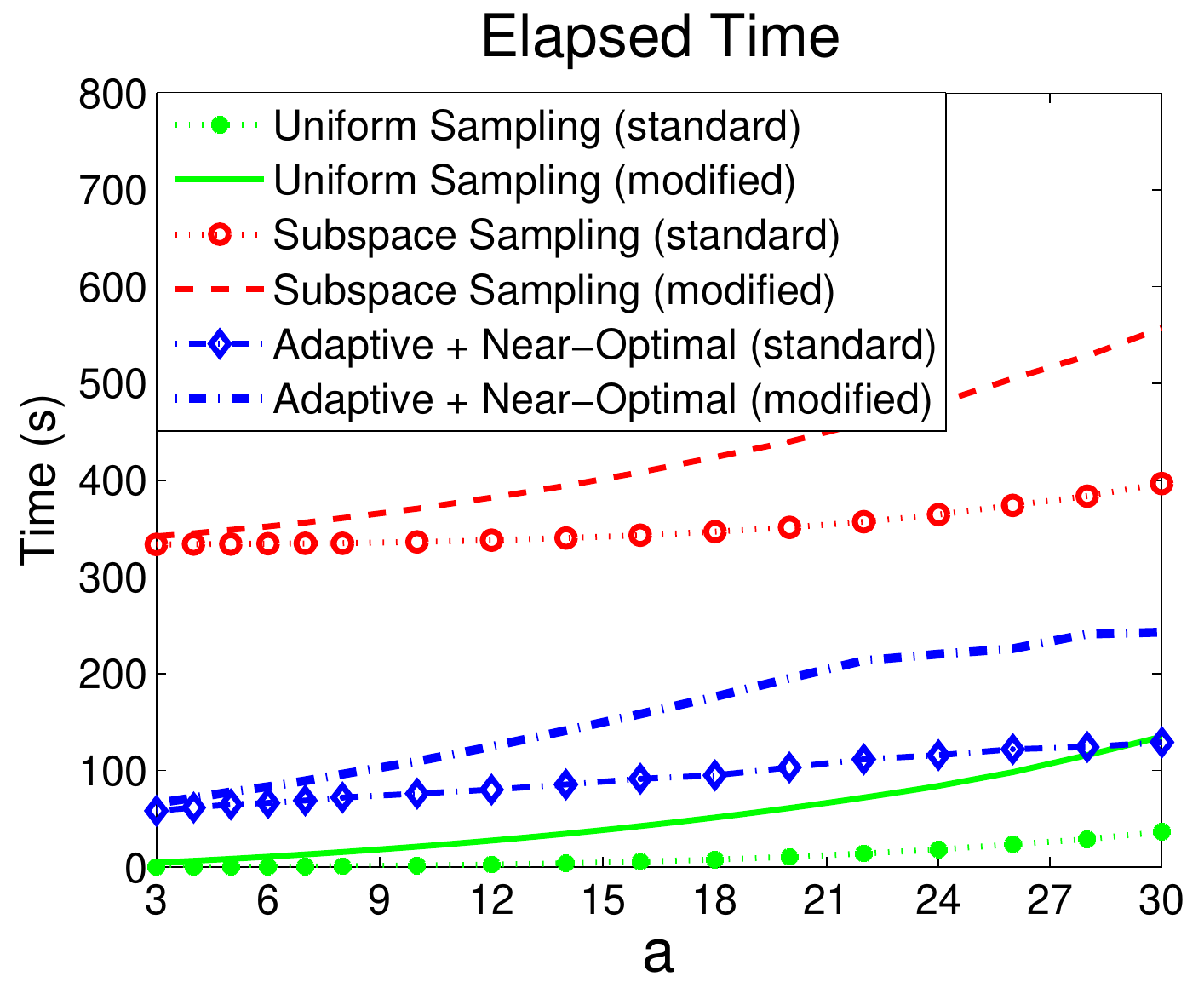}\\
\subfigure[\textsf{$\sigma = 0.2$, $k = 10$, and $c=a k$.}]{\includegraphics[width=50mm,height=45mm]{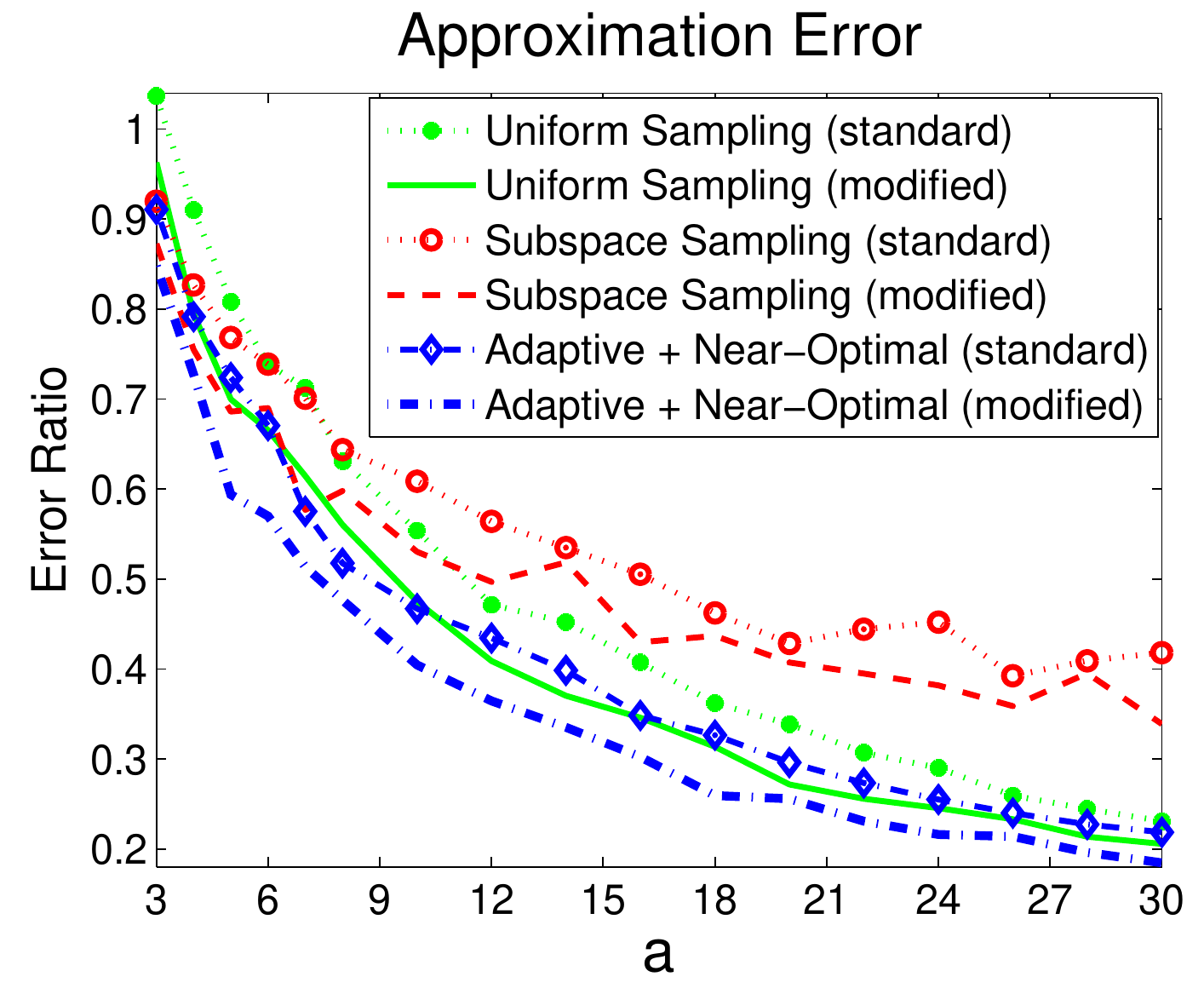}}
\subfigure[\textsf{$\sigma = 0.2$, $k = 20$, and $c=a k$.}]{\includegraphics[width=50mm,height=45mm]{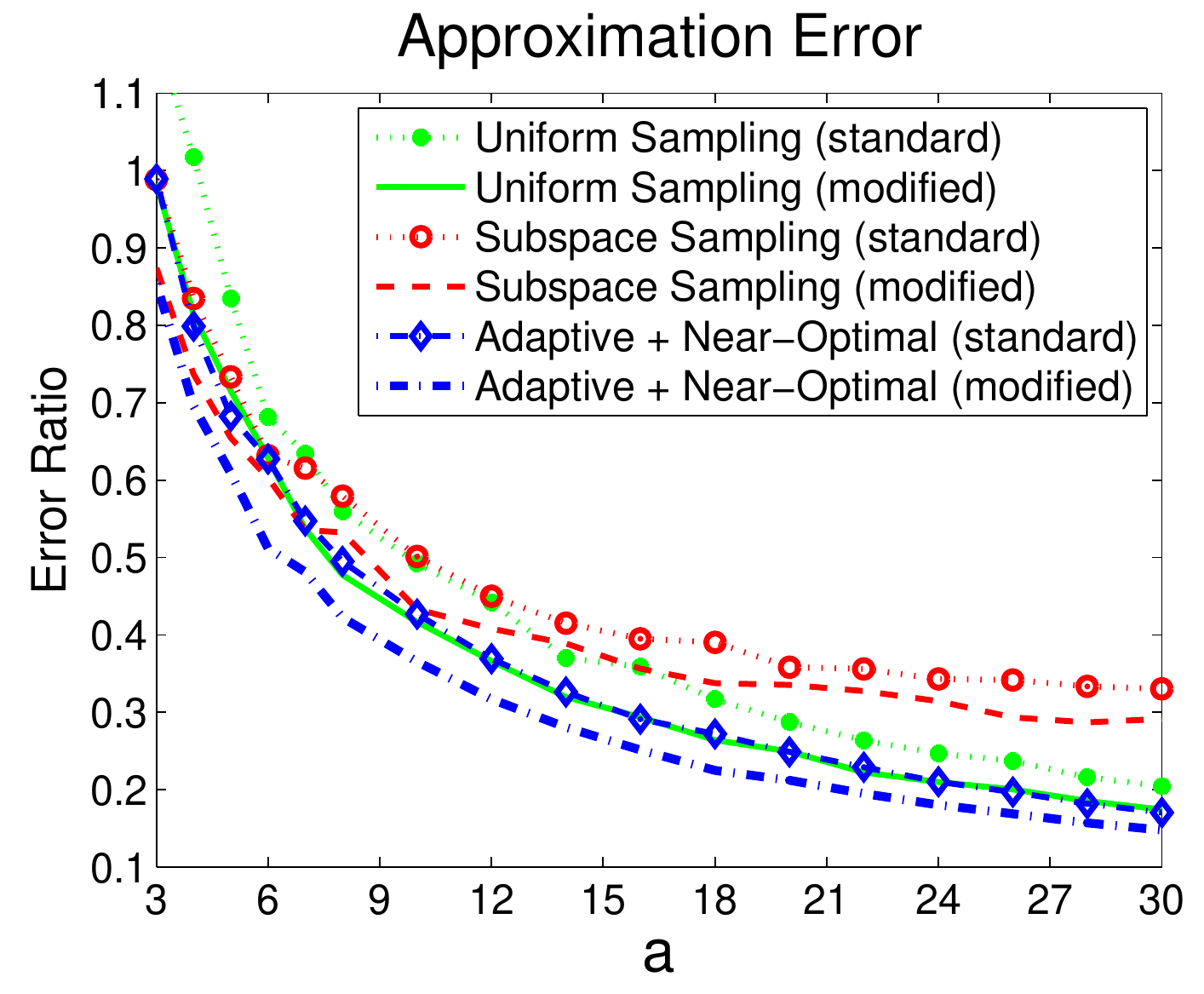}}
\subfigure[\textsf{$\sigma = 0.2$, $k = 50$, and $c=a k$.}]{\includegraphics[width=50mm,height=45mm]{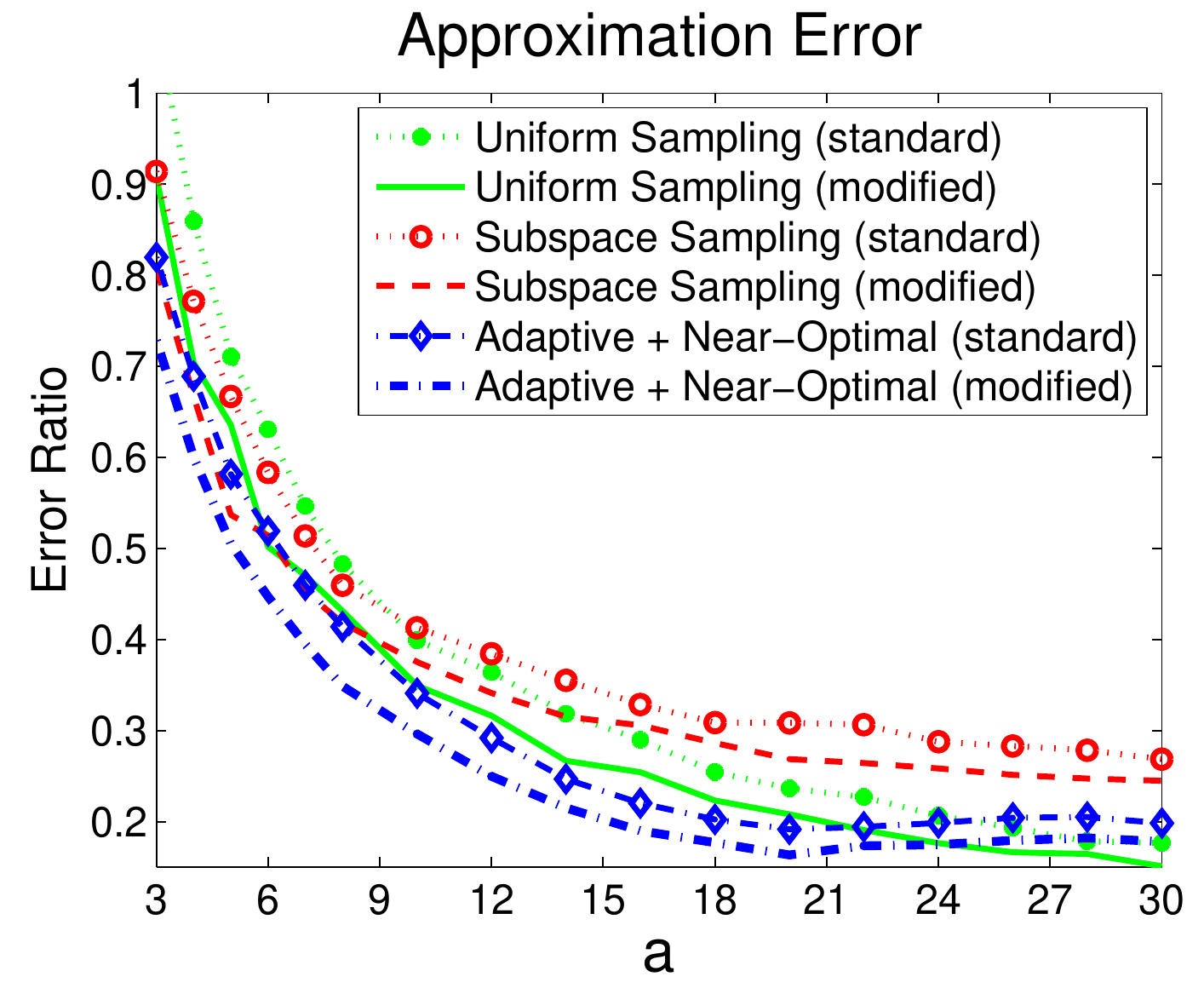}}\\

\vspace{8mm}

\includegraphics[width=50mm,height=45mm]{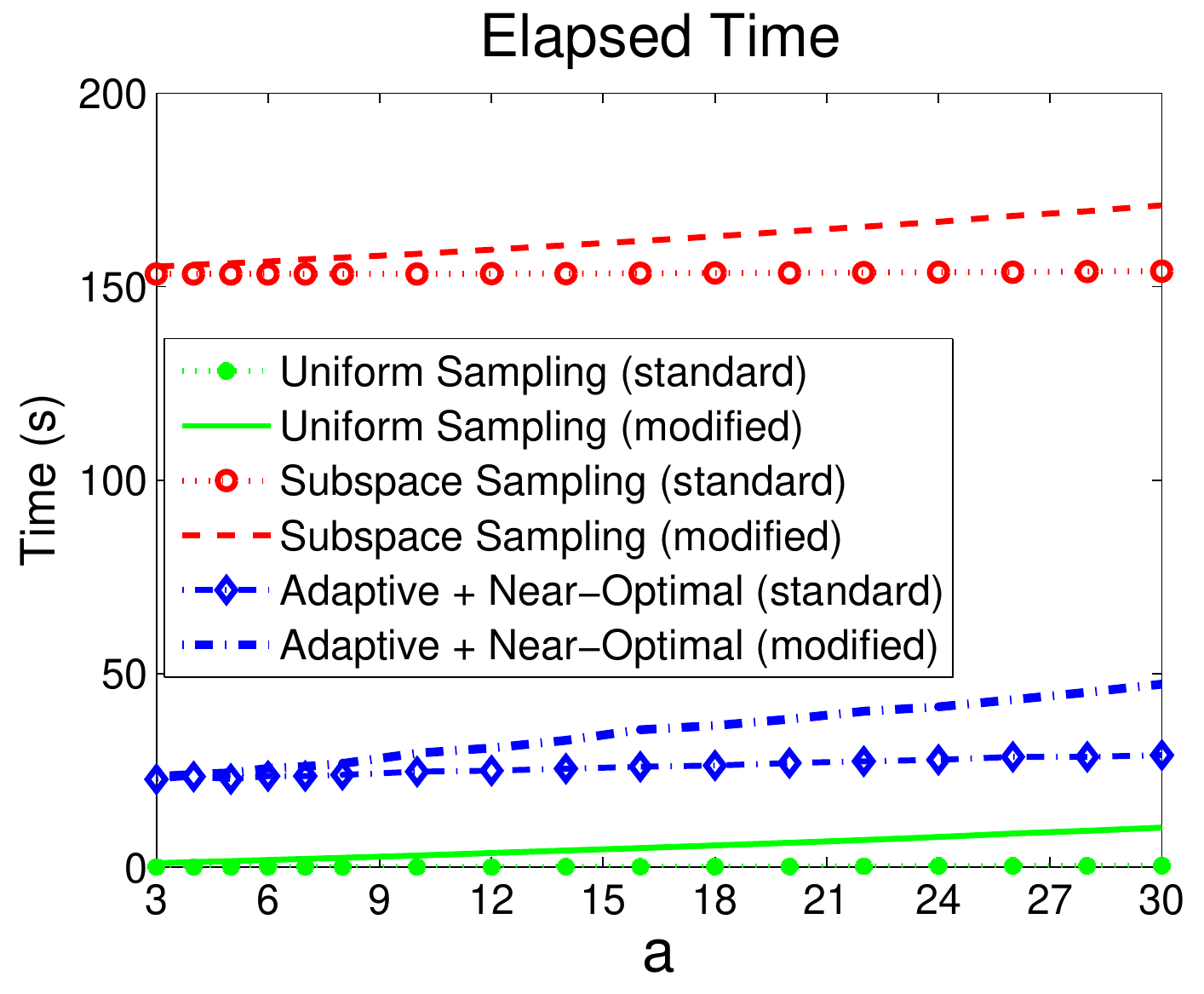}~
\includegraphics[width=50mm,height=45mm]{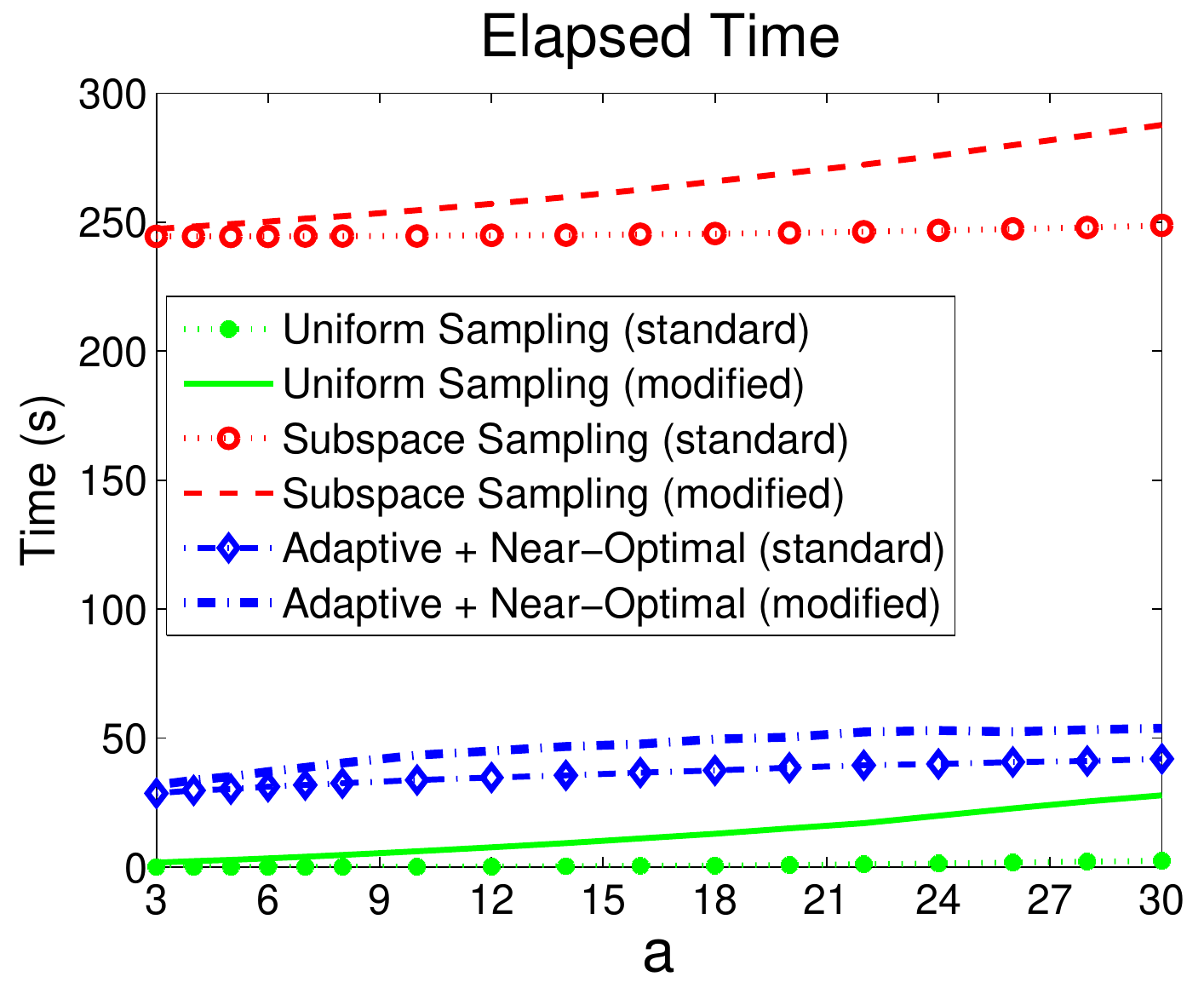}~
\includegraphics[width=50mm,height=45mm]{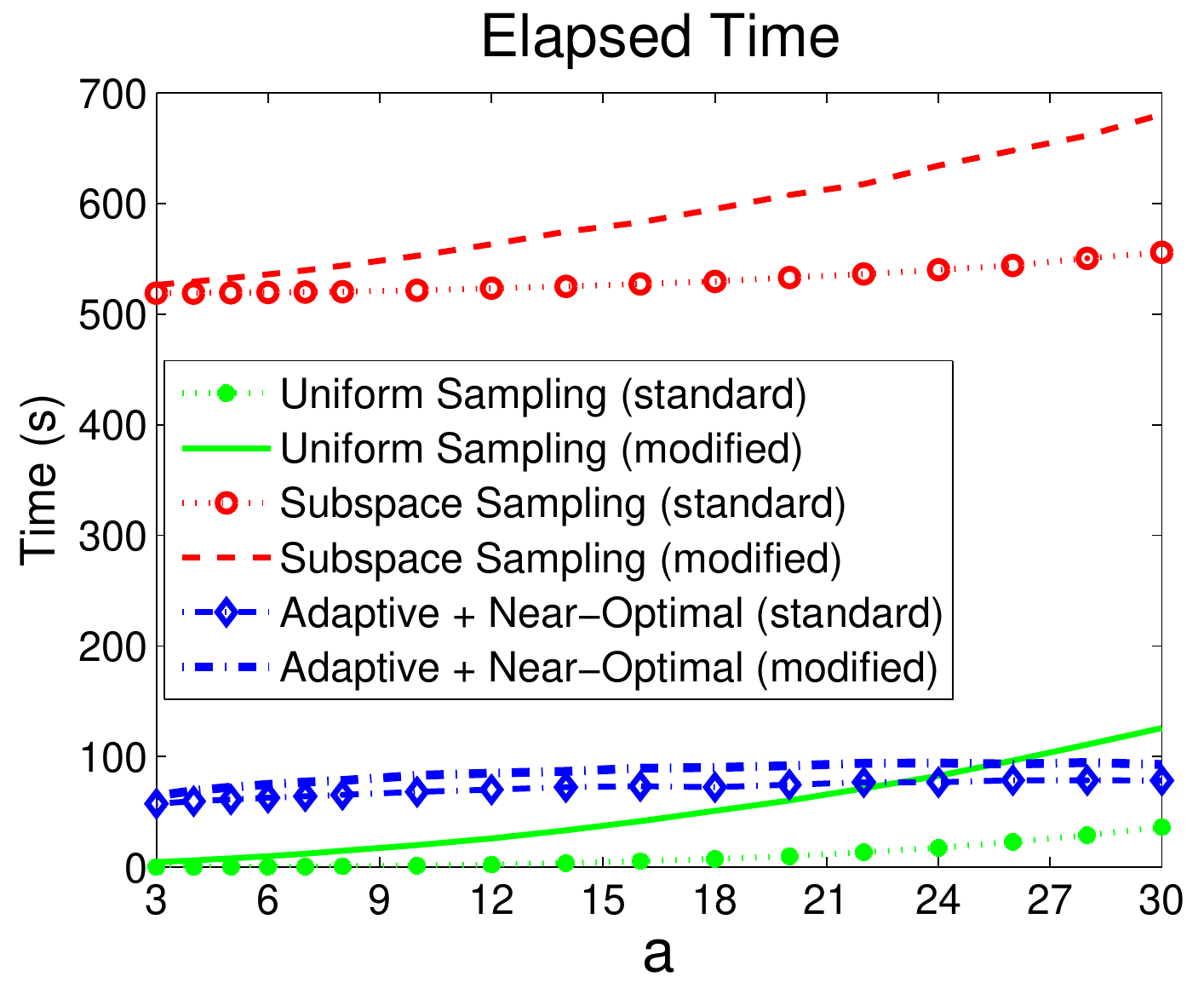}\\
\subfigure[\textsf{$\sigma = 1$, $k = 10$, and $c=a k$.}]{\includegraphics[width=50mm,height=45mm]{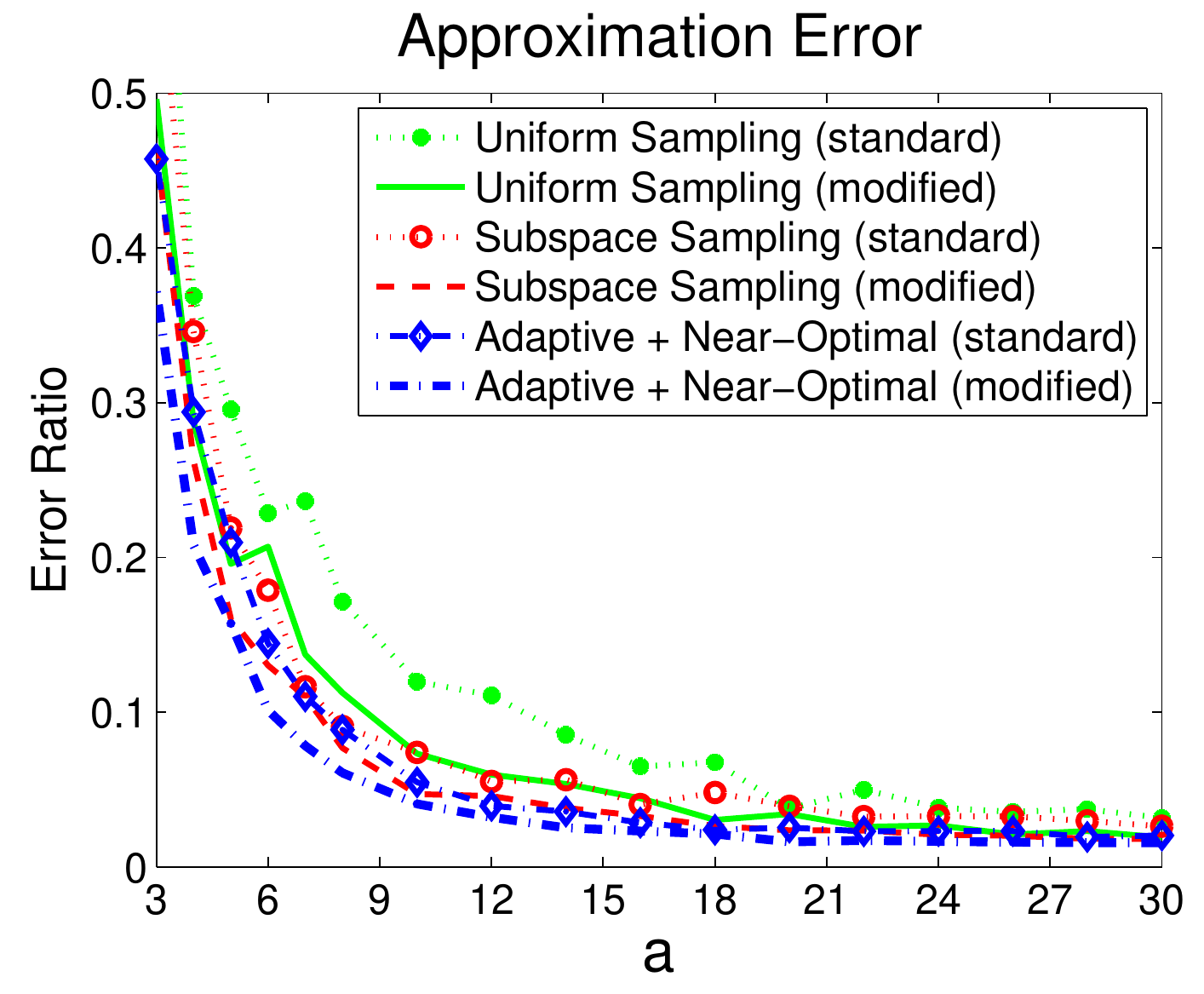}}
\subfigure[\textsf{$\sigma = 1$, $k = 20$, and $c=a k$.}]{\includegraphics[width=50mm,height=45mm]{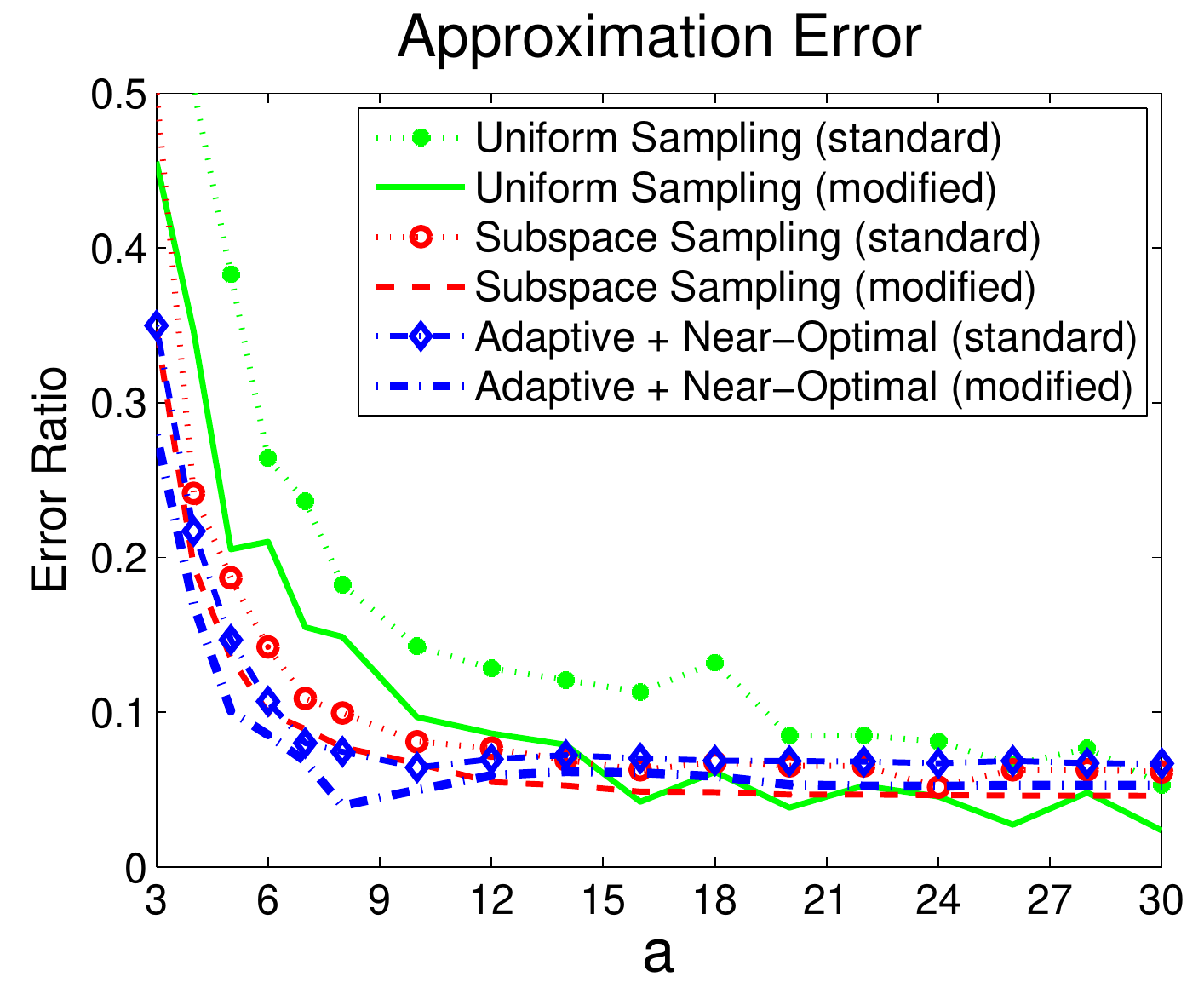}}
\subfigure[\textsf{$\sigma = 1$, $k = 50$, and $c=a k$.}]{\includegraphics[width=50mm,height=45mm]{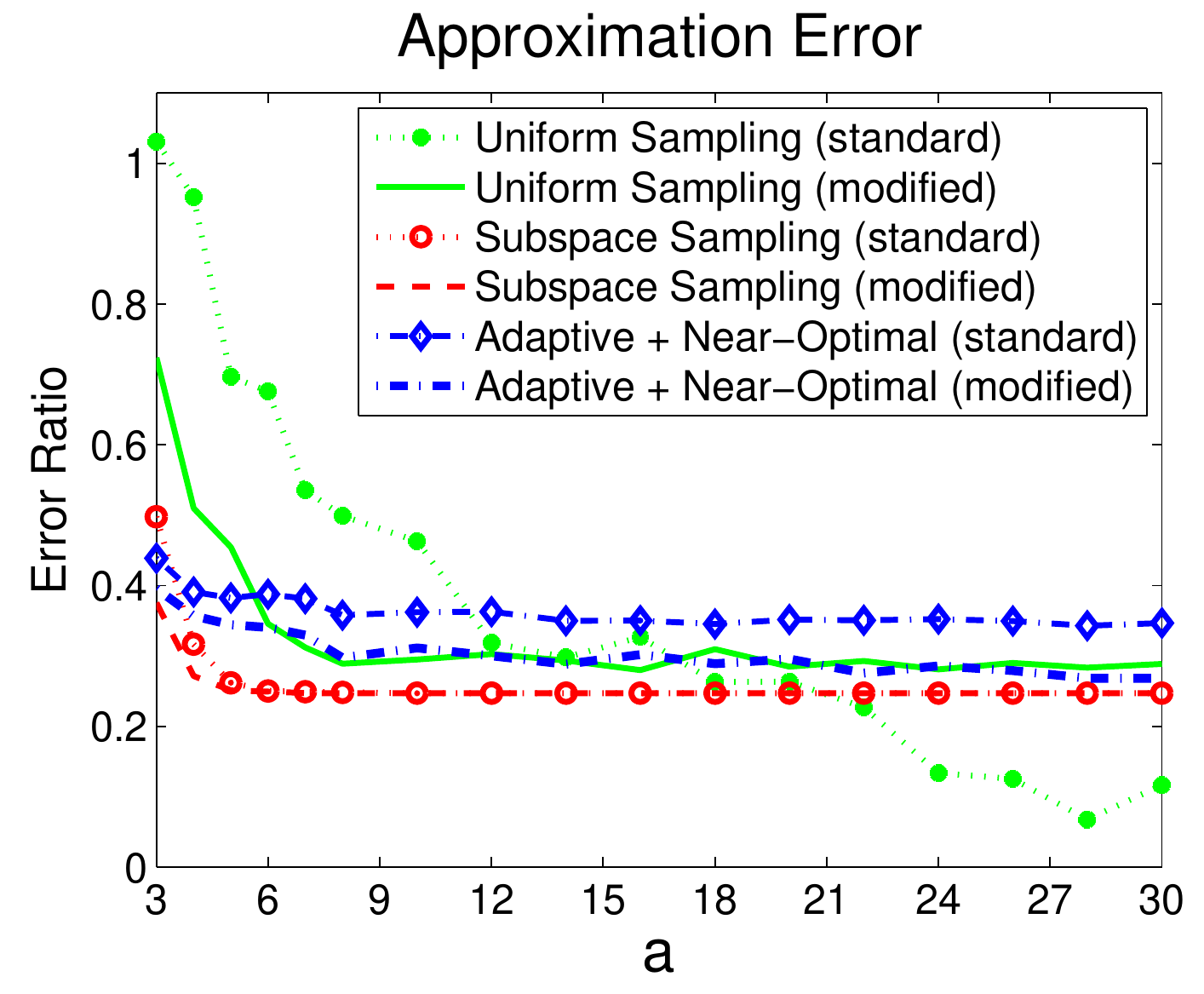}}
\end{center}
   \caption{Results of the \nystrom algorithms on the RBF kernel in the Abalone data set.}
\label{fig:abalone}
\end{figure*}

\begin{figure*}
\subfigtopskip = 0pt
\begin{center}
\centering
\includegraphics[width=50mm,height=45mm]{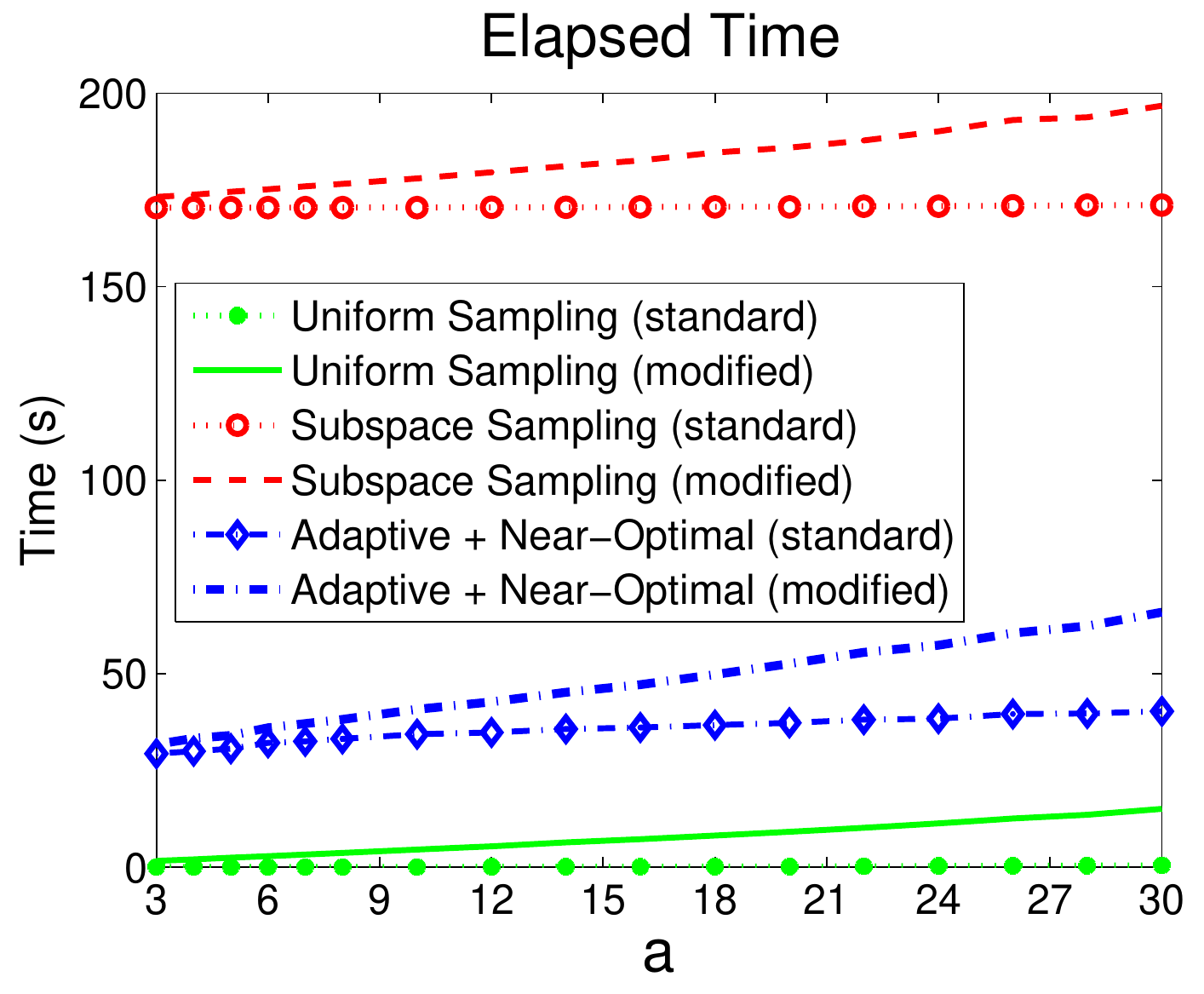}
\includegraphics[width=50mm,height=45mm]{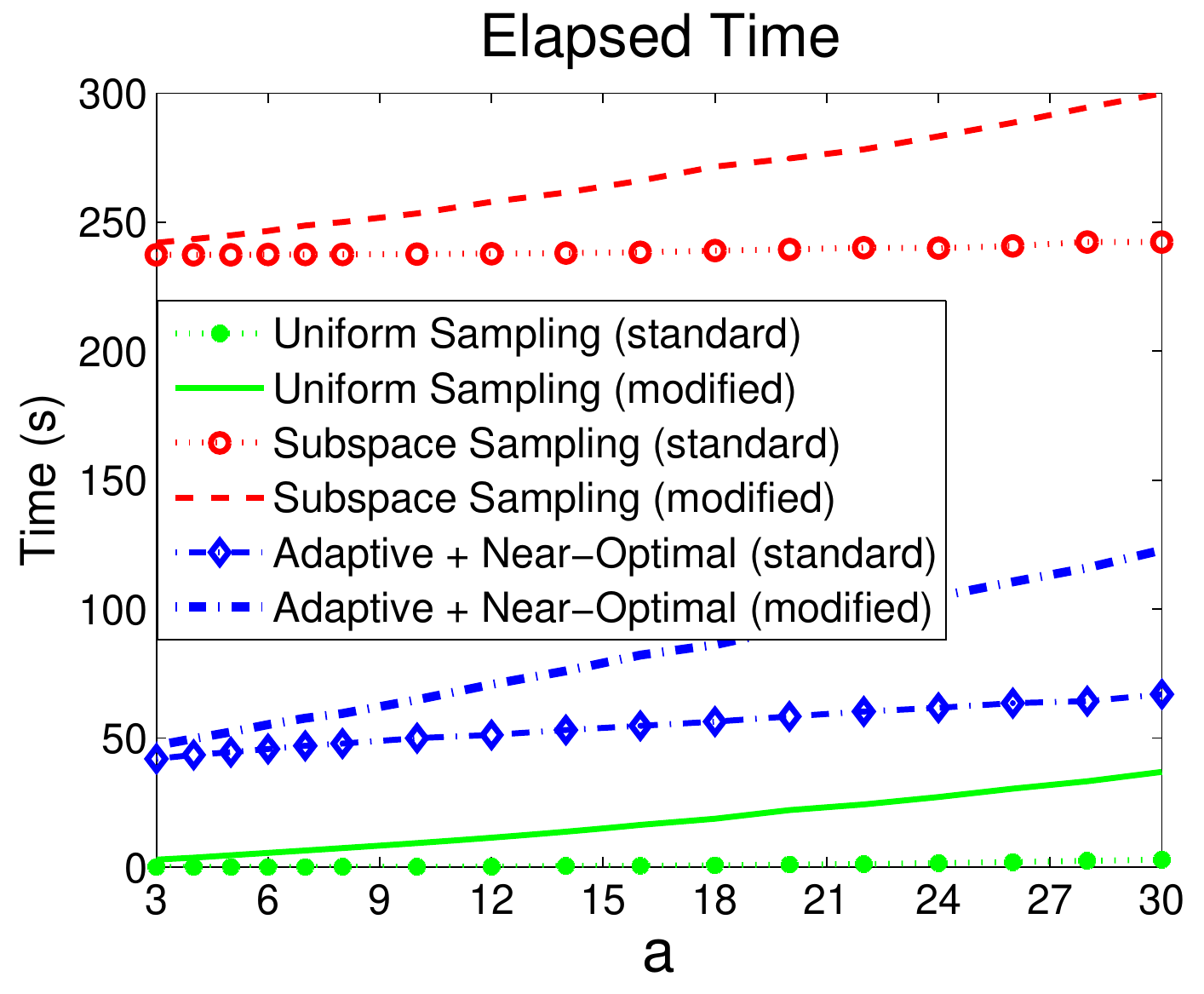}
\includegraphics[width=50mm,height=45mm]{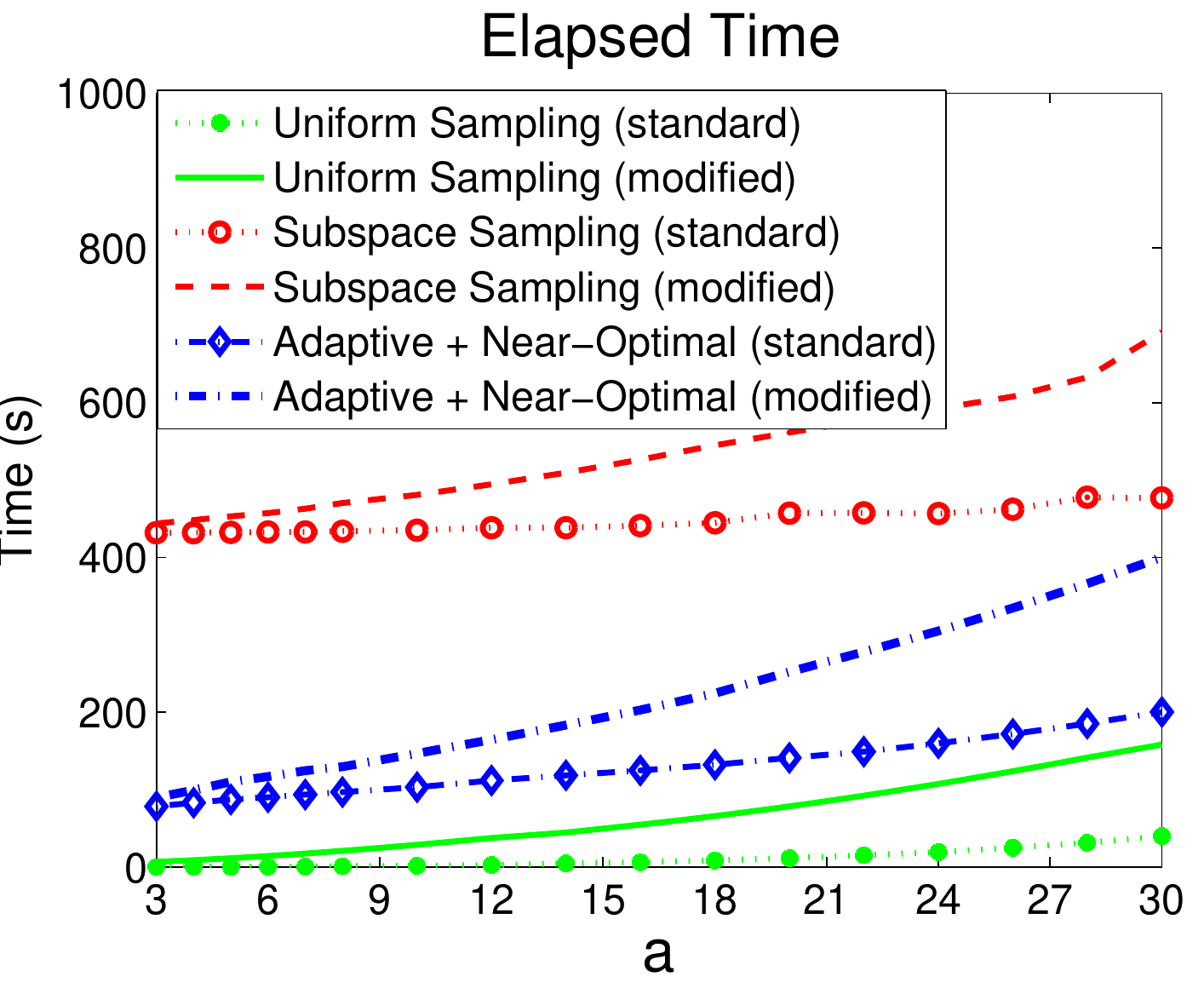}\\
\subfigure[\textsf{$\sigma = 0.2$, $k = 10$, and $c=a k$.}]{\includegraphics[width=50mm,height=45mm]{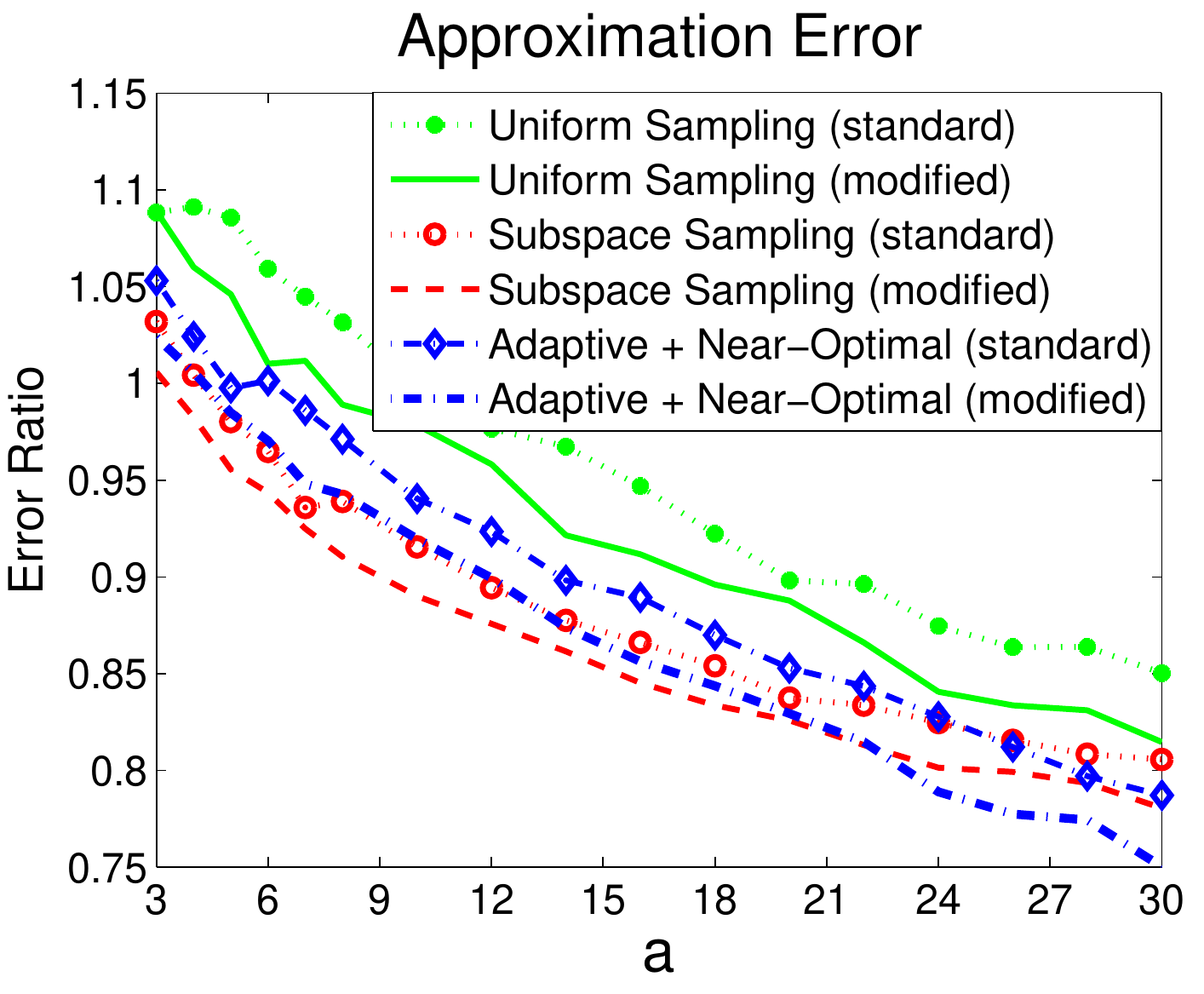}}
\subfigure[\textsf{$\sigma = 0.2$, $k = 20$, and $c=a k$.}]{\includegraphics[width=50mm,height=45mm]{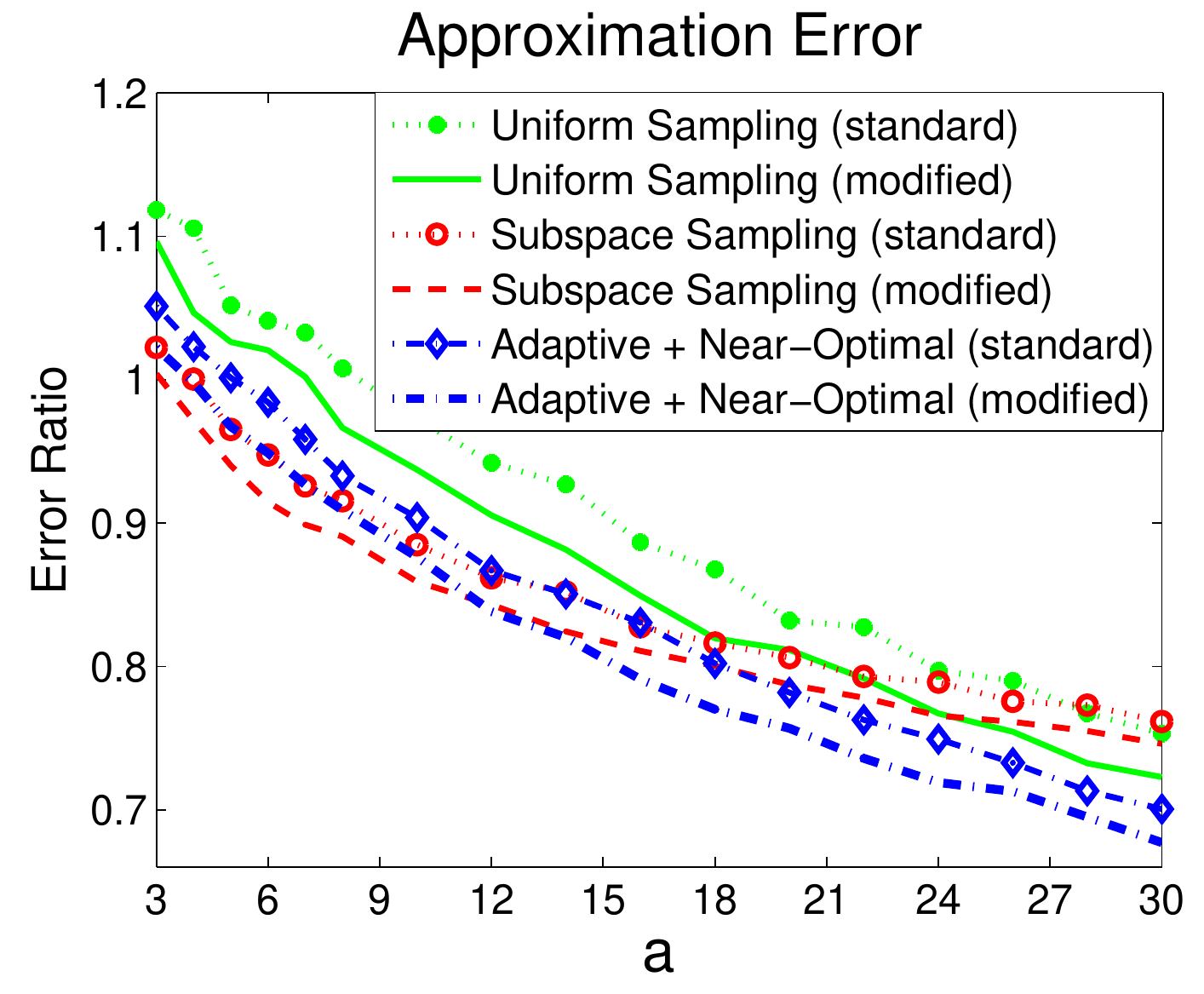}}
\subfigure[\textsf{$\sigma = 0.2$, $k = 50$, and $c=a k$.}]{\includegraphics[width=50mm,height=45mm]{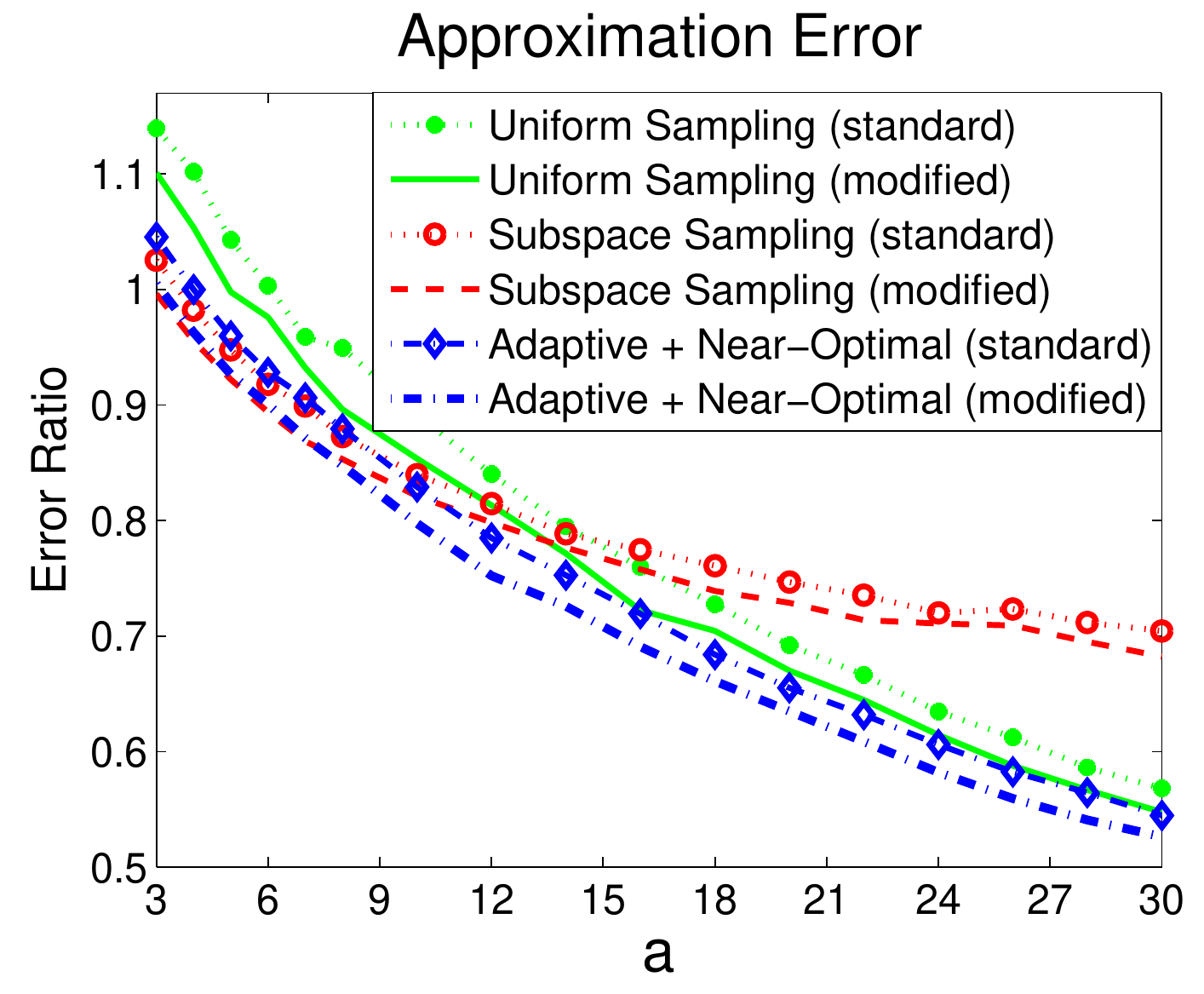}}\\

\vspace{8mm}

\includegraphics[width=50mm,height=45mm]{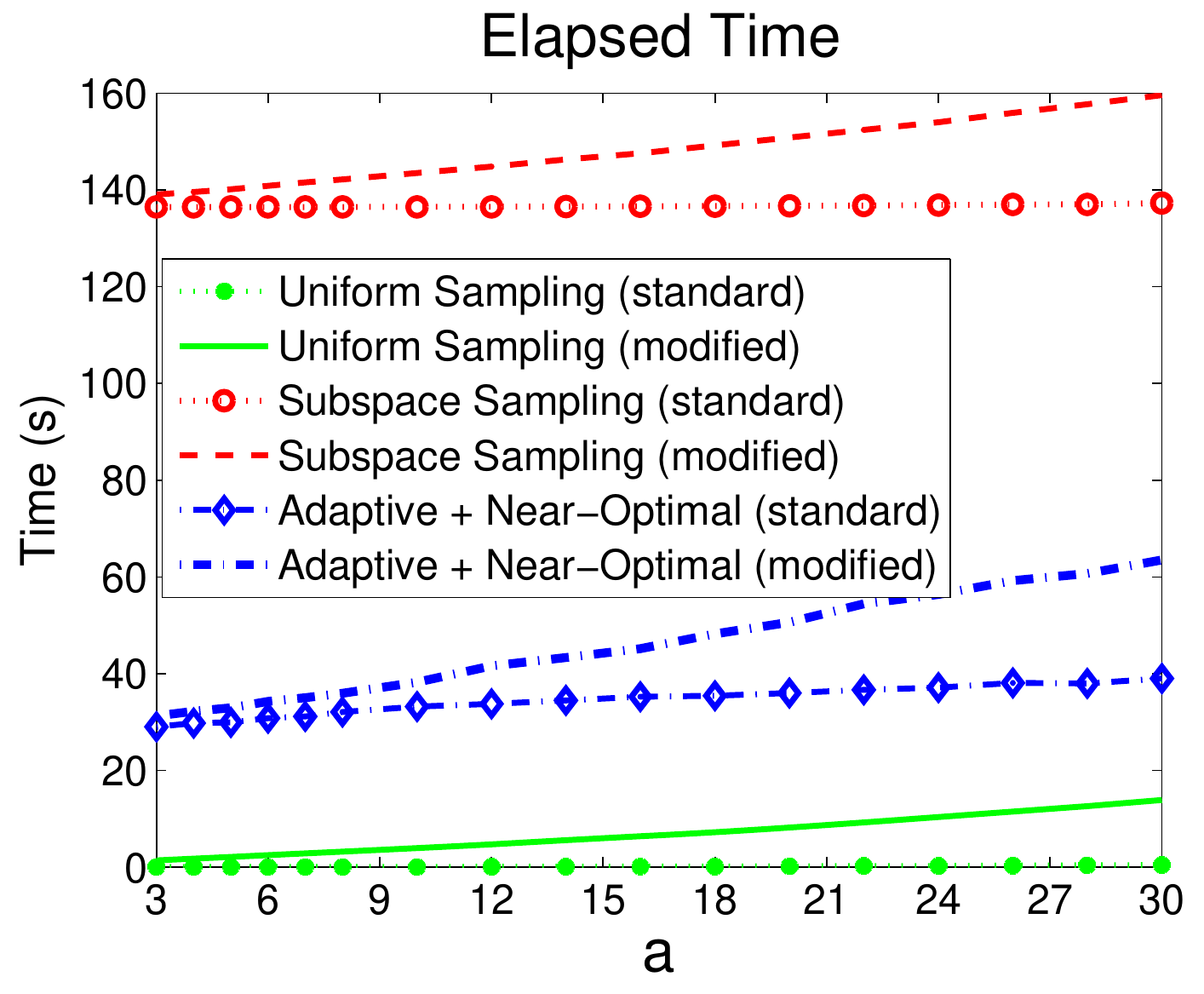}
\includegraphics[width=50mm,height=45mm]{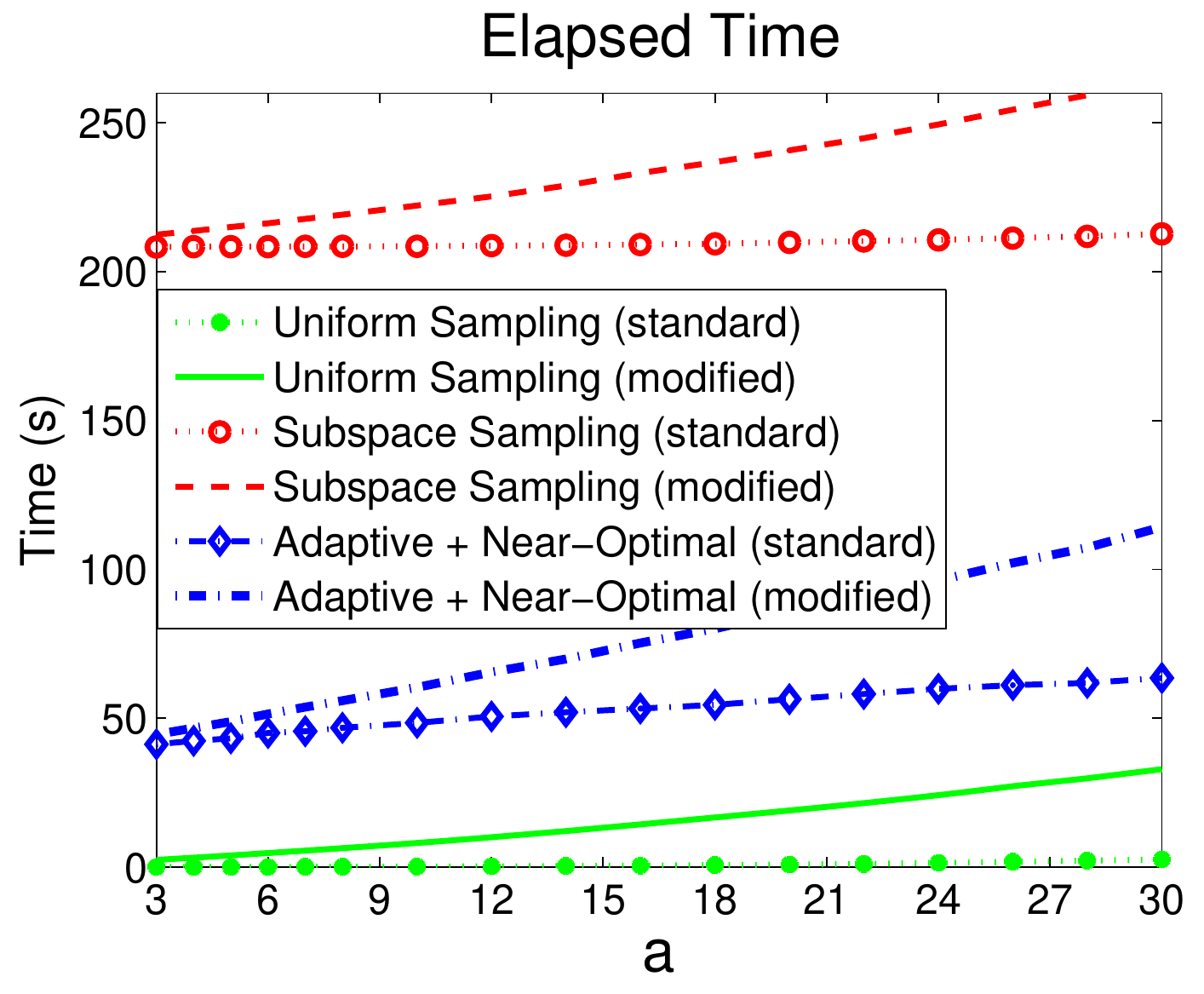}
\includegraphics[width=50mm,height=45mm]{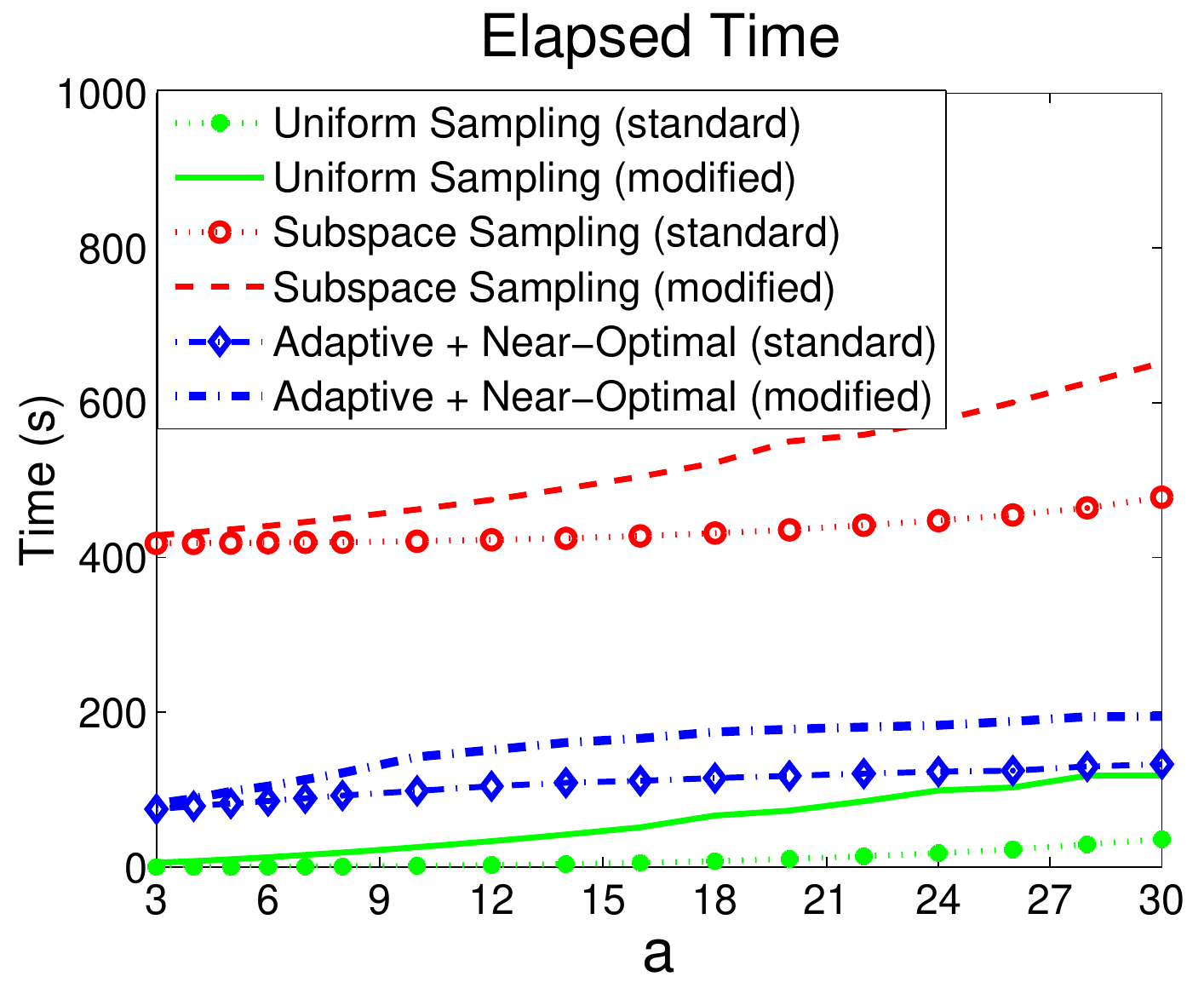}\\
\subfigure[\textsf{$\sigma = 1$, $k = 10$, and $c=a k$.}]{\includegraphics[width=50mm,height=45mm]{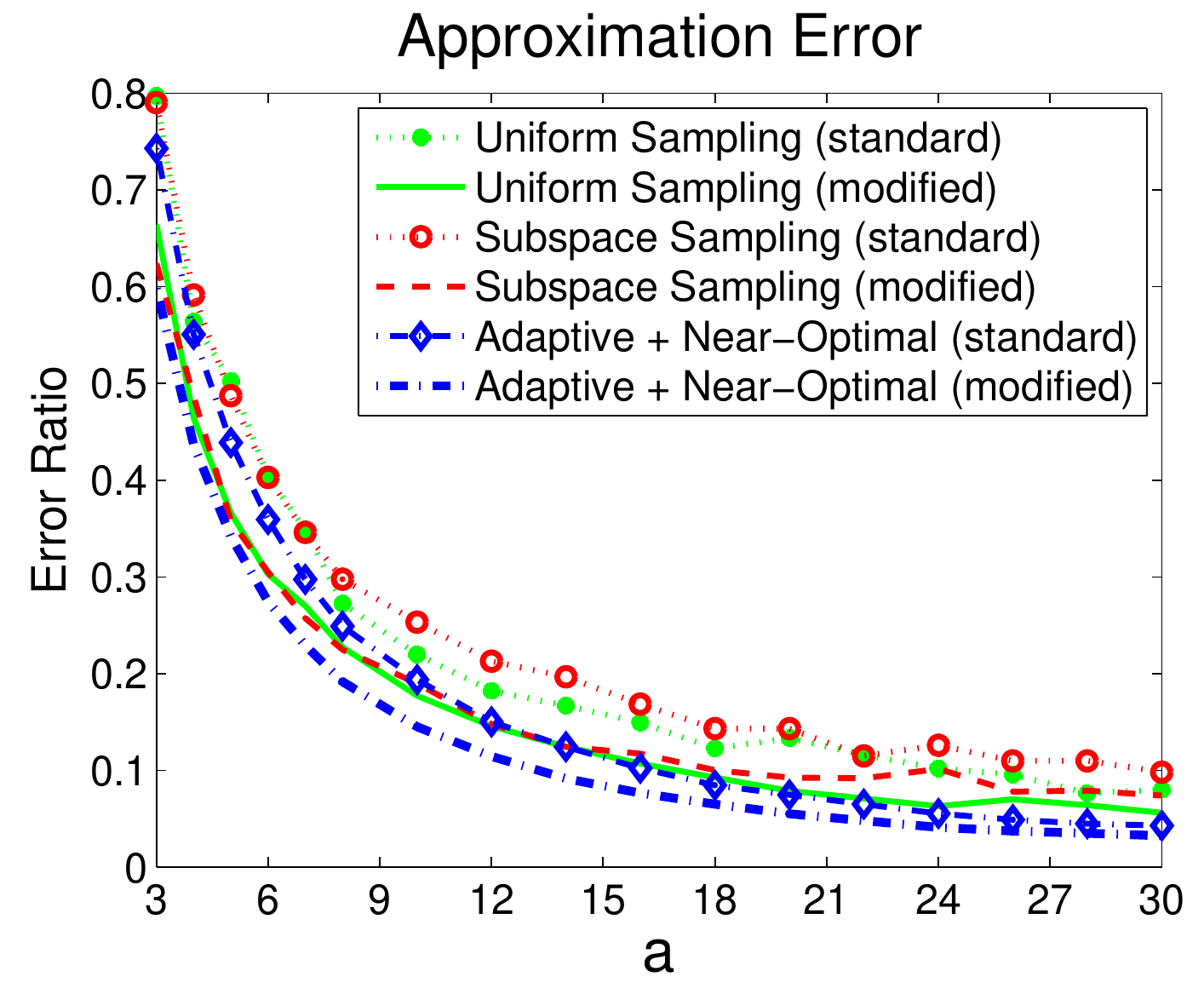}}
\subfigure[\textsf{$\sigma = 1$, $k = 20$, and $c=a k$.}]{\includegraphics[width=50mm,height=45mm]{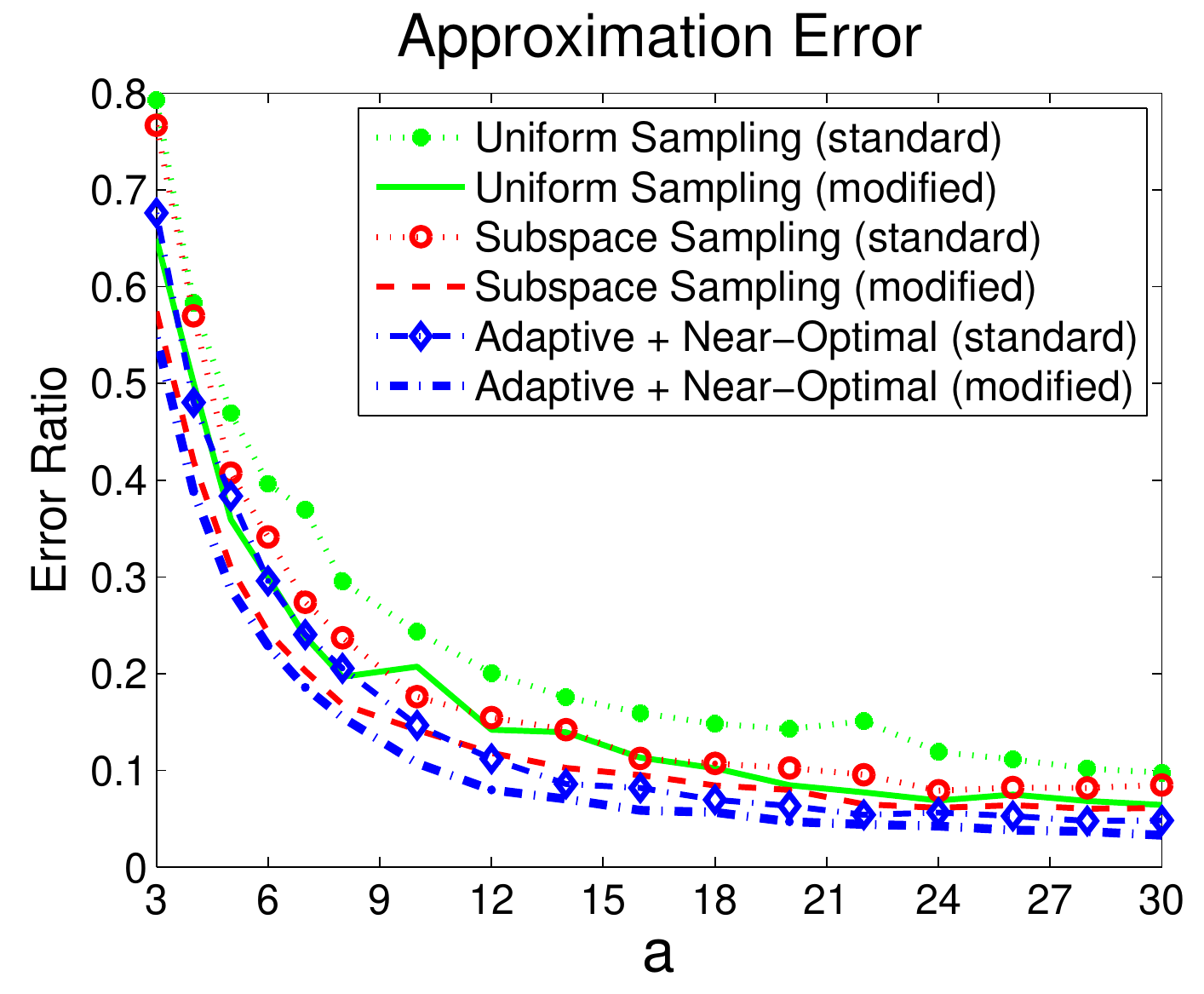}}
\subfigure[\textsf{$\sigma = 1$, $k = 50$, and $c=a k$.}]{\includegraphics[width=50mm,height=45mm]{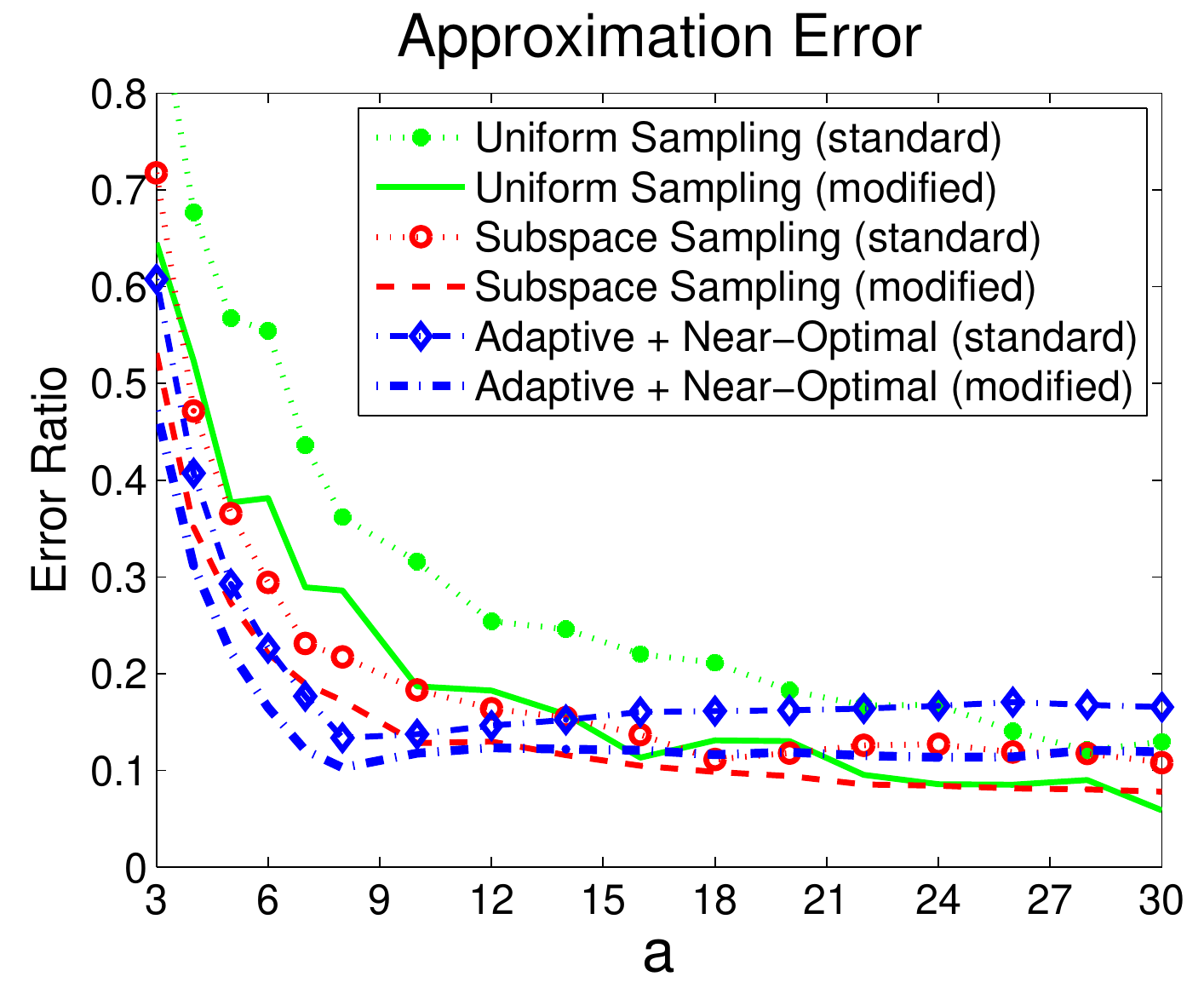}}
\end{center}
   \caption{Results of the \nystrom algorithms on the RBF kernel in the Wine Quality data set.}
\label{fig:wine}
\end{figure*}

\begin{figure*}
\subfigtopskip = 0pt
\begin{center}
\centering
\includegraphics[width=50mm,height=45mm]{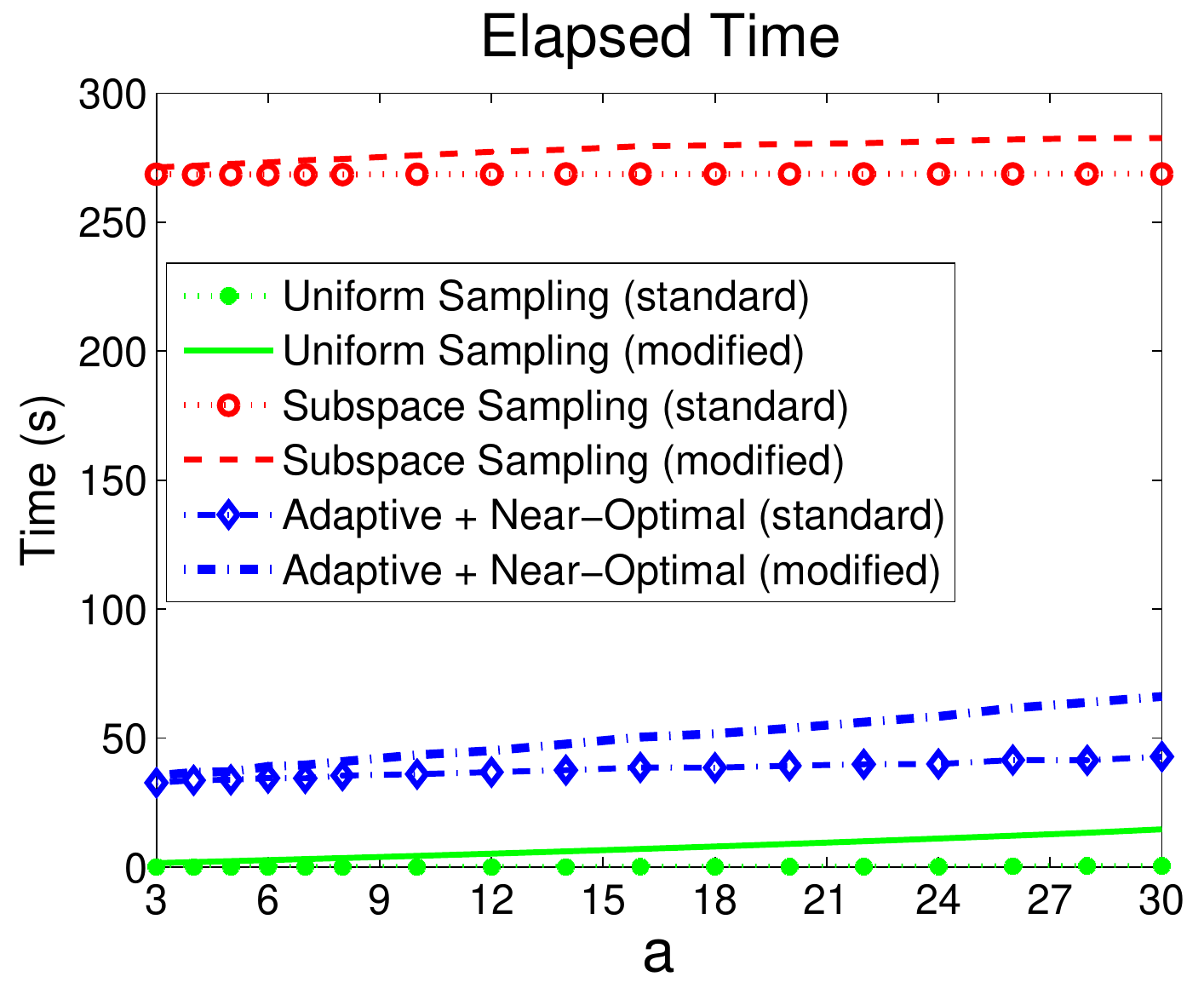}
\includegraphics[width=50mm,height=45mm]{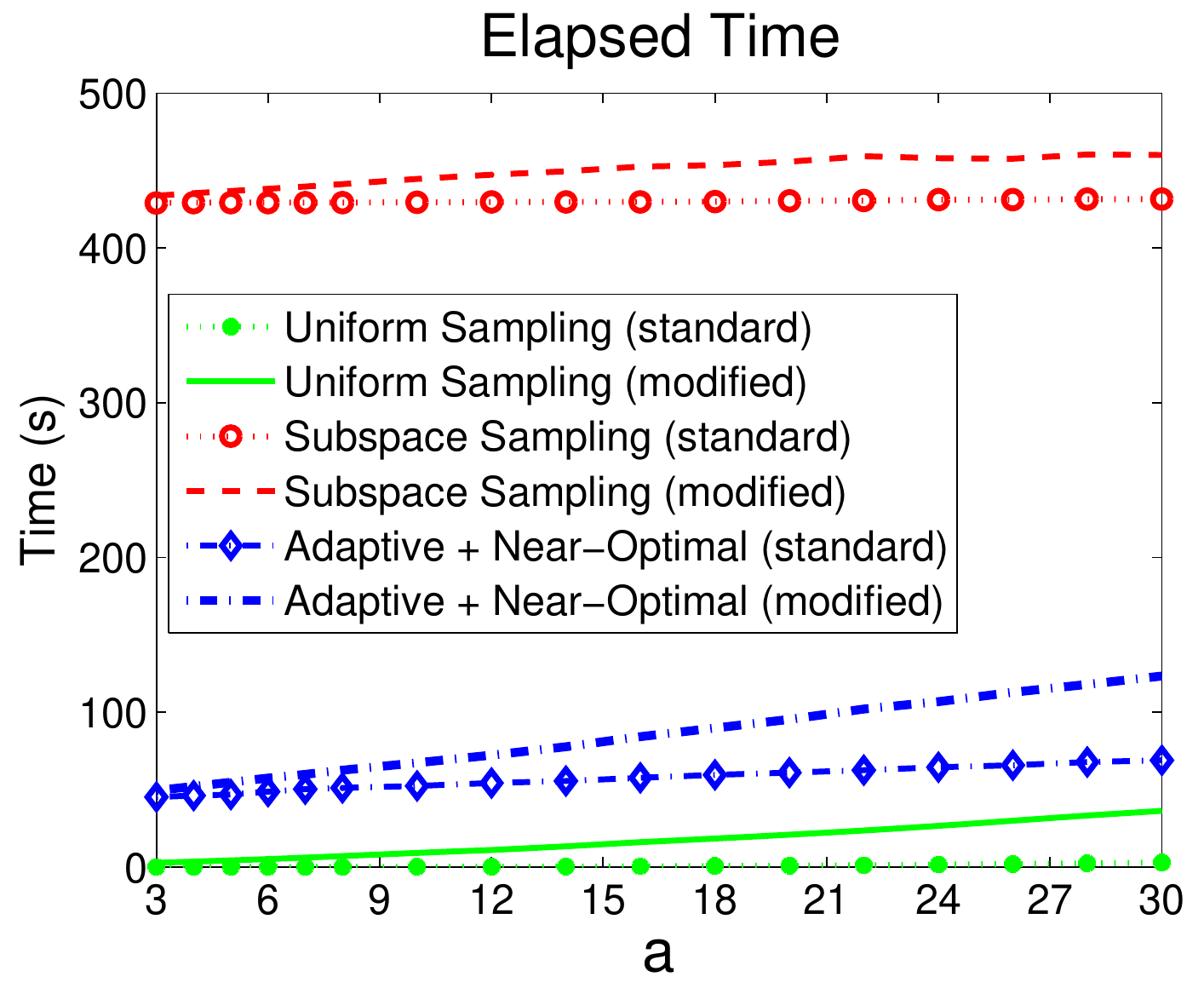}
\includegraphics[width=50mm,height=45mm]{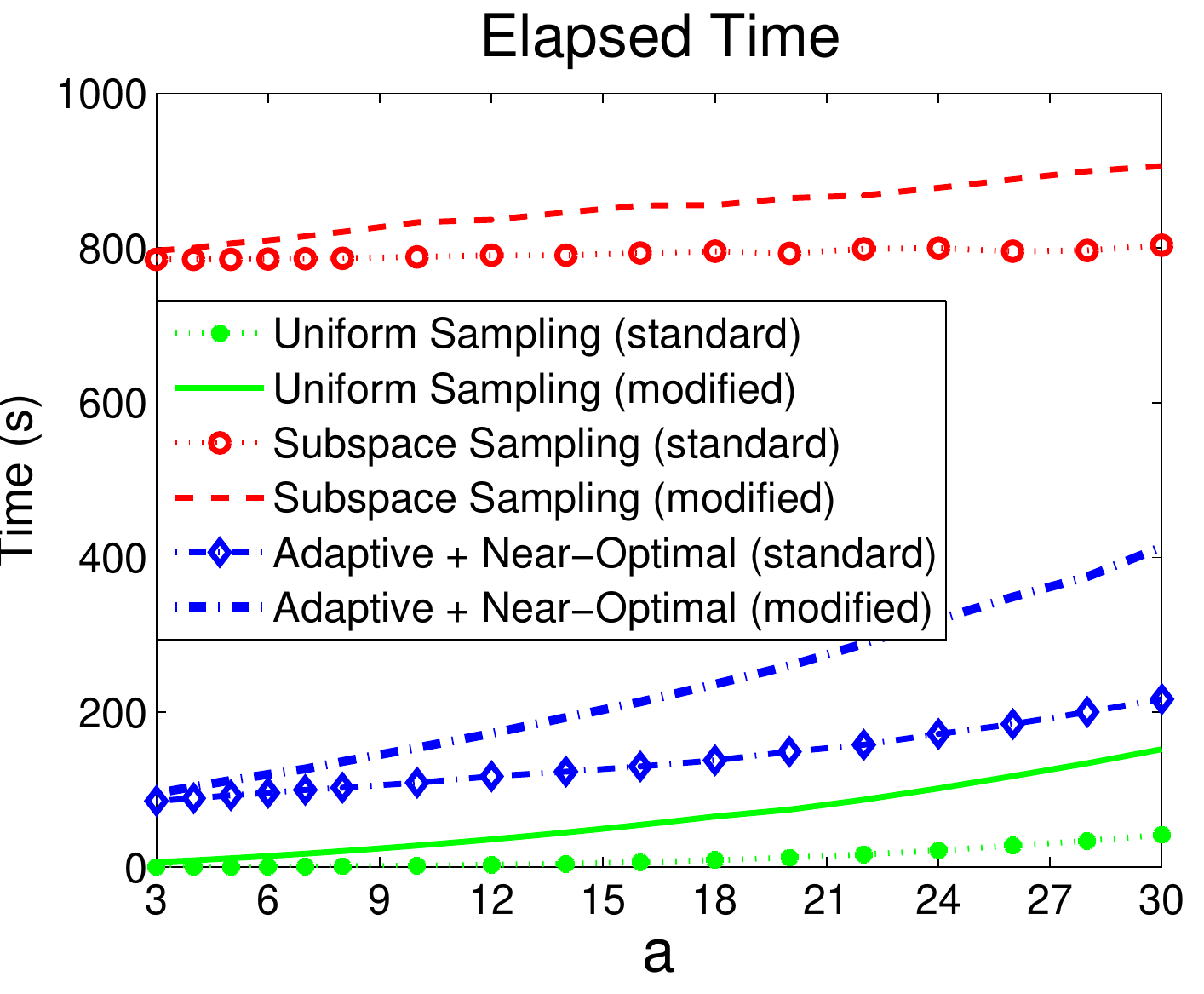}\\
\subfigure[\textsf{$\sigma = 0.2$, $k = 10$, and $c=a k$.}]{\includegraphics[width=50mm,height=45mm]{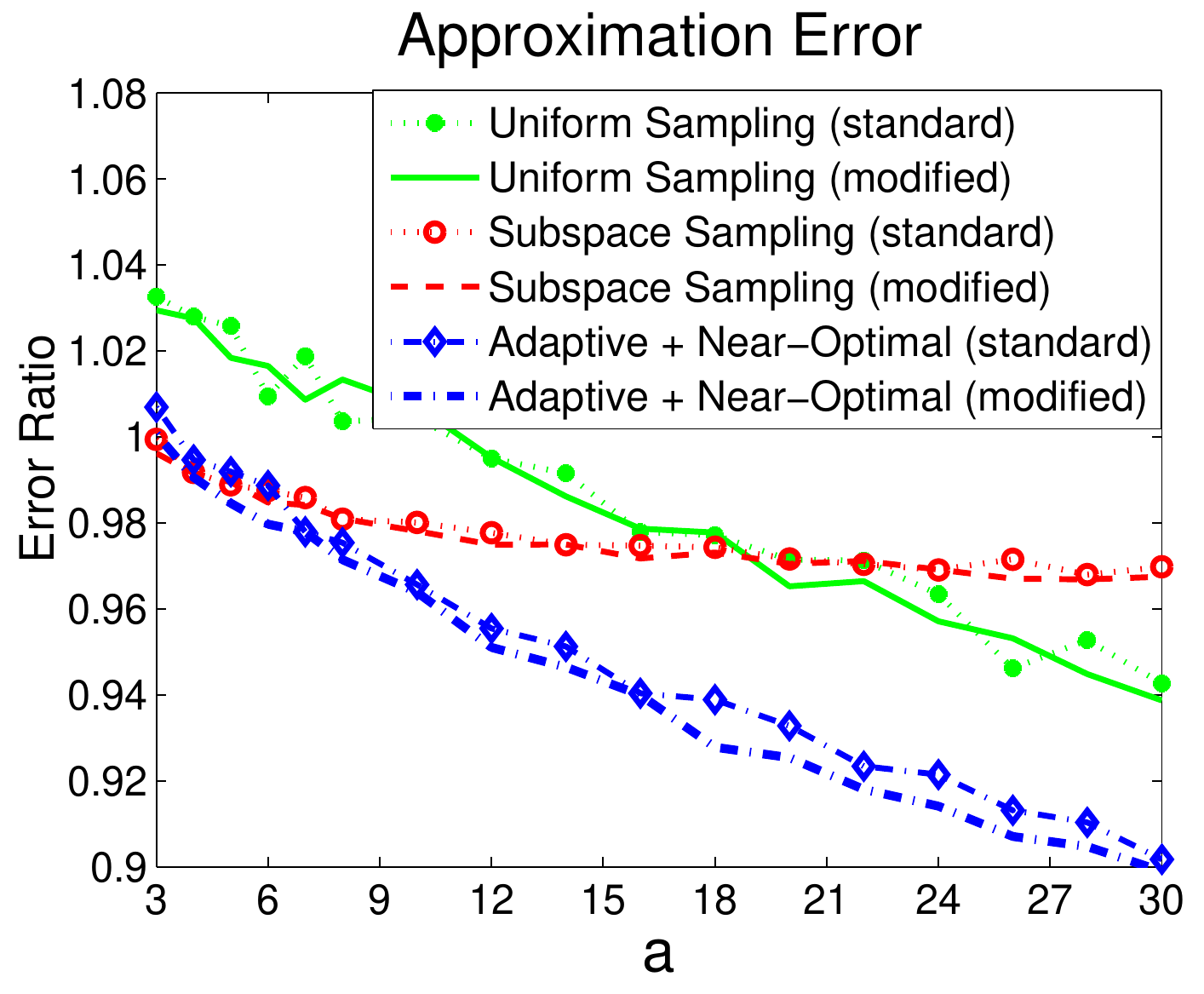}}
\subfigure[\textsf{$\sigma = 0.2$, $k = 20$, and $c=a k$.}]{\includegraphics[width=50mm,height=45mm]{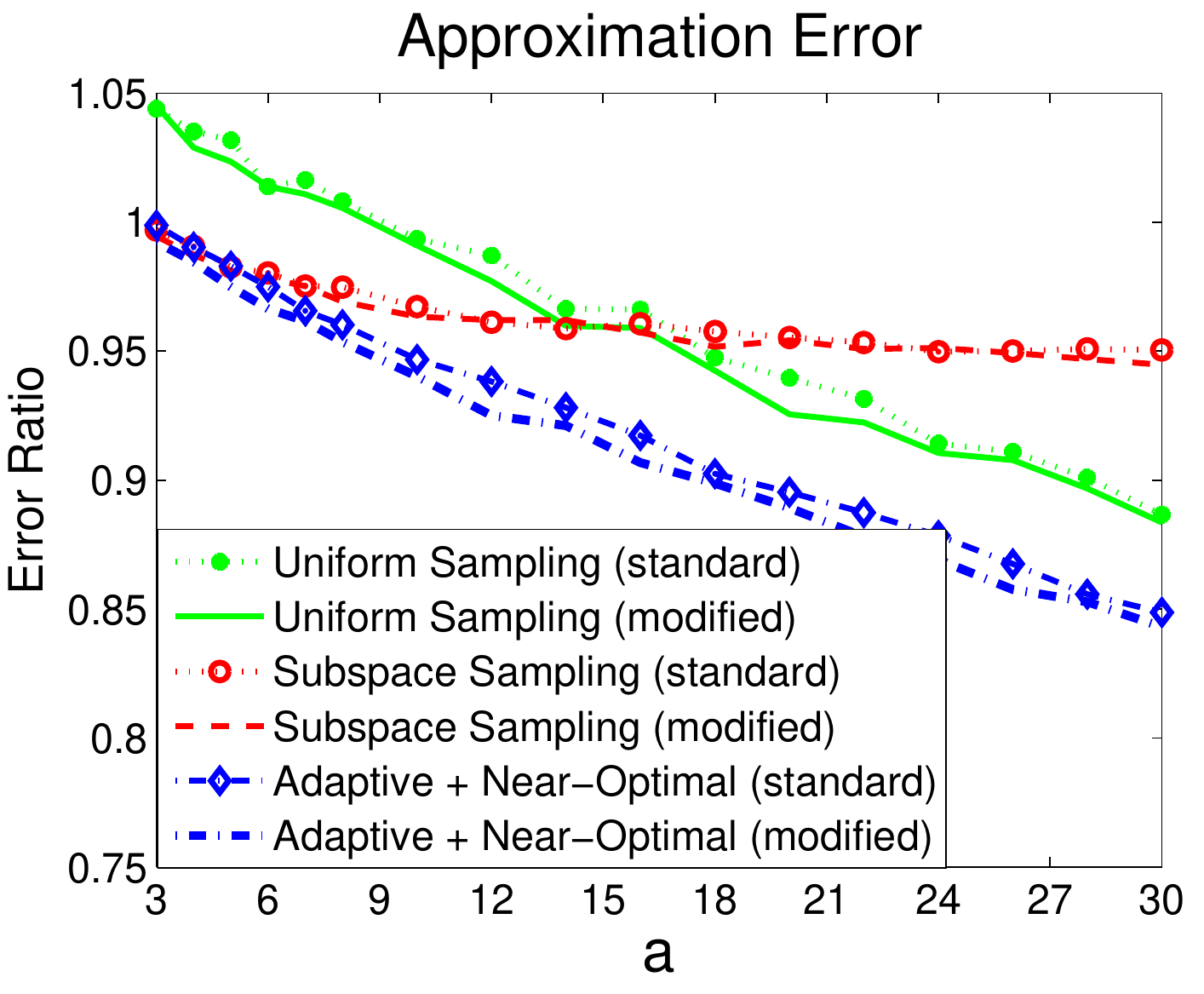}}
\subfigure[\textsf{$\sigma = 0.2$, $k = 50$, and $c=a k$.}]{\includegraphics[width=50mm,height=45mm]{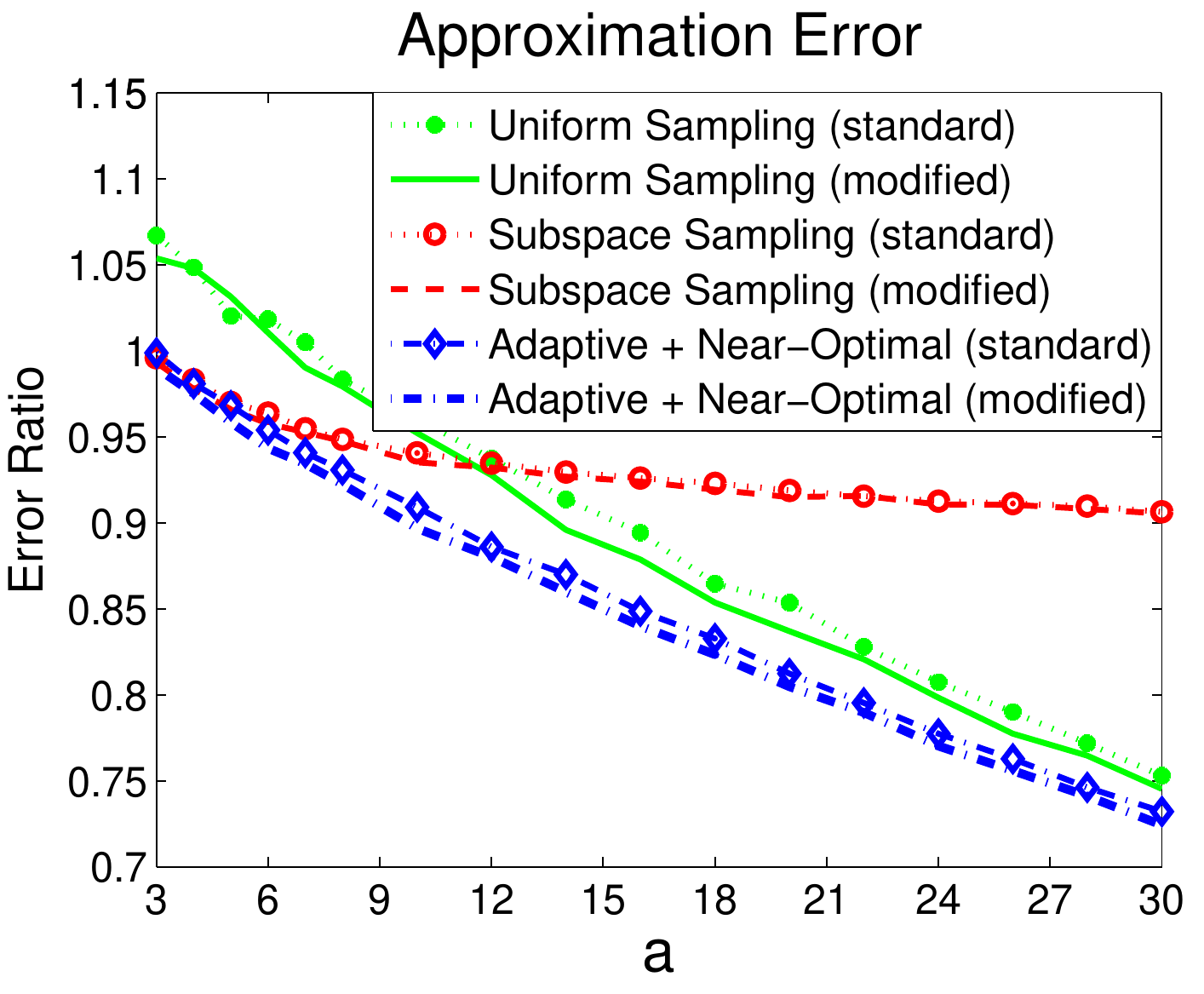}}\\

\vspace{8mm}

\includegraphics[width=50mm,height=45mm]{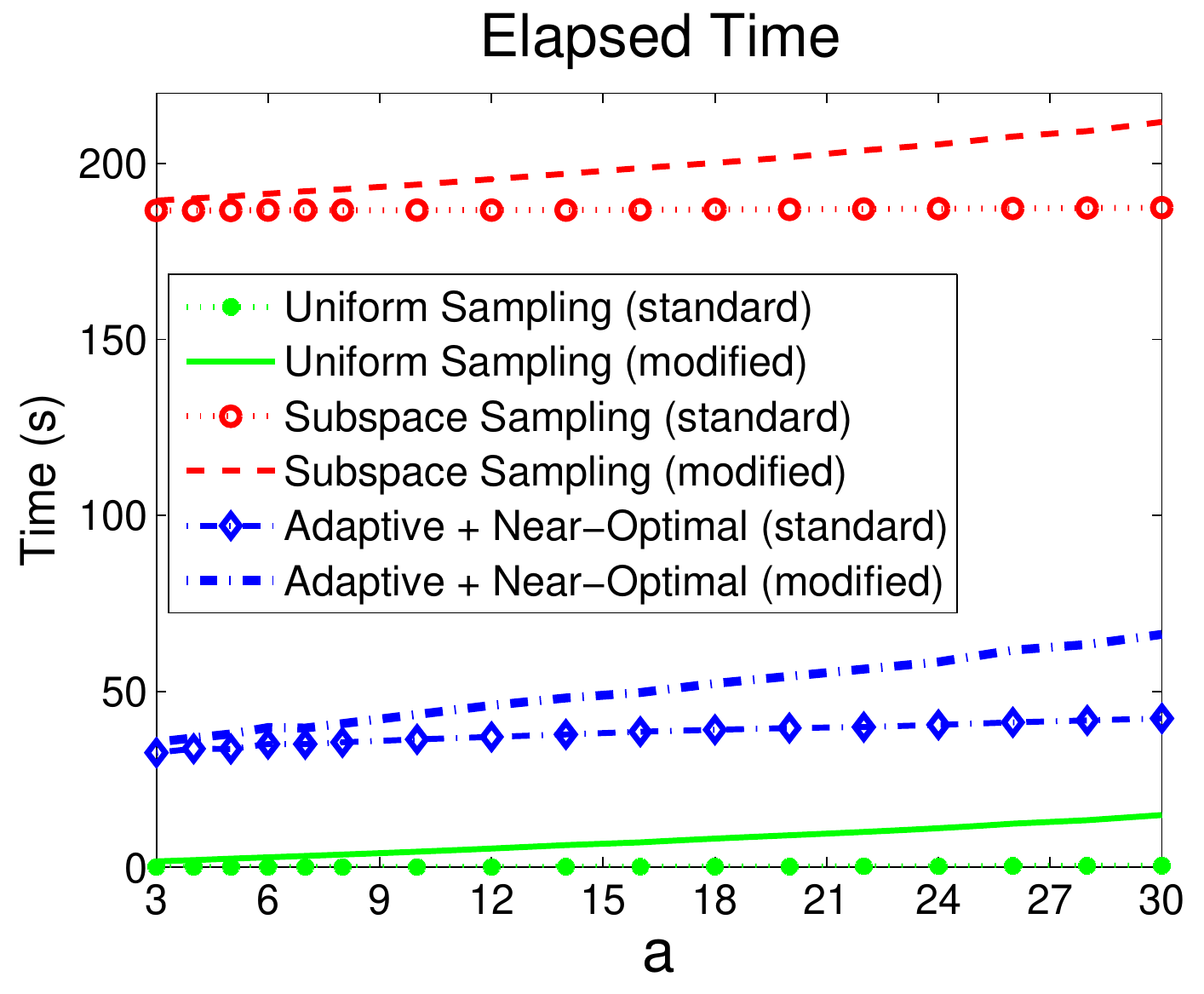}
\includegraphics[width=50mm,height=45mm]{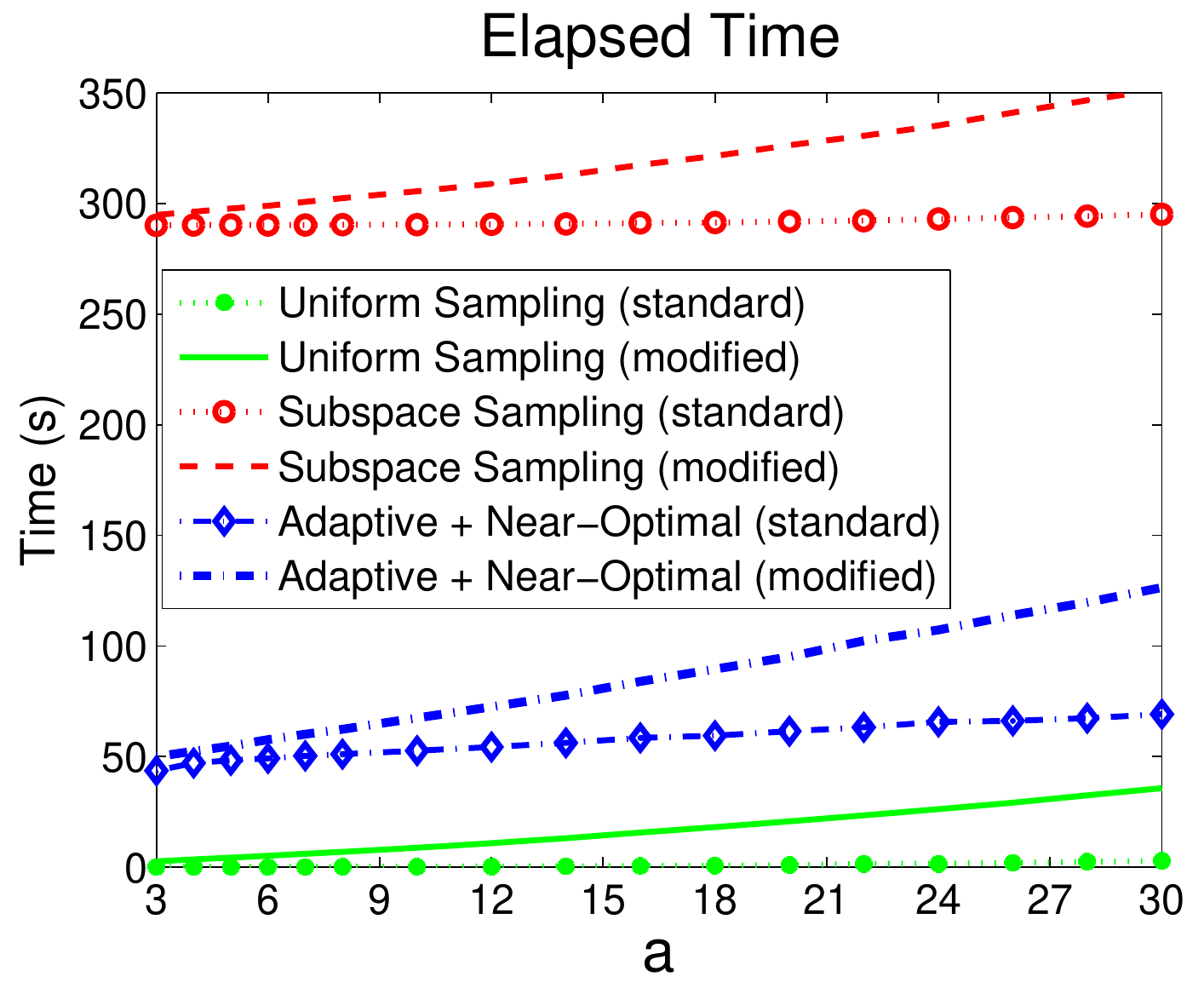}
\includegraphics[width=50mm,height=45mm]{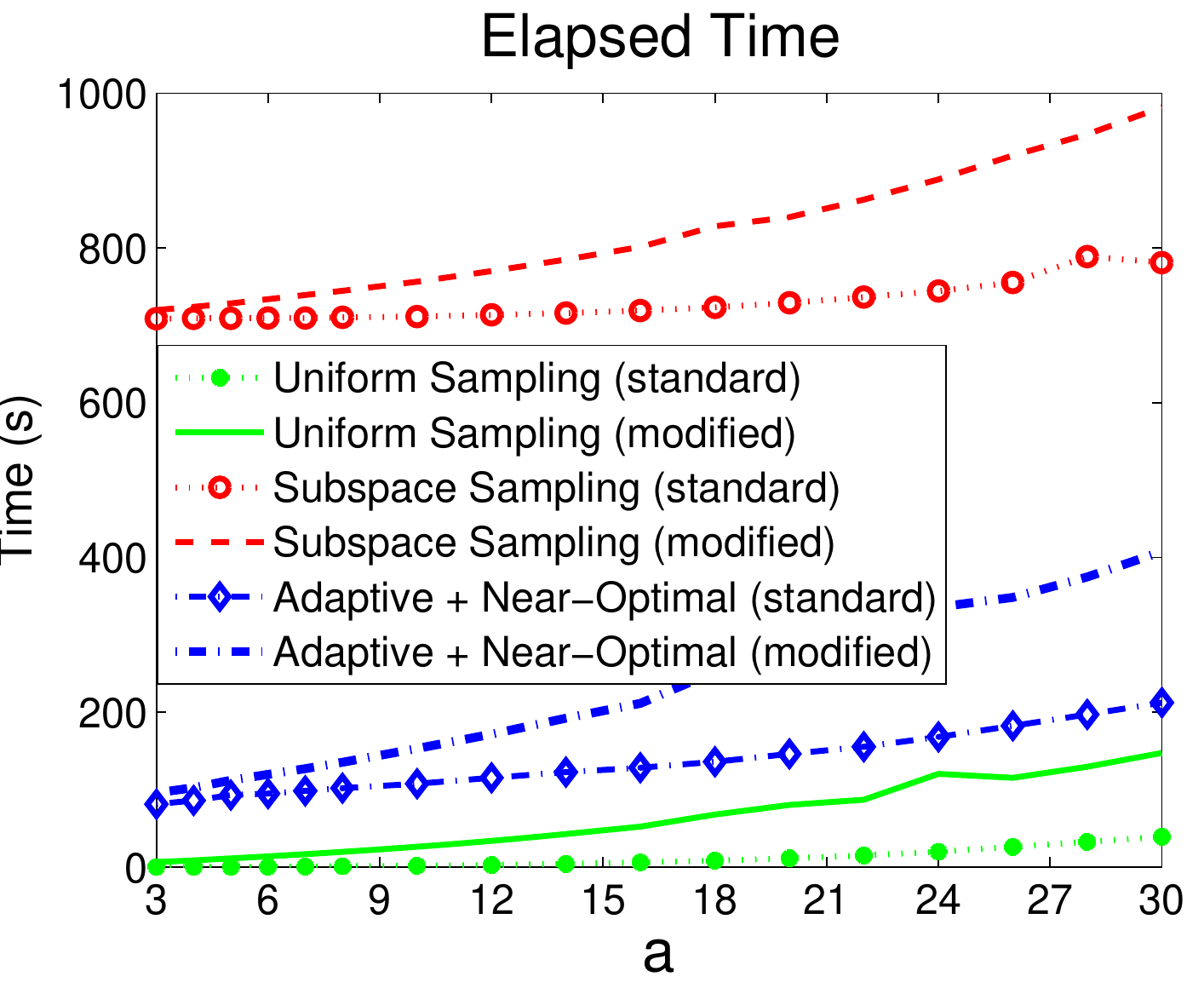}\\
\subfigure[\textsf{$\sigma = 1$, $k = 10$, and $c=a k$.}]{\includegraphics[width=50mm,height=45mm]{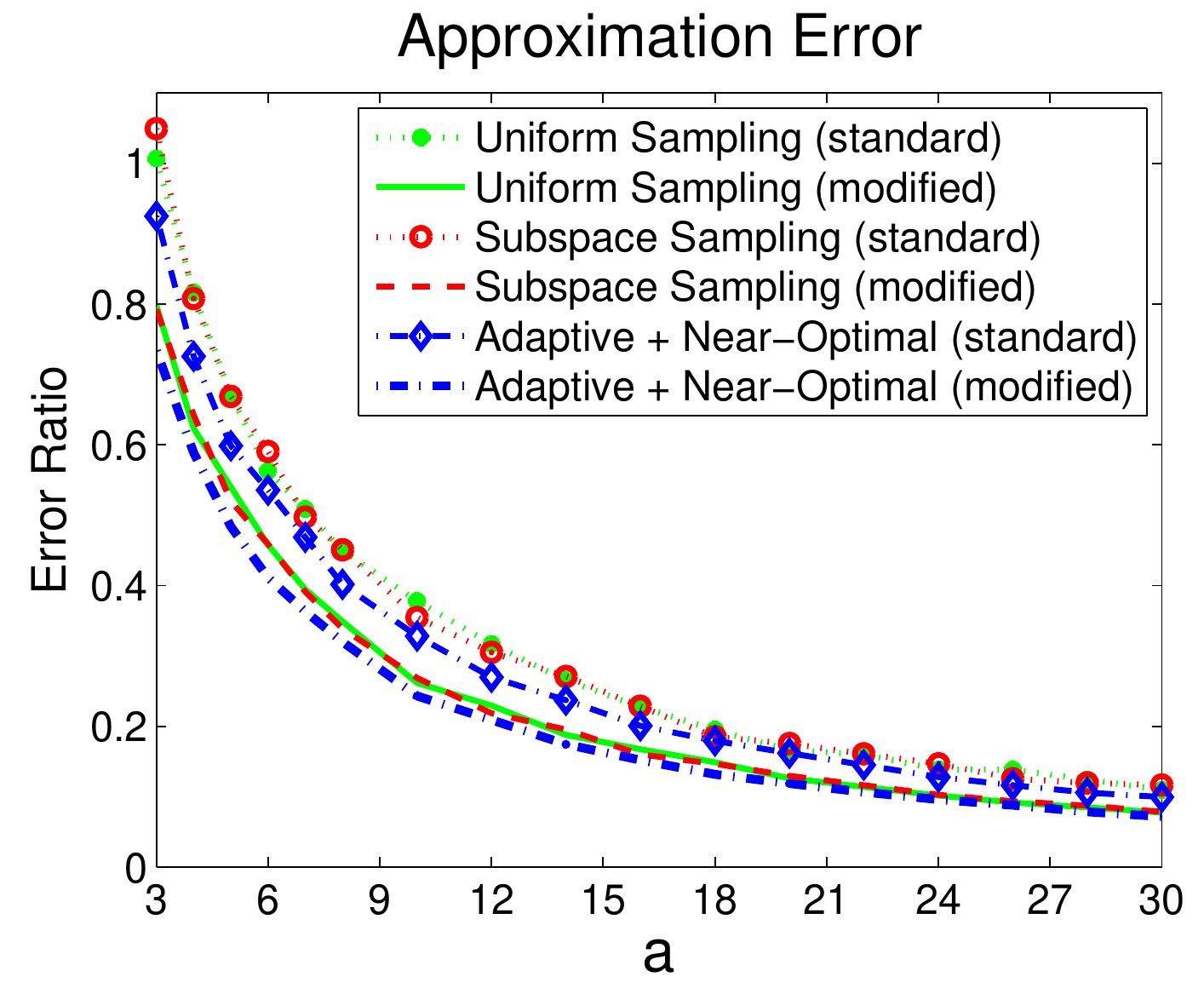}}
\subfigure[\textsf{$\sigma = 1$, $k = 20$, and $c=a k$.}]{\includegraphics[width=50mm,height=45mm]{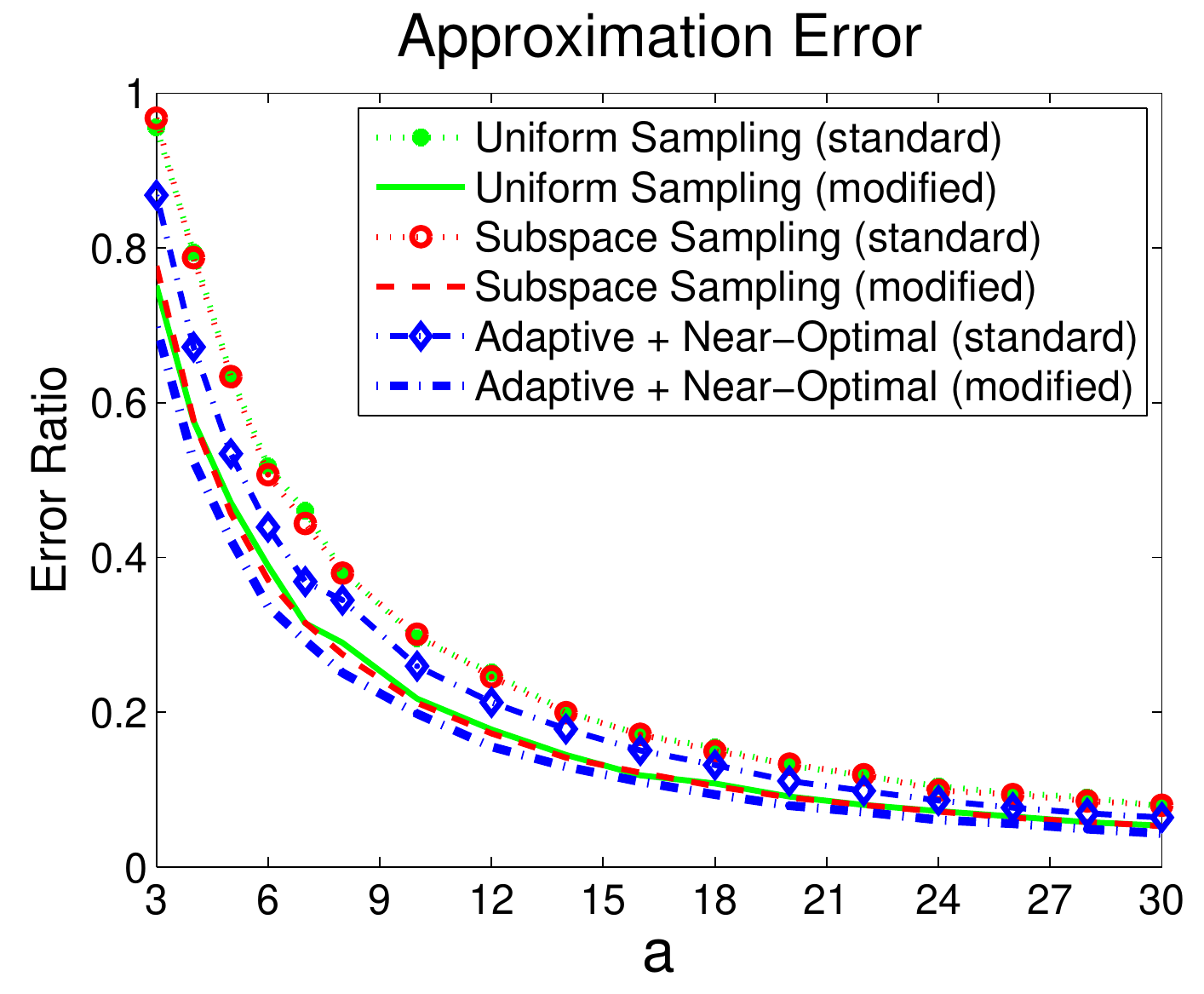}}
\subfigure[\textsf{$\sigma = 1$, $k = 50$, and $c=a k$.}]{\includegraphics[width=50mm,height=45mm]{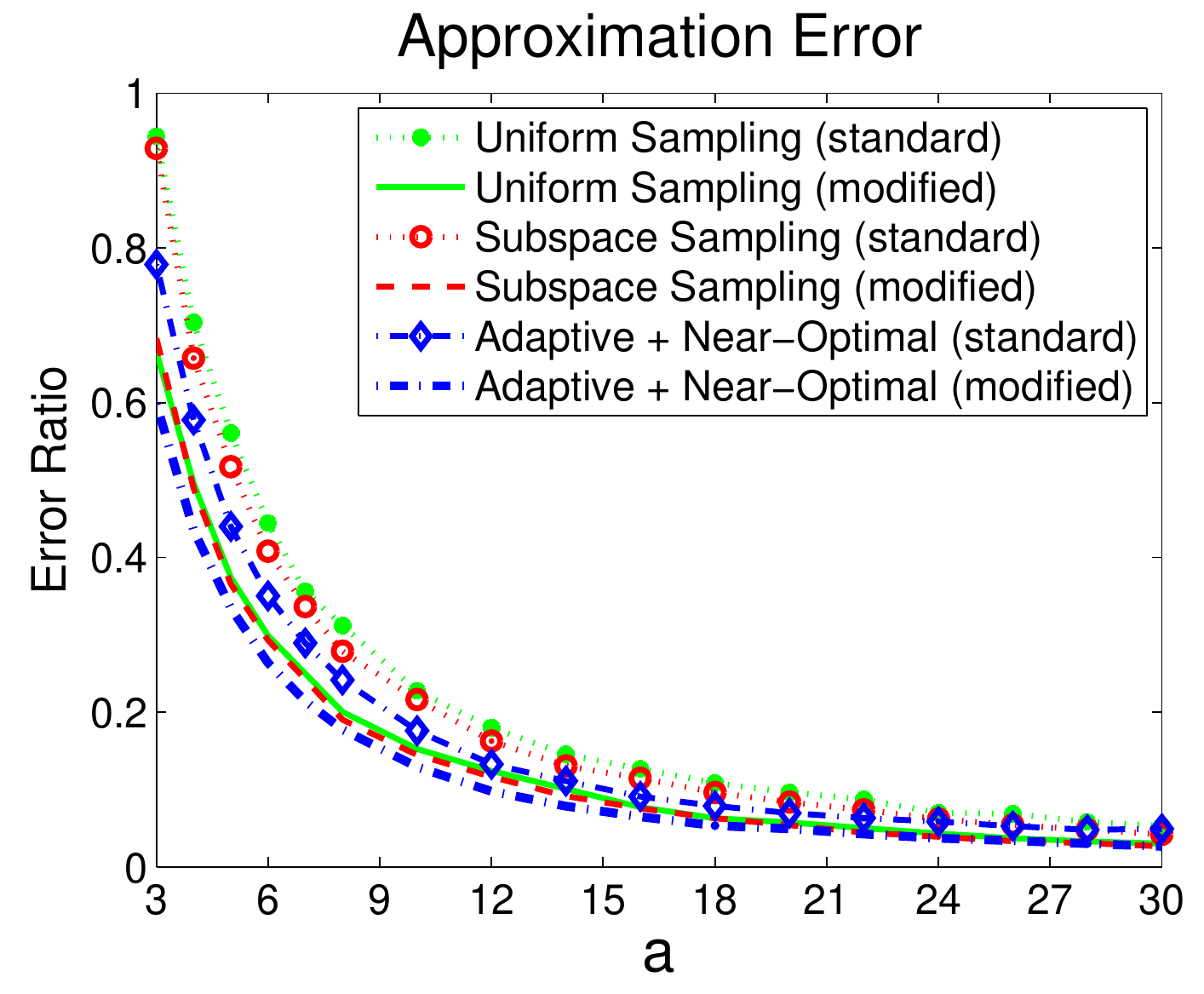}}
\end{center}
   \caption{Results of the \nystrom algorithms on the RBF kernel in the Letters data set.}
\label{fig:letter}
\end{figure*}

The experimental results also show that the subspace sampling and adaptive sampling algorithms significantly outperform
the uniform sampling when $c$ is reasonably small, say $c < 10 k$.
This indicates that the subspace sampling and adaptive sampling algorithms are good at choosing ``good" columns as basis vectors.
The effect is especially evident on the RBF kernel with the scale parameter $\sigma = 0.2$,
where the leverage scores are heterogeneous.
In most cases our adaptive sampling algorithm achieves the lowest approximation error among the three algorithms.
The error ratios of our adaptive sampling for the modified \nystrom are in accordance with the theoretical bound in Theorem~\ref{thm:nystrom_bound}; that is,
\[
\frac{\| \A - \C \U \C^T \|_F }{\|\A - \A_k\|_F}
\; \leq \; 1 + \sqrt{\frac{2k}{c}\big( 1+o(1) \big)}
\; = \; 1 + \sqrt{\frac{2}{a}\big( 1+o(1) \big)} \textrm{.}
\]
As for the running time, our adaptive sampling algorithm is more  efficient than the subspace sampling algorithm.
This is partly because the RBF kernel matrix is dense, and hence the subspace sampling algorithm costs $\OM(m^2 k)$ time to compute the truncated SVD.

Furthermore, the experimental results show that using $\U = \C^\dag \A (\C^\dag)^T$ as the intersection matrix (denoted by ``modified" in the figures)
always leads to much lower error than using $\U = \W^\dag$  (denoted by ``standard").
However, our modified \nystrom method costs more time to compute the intersection matrix than the standard \nystrom method costs.
Recall that the standard \nystrom costs $\OM(c^3)$ time to compute $\U = \W^\dag$ and
that the modified \nystrom costs $\OM(m c^2) + T_{\textrm{Multiply}} (m^2 c)$ time to compute $\U = \C^\dag \A (\C^\dag)^T$.
So the users should make a trade-off between time and accuracy
and decide whether it is worthwhile to sacrifice extra computational overhead for the improvement in accuracy by using the modified Nystr\"{o}m method.

\section{Conclusion}
\label{sec:concl}

In this paper we have built a novel and more general relative-error bound for the adaptive sampling algorithm. Accordingly,
we have devised novel CUR matrix decomposition and  \nystrom approximation algorithms which demonstrate significant improvement over
the classical counterparts.
Our relative-error CUR algorithm requires only $c = {2k}{\epsilon^{-1}}(1+o(1))$ columns and
$r = {c}{\epsilon^{-1}}(1 {+}\epsilon)$ rows selected from the original matrix.
To achieve relative-error bound,
the best previous algorithm---the subspace sampling algorithm---requires $c = \OM(k \epsilon^{-2} \log k)$ columns and
$r = \OM(c \epsilon^{-2} \log c )$ rows.
Our modified \nystrom method is different from the conventional \nystrom methods in that it uses a different intersection matrix.
We have shown that our adaptive sampling algorithm for the modified \nystrom achieves relative-error upper bound
by sampling only $c = {2k}{\epsilon^{-2}}(1{+}o(1))$ columns,
which even beats the lower error bounds of the standard \nystrom and the ensemble Nystr\"{o}m.
Our proposed CUR and \nystrom algorithms are scalable because they need only to maintain a small fraction
of columns or rows in RAM, and their time complexities are low provided that matrix multiplication can be highly efficiently executed.
Finally, the empirical comparison has also demonstrated the effectiveness and efficiency of our algorithms.

\acks{This work has been supported in part by the Natural Science Foundations of China (No. 61070239)
and the Scholarship Award for Excellent Doctoral Student granted by Chinese Ministry of Education.}

\appendix

\section{The Dual Set Sparsification Algorithm} \label{sec:dualset}

For the sake of self-contained, we attach the dual set sparsification algorithm  and describe some implementation details.
The deterministic dual set sparsification algorithm is established by \cite{boutsidis2011NOC}
and severs as an important step in the near-optimal column selection algorithm
(described in Lemma~\ref{prop:fast_column_selection} and Algorithm~\ref{alg:near_optimal_col} in this paper).
We show the dual set sparsification algorithm algorithm in Algorithm~\ref{alg:dual_set_sparsification}
and its bounds in Lemma~\ref{lem:sparsification},
and we also analyze the time complexity using our defined notation.

\begin{lemma}[Dual Set Spectral-Frobenius Sparsification] \label{lem:sparsification}
Let $\UM = \{\x_1, \cdots, \x_n\} \subset \RB^l$ $(l < n)$ contain the columns of an arbitrary matrix $\X \in \RB^{l\times n}$.
Let $\VM = \{\v_1, \cdots, \v_n\} \subset \RB^k$ $(k < n)$ be a decompositions of the identity,
that is, $\sum_{i=1}^n \v_i \v_i^T = \I_k$.
Given an integer $r$ with $k < r < n$,
Algorithm~\ref{alg:dual_set_sparsification} deterministically computes
a set of weights $s_i \geq 0$ ($i = 1,\cdots, n$) at most $r$ of which are non-zero, such that
\[
\lambda_k \Big( \sum_{i=1}^n s_i \v_i \v_i^T \Big) \geq \Big( 1 - \sqrt{\frac{k}{r}} \Big)^2 \qquad \mbox{ and }  \qquad
\tr \Big( \sum_{i=1}^n s_i \x_i \x_i^T \Big) \leq \|\X\|_F^2.
\]
The weights $s_i$  can be computed deterministically in $\OM \big(rnk^2\big) + \TimeMulti\big(n l \big)$ time.
\end{lemma}

Here we mention some implementation issues of Algorithm~\ref{alg:dual_set_sparsification}
which were not described in detail by \cite{boutsidis2011NOC}.
In each iteration the algorithm performs once eigenvalue decomposition: $\A_\tau = \W \Lam \W^T$.
Here $\A_\tau$ is guaranteed to be SPSD in each iteration.
Since
\[
\Big(\A_\tau - \alpha \I_k \Big)^q
\; =\; \W \Diag\Big( (\lambda_1 - \alpha)^q , \cdots, (\lambda_k - \alpha)^q \Big) \W^T \textrm{,}
\]
$( \A_\tau - (L_\tau + 1) \I_k )^q$ can be efficiently computed based on the eigenvalue decomposition of $\A_\tau$.
With the eigenvalues at hand, $\phi (L, \A_\tau)$ can also be computed directly.

The algorithm runs in $r$ iterations.
In each iteration, the eigenvalue decomposition of $\A_\tau$ requires $\OM(k^3)$,
and the $n$ comparisons in Line~\ref{alg:dual_set_sparsification:4} each requires $\OM(k^2)$.
Moreover, computing $\|\x_i\|_2^2$ for each $\x_i$ requires $\TimeMulti(nl)$.
Overall, the running time of Algorithm~\ref{alg:dual_set_sparsification} is at most
$\OM(r k^3) + \OM(r n k^2) + \TimeMulti(nl) = \OM(rnk^2) + \TimeMulti(nl)$.

\begin{algorithm}[tb]
   \caption{Deterministic Dual Set Spectral-Frobenius Sparsification Algorithm.}
   \label{alg:dual_set_sparsification}
\algsetup{indent=2em}
\begin{small}
\begin{algorithmic}[1]
   \STATE {\bf Input:} 	$\UM = \{\x_i\}_{i=1}^n \subset \RB^l$, ($l < n$);
   							$\VM = \{\v_i\}_{i=1}^n \subset \RB^k$, with $\sum_{i=1}^n \v_i \v_i^T = \I_k$ ($k < n$);
   							$k < r < n$;
   \STATE {\bf Initialize:} $\s_0 = \0$, $\A_0 = \0$;
   \STATE Compute $\|\x_i\|_2^2$ for $i = 1, \cdots, n$, and then compute $\delta_U = \frac{ \sum_{i=1}^n \|\x_i\|_2^2 }{ 1 - \sqrt{k/r} }$;
   \FOR{$\tau = 0$ to $r-1$}
   \STATE Compute the eigenvalue decomposition of $\A_{\tau}$;
   	\STATE Find any index $j$ in $\{1 , \cdots , n\}$ and compute a weight $t > 0$ such that
   				\begin{eqnarray}
   				\delta_U^{-1} \| \x_j \|_2^2 \; \leq \;  t^{-1} \;  \leq \;  \frac{\v_j^T \Big(\A_\tau - (L_\tau + 1) \I_k  \Big)^{-2} \v_j  }{ \phi (L_\tau + 1 , \A_\tau) - \phi(L_\tau, \A_\tau) }
   														- \v_j^T \Big(\A_\tau - (L_\tau + 1) \I_k \Big)^{-1} \v_j	\textrm{;}	 \nonumber
   				\end{eqnarray}
   				where
   				\begin{equation}
   				\phi (L, \A)   =   \sum_{i=1}^k \Big(\lambda_i (\A) - L \Big)^{-1} \textrm{, }	\qquad	\qquad
   				\quad L_\tau = \tau - \sqrt{rk} \textrm{;} \nonumber
   				\end{equation} \label{alg:dual_set_sparsification:4}
   	\STATE Update the $j$-th component of $\s_\tau$ and $\A_\tau$: $\quad \s_{\tau + 1} [j] = \s_{\tau} [j] + t$, $\quad \A_{\tau + 1} = \A_\tau + t \v_j \v_j^T$;
   \ENDFOR
   \RETURN $\s = \frac{1 - \sqrt{k/r}}{r} \s_r$.
\end{algorithmic}
\end{small}
\end{algorithm}

The near-optimal column selection algorithm described in Lemma~\ref{prop:fast_column_selection} has three steps:
randomized SVD via random projection which costs $\OM\big( m k^2 \epsilon^{-4/3} \big) + \TimeMulti \big(m n k \epsilon^{-2/3} \big)$ time,
the dual set sparsification algorithm which costs $\OM\big( n k^3 \epsilon^{-2/3} \big) + \TimeMulti \big(m n \big)$ time,
and the adaptive sampling algorithm which costs $\OM\big( m k^2 \epsilon^{-4/3} \big) + \TimeMulti\big(m n k \epsilon^{-2/3}\big)$ time.
Therefore, the near-optimal column selection algorithm
costs totally $\OM\big(m k^2 \epsilon^{-4/3} + n k^3 \epsilon^{-2/3}\big) + \TimeMulti \big(m n k \epsilon^{-2/3} \big)$ time.

\section{Proofs of the Adaptive Sampling Bounds} \label{sec:proofs}

We present the proofs of Theorem~\ref{thm:adaptive_bound}, Corollary~\ref{cor:adaptive_improved}, Theorem~\ref{cor:fast_cur}, and Theorem~\ref{thm:nystrom_bound}
in Appendices~\ref{sec:proof:thm:adaptive_bound}, \ref{sec:proof:cor:adaptive_improved}, \ref{sec:proof:thm:fast_cur}, and \ref{sec:proof:thm:nystrom_bound}, respectively.

\subsection{The Proof of Theorem~\ref{thm:adaptive_bound}} \label{sec:proof:thm:adaptive_bound}

Theorem~\ref{thm:adaptive_bound} can be equivalently expressed in Theorem~\ref{thm:adaptive_bound_2}.
In order to stick to the column space convention throughout this paper,
we prove Theorem~\ref{thm:adaptive_bound_2} instead of Theorem~\ref{thm:adaptive_bound}.

\begin{theorem}[The Adaptive Sampling Algorithm] \label{thm:adaptive_bound_2}
Given a matrix $\A \in \RBmn$ and
a matrix $\R \in \RB^{r\times n}$ such that $\rk(\R) = \rk(\A \R^\dag \R) = \rho$ $(\rho \leq r \leq m)$,
let $\C_1 \in \RB^{m\times c_1}$ consist of $c_1$ columns of $\A$,
and define the residual $\B = \A - \C_1 \C_1^\dag \A$.
For $i = 1,\cdots, n$, let
\[
p_i \;=\; \|\bb_i\|_2^2 / \|\B\|_F^2,
\]
where $\bb_i$ is the $i$-th column of the matrix $\B$.
Sample further $c_2$ columns from $\A$ in $c_2$ i.i.d.\ trials,
where in each trial the $i$-th column is chosen with probability $p_i$.
Let $\C_2 \in \RB^{m\times c_2}$ contain the $c_2$ sampled columns
and  $\C = [\C_1,\C_2] \in \RB^{m\times (c_1 + c_2)}$ contain the columns of both $\C_1$ and $\C_2$,
all of which are columns of $\A$.
Then the following inequality holds:
\[
\EB \|\A - \C \Cmp \A \R^\dag \R \|_F^2
\; \leq \; \|\A - \A \R^\dag \R \|_F^2 + \frac{\rho}{c_2} \|\A - \C_1 \C_1^\dag \A\|_F^2.
\]
where the expectation is taken w.r.t. $\C_2$.
\end{theorem}

\begin{proof}
With a little abuse of symbols, we use bold uppercase letters to denote random matrices
and bold lowercase to denote random vectors,
without distinguishing between random matrices/vectors and non-random matrices/vectors.

We denote the $j$-th column of
$\V_{\A\R^\dag \R, \rho} \in \RB^{n \times \rho}$ as $\v_{j}$,
and the $(i,j)$-th entry of $\V_{\A\R^\dag \R, \rho}$ as $v_{i j}$.
Define  random vectors $\x_{j,(l)} \in \RB^{m}$  such that for $j = 1,\cdots,n$ and $l = 1,\cdots,c_2$,
\begin{equation}
\x_{j,(l)} = \frac{v_{i j}}{p_i} \bb_i = \frac{v_{i j}}{p_i} \Big( \a_i - \C_1 \C_1^\dag \a_i \Big)
\quad \textrm{with probability } p_i \textrm{,} \quad \textrm{for } i = 1, \cdots, n \textrm{,} \nonumber
\end{equation}
Notice that $\x_{j,(l)}$ is a linear function of a column of $\A$ sampled from the above defined distribution.
We have that
\begin{eqnarray}
\EB [\x_{j,(l)}] 				& = & \sum_{i=1}^n p_i \frac{v_{i j}}{p_i} \bb_i \quad = \quad \B \v_{j} \textrm{,} \nonumber\\
\EB \| \x_{j, (l)} \|_2^2 	& = & \sum_{i=1}^n p_i \frac{v^2_{i j}}{p_i^2} \|\bb_i\|_2^2
\quad =\quad   \sum_{i=1}^n \frac{v^2_{i j}}{\|\bb_i\|_2^2 / \|\B\|_F^2 } \|\bb_i\|_2^2 \quad =\quad  \|\B\|_F^2 \textrm{.}\nonumber
\end{eqnarray}
Then we let $\x_j = \frac{1}{c_2} \sum_{l=1}^{c_2} \x_{j,(l)}$, we have
\begin{eqnarray}
\EB [\x_{j}]  & = & \EB [\x_{j,(l)}] \quad = \quad \B \v_{j} \textrm{,} \nonumber \\
\EB \| \x_{j} - \B \v_{j} \|_2^2 & = & \EB \Big\| \x_{j} - \EB[\x_{j}] \Big\|_2^2
= \frac{1}{c_2} \EB \Big\| \x_{j, (l)}  - \EB [\x_{j, (l)} ]  \Big\|_2^2
= \frac{1}{c_2} \EB \| \x_{j, (l)}  - \B \v_{j}  \|_2^2  \textrm{.}\nonumber
\end{eqnarray}
According to the construction of $\x_1, \cdots, \x_\rho$,
we define the $c_2$ columns of $\A$ to be $\C_2 \in \RB^{m\times c_2}$.
Note that all the random vectors $ \x_1 \cdots, \x_\rho$ lie in the subspace $\spann(\C_1)+ \spann(\C_2)$.
We define random vectors
\begin{equation}
\w_j = \C_1 \C_1^\dag \A \R^\dag \R \v_j + \x_j = \C_1 \C_1^\dag \A \v_j + \x_j \textrm{,}
\qquad \textrm{for } j = 1,\cdots,\rho \textrm{,} \nonumber
\end{equation}
where the second equality follows from Lemma~\ref{lem:right_singular_vector}; that is, $\A \R^\dag \R \v_j = \A \v_j$
if $\v_j$ is one of the top $\rho$ right singular vectors of $\A \R^\dag \R$.
Then we have that any set of random vectors $\{\w_1,\cdots,\w_\rho \}$ lies in $\spann(\C) = \spann(\C_1) + \spann(\C_2)$.
Let $\W = [\w_1, \cdots , \w_\rho] $ be a  random matrix,
we have that $\spann(\W) \subset \spann(\C)$.
The expectation of $\w_j$ is
\begin{equation}
\EB [\w_j] = \C_1 \C_1^\dag \A \v_j + \EB [\x_j] = \C_1 \C_1^\dag \A \v_j + \B \v_j = \A \v_j \textrm{,} \nonumber
\end{equation}
therefore we have that
\begin{equation}
\w_j - \A \v_j = \x_j - \B \v_j \textrm{.}\nonumber
\end{equation}
The expectation of $\|\w_j - \A \v_j\|_2^2$ is
\begin{eqnarray}
\EB \|\w_j - \A \v_j \|_2^2 	& = & \EB \| \x_j - \B \v_j \|_2^2
												\quad = \quad \frac{1}{c_2} \EB \| \x_{j,(l)} - \B \v_j \|_2^2 \nonumber \\
												& = & \frac{1}{c_2} \EB \| \x_{j,(l)} \|_2^2 - \frac{2}{c_2} (\B \v_j )^T \EB [\x_{j,(l)}] + \frac{1}{c_2} \|\B \v_j \|_2^2 \nonumber \\
												& = & \frac{1}{c_2} \EB \| \x_{j,(l)} \|_2^2 - \frac{1}{c_2}  \| \B \v_j \|_2^2 \nonumber
												\quad = \quad \frac{1}{c_2} \|\B\|_F^2 - \frac{1}{c_2}  \| \B \v_j \|_2^2 \nonumber \\
												&\leq & \frac{1}{c_2} \|\B\|_F^2 \textrm{.} \label{eq:adaptive_bound:1}
\end{eqnarray}

To complete the proof,
we denote
\[
\F = (\sum_{q=1}^\rho \sigma_q^{-1} \w_q \u_q^T ) \A \R^\dag \R \textrm{,}
\]
where $\sigma_q$ is the $q$-th largest singular value of $\A \R^\dag \R$
and $\u_q$ is the corresponding left singular vector of $\A \R^\dag \R$.
The column space of $\F$ is contained in $\spann (\W)$ ($\subset \spann(\C)$),
and thus
\[
\|\A\R^\dag \R - \C \C^\dag \A \R^\dag \R\|_F^2 \leq \|\A\R^\dag \R - \W \W^\dag \A \R^\dag \R\|_F^2 \leq \| \A \R^\dag \R - \F \|_F^2 \textrm{.}
\]
We use $\F$ to bound the error $\|\A \R^\dag \R - \C \C^\dag \A \R^\dag \R \|_F^2$. That is,
\begin{eqnarray}
\EB \| \A - \C \C^\dag \A \R^\dag \R \|_F^2 	& = & \EB \| \A - \A \R^\dag \R + \A \R^\dag \R - \C \C^\dag \A \R^\dag \R \|_F^2 \nonumber \\
		& = & \EB \Big[\| \A - \A \R^\dag \R \|_F^2 + \| \A \R^\dag \R - \C \C^\dag \A \R^\dag \R \|_F^2 \Big] \label{eq:adaptive_bound:2} \\
		&\leq & \| \A - \A \R^\dag \R \|_F^2 + \EB \| \A \R^\dag \R - \F \|_F^2 \textrm{,}	 \nonumber
\end{eqnarray}
where (\ref{eq:adaptive_bound:2}) is due to that $\A(\I - \R^\dag \R)$
is orthogonal to $(\I - \C \C^\dag) \A \R^\dag \R$.
Since $\A \R^\dag \R$ and $\F$ both lie on the space spanned by the right singular vectors of ${\A\R^\dag \R}$
(i.e., $\{\v_j \}_{j=1}^\rho$),
we decompose $\A \R^\dag \R - \F$ along $\{\v_j \}_{j=1}^\rho$, obtaining that
\begin{eqnarray}
\EB \| \A - \C \C^\dag \A \R^\dag \R \|_F^2	&\leq & \| \A - \A \R^\dag \R \|_F^2 + \EB \| \A \R^\dag \R - \F \|_F^2 \textrm{,}	\nonumber \\
																& = & \| \A - \A \R^\dag \R \|_F^2 +  \sum_{j = 1}^\rho \EB \Big\| (\A \R^\dag \R - \F) \v_{j} \Big\|_2^2  \nonumber \\
																& = & \| \A - \A \R^\dag \R \|_F^2 +  \sum_{j = 1}^\rho \EB \Big\| \A \R^\dag \R \v_{j} - (\sum_{q=1}^\rho \sigma_q^{-1} \w_q \u_q^T ) \sigma_j \u_j  \Big\|_2^2 \nonumber \\
																& = & \| \A - \A \R^\dag \R \|_F^2 +  \sum_{j = 1}^\rho \EB \Big\| \A \R^\dag \R \v_{j} - \w_j  \Big\|_2^2 \nonumber \\
																& = & \| \A - \A \R^\dag \R \|_F^2 +  \sum_{j = 1}^\rho \EB \| \A \v_j - \w_j  \|_2^2 \label{eq:adaptive_bound:4}  \\
																&\leq & \| \A - \A \R^\dag \R \|_F^2 +  \frac{\rho}{c_2} \|\B\|_F^2 \textrm{,} \label{eq:adaptive_bound:5}
\end{eqnarray}
where (\ref{eq:adaptive_bound:4}) follows from Lemma~\ref{lem:right_singular_vector}
and (\ref{eq:adaptive_bound:5}) follows from (\ref{eq:adaptive_bound:1}).
\end{proof}

\begin{lemma} \label{lem:right_singular_vector}
We are given a matrix $\A \in \RBmn$ and a matrix $\R \in \RB^{r\times n}$ such that
$\rk(\A \R^\dag \R) = \rk(\R) = \rho$ $(\rho \leq r \leq m)$.
Letting $\v_j \in \RB^n$ be the $j$-th top right singular vector of $\A \R^\dag \R$,
we have that
\begin{equation}
\A \R^\dag \R \v_j
\;=\; \A \v_j \textrm{,} \qquad \textrm{for } j = 1, \cdots, \rho  \textrm{.} \nonumber
\end{equation}
\end{lemma}

\begin{proof}
First let $\V_{\R,\rho} \in \RB^{n \times \rho}$ contain the top $\rho$ right singular vectors of $\R$.
Then the projection of $\A$ onto the row space of $\R$ is $\A \R^\dag \R = \A \V_{\R,\rho} \V_{\R,\rho}^T$.
Let the thin SVD of $\A \V_{\R,\rho} \in \RB^{m\times \rho}$ be $\tilde{\U} \tilde{\Si} \tilde{\V}^T$,
where $\tilde{\V} \in \RB^{\rho \times \rho}$.
Then the compact SVD of $\A \R^\dag \R$ is
\[
\A \R^\dag \R
\;=\; \A \V_{\R,\rho} \V_{\R,\rho}^T
\;=\; \tilde{\U} \tilde{\Si} \tilde{\V}^T \V_{\R,\rho}^T \textrm{.}
\]
According to the definition,
$\v_j$ is the $j$-th column of $(\V_{\R,\rho} \tilde{\V}) \in \RB^{n\times \rho}$.
Thus $\v_j$ lies on the column space of $\V_{\R,\rho}$,
and $\v_j$ is orthogonal to $\V_{\R,\rho\perp}$.
Finally, since $\A - \A \R^\dag \R = \A \V_{\R,\rho\perp} \V_{\R,\rho\perp}^T$,
we have that $\v_j$ is orthogonal to $\A - \A \R^\dag \R$,
that is, $(\A - \A \R^\dag \R) \v_j = \0$,
which directly proves the lemma.
\end{proof}

\subsection{The Proof of Corollary~\ref{cor:adaptive_improved}} \label{sec:proof:cor:adaptive_improved}

Since $\C$ is constructed by columns of $\A$ and
the column space of $\C$ is contained in the column space of $\A$,
we have $\rk(\C \C^\dag \A) = \rk(\C) = \rho \leq c$. Consequently,
the assumptions of Theorem~\ref{thm:adaptive_bound} are satisfied.
The assumptions in turn imply
\begin{eqnarray}
\| \A - \C \C^\dag \A \|_F & \leq & (1+\epsilon) \| \A - \A_k \|_F \textrm{,} \nonumber \\
\| \A - \A \R_1^\dag \R_1 \|_F & \leq & (1+\epsilon) \| \A - \A_k \|_F \textrm{,} \nonumber
\end{eqnarray}
and $c / r_2 = \epsilon$.
It then follows from Theorem~\ref{thm:adaptive_bound} that
\begin{eqnarray}
\EB_\R \big\|\A - \C \Cmp \A \R^\dag \R \big\|_F^2
& = & \EB_{\R_1} \Big[ \EB_{\R_2} \Big[ \|\A - \C \Cmp \A \R^\dag \R \|_F^2 \Big| \R_1 \Big]  \Big]\nonumber \\
& \leq & \EB_{\R_1} \Big[ \|\A - \C \Cmp \A \|_F^2 + \frac{\rho}{r_2} \|\A - \A \R_1^\dag \R_1 \|_F^2 \Big] \nonumber \\
& \leq & \|\A - \C \Cmp \A \|_F^2 + \frac{c}{r_2} (1+\epsilon) \|\A - \A_k\|_F^2 \nonumber \\
& = & \|\A - \C \Cmp \A \|_F^2 + \epsilon (1+\epsilon) \|\A - \A_k\|_F^2 \textrm{.} \nonumber
\end{eqnarray}
Furthermore, we have that
\begin{eqnarray}
\Big[ \EB \|\A - \C \U \R\|_F \Big]^2
& \leq & \EB \|\A - \C \U \R\|_F^2
\quad = \quad \EB_{\C,\R} \| \A - \C \Cmp \A \R^\dag \R \|_F^2 \nonumber \\
& = & \EB_{\C} \Big[ \EB_{\R} \Big[ \| \A - \C \Cmp \A \R^\dag \R \|_F^2  \Big| \C \Big] \Big]  \nonumber \\
&\leq & \EB_{\C} \Big[ \|\A - \C \Cmp \A\|_F^2 + \epsilon (1+\epsilon) \|\A - \A_k\|_F^2 \Big]  \nonumber \\
&\leq & (1+\epsilon)^2 \| \A - \A_k \|_k^2 \textrm{,} \nonumber
\end{eqnarray}
which yields the error bound for CUR matrix decomposition.

When the matrix $\A$ is symmetric, the matrix $\C_1^T$ consists of the rows $\A$,
and thus we can use Theorem~\ref{thm:adaptive_bound_2} (which is identical to Theorem~\ref{thm:adaptive_bound}) to prove the error bound for the \nystrom approximation.
By replacing  $\R$ in Theorem~\ref{thm:adaptive_bound_2} by $\C_1^T$,
we have that
\begin{eqnarray}
\EB \big\|\A - \C \C^\dag \A {(\C_1^\dag)}^T {\C_1^T} \big\|_F^2
& \leq & \big\| \A - \A {(\C_1^\dag)}^T {\C_1^T} \big\|_F^2 + \frac{c_1}{c_2} \big\|\A - \C_1 \C_1^\dag \A \big\|_F^2 \nonumber \\
& = &  \Big( 1 + \frac{c_1}{c_2} \Big) \big\|\A - \C_1 \C_1^\dag \A \big\|_F^2 \textrm{,} \nonumber
\end{eqnarray}
where the expectation is taken w.r.t.\ $\C_2$.
Together with the inequality
\begin{equation}
\big\|\A - \C \C^\dag \A {(\C^\dag)}^T {\C^T} \big\|_F^2
\; \leq \; \big\|\A - \C \C^\dag \A {(\C_1^\dag)}^T {\C_1^T} \big\|_F^2 \nonumber
\end{equation}
given by Lemma~\ref{lem:cor:nystrom_bound},
we have that
\begin{eqnarray}
\EB_{\C_1, \C_2} \big\|\A - \C \C^\dag \A {(\C^\dag)}^T {\C^T} \big\|_F^2
& \leq & \EB_{\C_1, \C_2} \big\|\A - \C \C^\dag \A {(\C_1^\dag)}^T {\C_1^T} \big\|_F^2 \nonumber \\
& = &  \Big( 1 + \frac{c_1}{c_2} \Big) \EB_{\C_1} \big\|\A - \C_1 \C_1^\dag \A \big\|_F^2 \nonumber \\
& = & ( 1 {+} \epsilon )^2 \big\|\A {-} \A_k \big\|_F^2 \textrm{.}\nonumber
\end{eqnarray}
Hence $\EB \big\|\A {-} \C \C^\dag \A {(\C^\dag)}^T {\C^T} \big\|_F
\leq \Big[\EB \big\|\A {-} \C \C^\dag \A {(\C^\dag)}^T {\C^T} \big\|_F^2 \Big]^{-\frac{1}{2}}
\leq ( 1 {+} \epsilon ) \big\|\A {-} \A_k \big\|_F$.

\begin{lemma} \label{lem:cor:nystrom_bound}
Given an $m {\times} m$ matrix $\A$ and an $m {\times} c$ matrix $\C = [\C_1 , \C_2]$,
the following inequality holds:
\begin{equation}
\big\|\A - \C \C^\dag \A {(\C^\dag)}^T {\C^T} \big\|_F^2
\; \leq \; \big\|\A - \C \C^\dag \A {(\C_1^\dag)}^T {\C_1^T} \big\|_F^2 \textrm{.} \nonumber
\end{equation}
\end{lemma}

\begin{proof}
Let $\PM_\C \A = \C \C^\dag \A$ denote the projection of $\A$ onto the column space of $\C$,
and $\bar{\PM}_\C = \I_m - \C \C^\dag$ denote the projector onto the space orthogonal to the column space of $\C$.
It has been shown by \cite{halko2011ramdom} that, for any matrix $\A$, if $\spann(\M)\subset\spann(\N)$,
then the following inequalities hold:
\[
\|\PM_\M \A \|_\xi \leq \|\PM_\N \A\|_\xi \quad \textrm{ and }\quad
\|\bar{\PM}_\M \A \|_\xi \geq \|\bar{\PM}_\N \A\|_\xi \textrm{.}
\]
Accordingly, $\A \PM^T_{\R^T} = \A \R^\dag \R$ is the projection of $\A$ onto the row space of $\R\in\RB^{r\times n}$.
We further have that
\begin{eqnarray}
\|\A - \PM_\C \A \PM^T_\C\|_F^2
& = & \|\A - \PM_\C \A + \PM_\C \A - \PM_\C \A \PM^T_\C\|_F^2 \nonumber \\
& = & \|\bar{\PM}_\C \A + \PM_\C \A \bar{\PM}^T_\C\|_F^2
\;=\;  \|\bar{\PM}_\C \A\|_F^2 + \|\PM_\C \A \bar{\PM}^T_\C\|_F^2  \nonumber
\end{eqnarray}
and
\begin{eqnarray}
\|\A - \PM_\C \A \PM^T_{\C_1}\|_F^2
& = & \|\A - \PM_\C \A + \PM_\C \A - \PM_\C \A \PM^T_{\C_1} \|_F^2 \nonumber \\
& = & \|\bar{\PM}_\C \A + \PM_\C \A \bar{\PM}^T_{\C_1} \|_F^2
\;=\;  \|\bar{\PM}_\C \A\|_F^2 + \|\PM_\C \A \bar{\PM}^T_{\C_1} \|_F^2 \textrm{,} \nonumber
\end{eqnarray}
where the last equalities follow from $\PM_\C \perp \bar{\PM}_\C$.
Since $\spann(\C_1) \subset \spann(\C)$, we have $\|\PM_\C \A \bar{\PM}^T_{\C_1} \|_F^2 \geq \|\PM_\C \A \bar{\PM}^T_{\C} \|_F^2$,
which proves the lemma.
\end{proof}

\subsection{The Proof of Theorem~\ref{cor:fast_cur}} \label{sec:proof:thm:fast_cur}

The error bound follows directly from Lemma~\ref{prop:fast_column_selection} and Corollary~\ref{cor:adaptive_improved}.
The near-optimal column selection algorithm costs
$\OM\big(m k^2 \epsilon^{-4/3} + n k^3 \epsilon^{-2/3}\big) + \TimeMulti \big(m n k \epsilon^{-2/3} \big)$
time to construct $\C$ and
$\OM\big(n k^2 \epsilon^{-4/3} + m k^3 \epsilon^{-2/3}\big) + \TimeMulti \big(m n k \epsilon^{-2/3} \big)$
time to construct $\R_1$.
Then the adaptive sampling algorithm costs
$\OM\big( n k^2 \epsilon^{-2} \big) + \TimeMulti\big( m n k \epsilon^{-1} \big)$
time to construct $\R_2$.
Computing the Moore-Penrose inverses of $\C$ and $\R$ costs
$\OM(mc^2) + \OM(nr^2) = \OM\big( m k^2 \epsilon^{-2} + n k^2 \epsilon^{-4} \big)$ time.
The multiplication of $\C^\dag \A \R^\dag$ costs $\TimeMulti(m n c) = \TimeMulti(m n k \epsilon^{-1})$ time.
So the total time complexity is
$\OM \big( (m + n)k^3 \epsilon^{-2/3} + m k^2 \epsilon^{-2} + nk^2\epsilon^{-4} \big) + \TimeMulti\big( m n k \epsilon^{-1} \big)$.

\subsection{The Proof of Theorem~\ref{thm:nystrom_bound}} \label{sec:proof:thm:nystrom_bound}

The error bound follows immediately from Lemma~\ref{prop:fast_column_selection} and Corollary~\ref{cor:adaptive_improved}.
The near-optimal column selection algorithm costs
$\OM\big(m k^2 \epsilon^{-4/3} + m k^3 \epsilon^{-2/3}\big) + \TimeMulti \big(m^2 k \epsilon^{-2/3} \big)$
time to select $c_1 = \OM(k \epsilon^{-1})$ columns of $\A$ construct $\C_1$.
Then the adaptive sampling algorithm costs
$\OM\big( m k^2 \epsilon^{-2} \big) + \TimeMulti\big( m^2 k \epsilon^{-1} \big)$
time to select $c_2 = \OM(k \epsilon^{-2})$ columns construct $\C_2$.
Finally it costs $\OM(m c^2) + \TimeMulti ( m^2 c ) = \OM(m k^2 \epsilon^{-4}) + \TimeMulti\big( m^2 k \epsilon^{-2} \big)$
time to construct the intersection matrix $\U = \C^\dag \A (\C^\dag)^T$.
So the total time complexity is $\OM\big( m k^2 \epsilon^{-4} + m k^3 \epsilon^{-2/3} \big) + \TimeMulti\big( m^2 k \epsilon^{-2} \big)$.

\section{Proofs of the Lower Error Bounds} \label{sec:proof_lower_bounds}

In Appendix~\ref{sec:construction_bad_case} we construct two adversarial cases which will be used throughout this appendix.
In  Appendix~\ref{sec:lower_bound_standard} we prove the lower bounds of the standard \nystrom method.
In  Appendix~\ref{sec:lower_bound_ensemble} we prove the lower bounds of the ensemble \nystrom method.
Theorems~\ref{thm:nystrom_error}, \ref{thm:lower_bound_conventional}, \ref{thm:lower_bound}, \ref{thm:lower_bound_ensemble_additive}, and \ref{thm:lower_bound_ensemble}
are used for proving Theorem~\ref{thm:lower_bounds_nystrom}.

\subsection{Construction of the Adversarial Cases} \label{sec:construction_bad_case}

We now consider the construction of adversarial cases for the spectral norm bounds and the Frobenius norm and nuclear norm bounds, respectively.

\subsubsection{The Adversarial Case for the Spectral Norm Bound}
We construct an $m {\times} m$ positive definite matrix $\B$ as follows:
\begin{equation} \label{eq:construction_bad_nystrom}
\B = (1-\alpha) \I_m + \alpha \1_{m} \1_m^T
= \left[
    \begin{array}{cccc}
      1 & \alpha & \cdots & \alpha \\
      \alpha & 1 &  \cdots & \alpha \\
      \vdots & \vdots  & \ddots & \vdots \\
      \alpha & \alpha & \cdots & 1 \\
    \end{array}
  \right]
= \left[
    \begin{array}{cc}
      \W & \B_{2 1}^T \\
      \B_{2 1} & \B_{2 2} \\
    \end{array}
  \right]
  \textrm{,}
\end{equation}
where $\alpha \in [0, 1)$.
It is easy to verify $\x^T \B \x > 0$ for any nonzero $\x \in \RB^{m}$.
We show some properties of $\B$ in Lemma~\ref{lem:nystrom_residual}.

\begin{lemma} \label{lem:nystrom_residual}
Let $\B_k$ be the best rank-$k$ approximation to the matrix $\B$ defined in (\ref{eq:construction_bad_nystrom}).
Then we have that
\begin{eqnarray}
&
\|\B\|_F \;=\; \sqrt{m^2 \alpha^2 + m(1-\alpha^2)}
\textrm{, }&
\quad
\|\B - \B_k \|_F \;=\; \sqrt{m-k}\, (1-\alpha)  \textrm{, }
\nonumber \\
&
\|\B\|_2 \;=\; 1+ m\alpha -\alpha
\textrm{ , }&
\quad
\|\B - \B_k \|_2 \;=\; 1-\alpha
\textrm{,}\nonumber \\
&
\|\B\|_* \;=\; m
\textrm{,}&
\quad
\|\B - \B_k \|_* \;=\; (m-k)(1-\alpha)
\textrm{, } \nonumber
\end{eqnarray}
where $1\leq k\leq m-1$.
\end{lemma}

\begin{proof}
The squared Frobenius norm of $\B$ is
\[
\|\B\|_F^2 \;=\; \sum_{i, j} b_{i j}^2 \;=\; m + (m^2-m)\alpha^2 .
\]
Then we study the singular values of $\B$. Since $\B$ is SPSD,
here we do not distinguish between its singular values and eigenvalues.

The spectral norm, that is, the largest singular value, of $\B$ is
\[
\|\B\|_2 \;=\; \sigma_{1} \;=\; \lambda_1
\;=\; \max_{\|\x\|_2 \leq 1} \x^T \B \x
\;=\; \max_{\|\x\|_2 \leq 1}  (1-\alpha) \|\x\|_2^2 + \alpha (\1_m^T \x)^2
\;=\; 1 -\alpha + m\alpha \textrm{,}
\]
where the maximum is attained when $\x =  \frac{1}{\sqrt{m}} \1_m$.
Thus $\u_1 = \frac{1}{\sqrt{m}} \1_m$ is the top singular vector of $\B$.
Then the projection of $\B$ onto the subspace orthogonal to $\u_1$ is
\[
\B_{1 \perp} \triangleq \B - \B_1
\;=\; \B - \sigma_1 \u_1 \u_1^T
\;=\; \frac{1-\alpha}{m} (m \I_m - \1_m \1_m^T) .
\]
Then for all $j > 1$, the $j$-th top eigenvalue $\sigma_j$ and eigenvector $\u_j$,
that is, the singular value and singular vector, of $\B$ satisfy
\[
\sigma_j \u_j
\;=\; \B \u_j
\;=\; \B_{1 \perp} \u_j
\;=\; \frac{1-\alpha}{m} \big( m \u_j - (\1_m^T \u_j) \1_m \big)
\;=\; \frac{1-\alpha}{m} ( m \u_j - \0 ) ,
\]
where the last equality follows from $\u_j \perp \u_1$, that is, $\1_m^T \u_j =0$.
Thus $\sigma_j = 1-\alpha$, and
\[
\|\B - \B_k\|_2 \;=\; \sigma_{k+1}
\;=\; 1-\alpha
\]
for all $1 \leq k < m$.
Finally we have that
\begin{eqnarray}
\|\B - \B_k \|_F^2 &=& \|\B\|_F^2 - \sum_{i=1}^k \sigma_{i}^2 \;=\; (m-k)(1-\alpha)^2 \textrm{,} \nonumber \\
\|\B - \B_k\|_* &=& (m-k)\sigma_{2} \;=\; (m-k)(1-\alpha) \textrm{,} \nonumber \\
\|\B \|_* &=& \sum_{i=1}^m \sigma_{i} \;=\; (1+m\alpha -\alpha) + (m-1)(1-\alpha) \;=\; m \textrm{,} \nonumber
\end{eqnarray}
which complete our proofs.
\end{proof}

\subsubsection{The Adversarial Case for The Frobenius Norm and Nuclear Norm Bounds}
Then we construct another adversarial case for proving the Frobenius norm and nuclear norm bounds.
Let $\B$ be a $p\times p$ matrix with diagonal entries equal to one and off-diagonal entries equal to $\alpha$.
Let $m=k p$ and we construct an $m\times m$ block diagonal matrix $\A$ as follows:
\begin{eqnarray}\label{eq:construction_bad_nystrom_blk}
\A
\; = \; \blkdiag(\underbrace{\B, \cdots, \B}_{k \; \textrm{blocks}})
\; = \; \left[
          \begin{array}{cccc}
            \B & \0 & \cdots & \0 \\
            \0 & \B & \cdots & \0 \\
            \vdots & \vdots & \ddots & \vdots \\
            \0 & \0 & \cdots & \B \\
          \end{array}
        \right]
  \textrm{.}
\end{eqnarray}

\begin{lemma} \label{lem:nystrom_residual_new}
Let $\A_k$ be the best rank-$k$ approximation to the matrix $\A$ defined in (\ref{eq:construction_bad_nystrom_blk}).
Then we have that
\begin{eqnarray}
\sigma_1 (\A) & = & \cdots \; = \; \sigma_k (\A) = 1 + p\alpha - \alpha  \textrm{,}\nonumber \\
\sigma_{k+1} (\A) & = & \cdots \; = \; \sigma_m (\A) = 1  - \alpha \textrm{,} \nonumber \\
\big\| \A - \A_k \big\|_F & = & (1-\alpha) \sqrt{m-k}  \textrm{,} \nonumber \\
\big\| \A - \A_k \big\|_* & = & (1-\alpha) (m-k)  \textrm{.} \nonumber
\end{eqnarray}
\end{lemma}

Lemma~\ref{lem:nystrom_residual_new} can be easily proved using Lemma~\ref{lem:nystrom_residual}.

\subsection{Lower Bounds of the Standard \nystrom Method} \label{sec:lower_bound_standard}

\begin{theorem} \label{thm:nystrom_error}
For an $m\times m$ matrix $\B$ with diagonal entries equal to one and off-diagonal entries equal to $\alpha \in [0,1)$,
the approximation error incurred by the standard \nystrom method is lower bounded by
\begin{eqnarray}
\big\| \B - \tilde{\B}_c^{\textrm{nys}} \big\|_F
&\geq& (1-\alpha) \sqrt{(m-c) \Big(1 + \frac{ m+c + \frac{2}{\alpha} - 2}{( c + \frac{1-\alpha}{\alpha} )^2} \Big)} \textrm{,} \nonumber \\
\big\| \B - \tilde{\B}_c^{\textrm{nys}} \big\|_2
&\geq& \frac{(1-\alpha) \Big(m + \frac{1-\alpha}{\alpha}\Big)}{c + \frac{1-\alpha}{\alpha}} \textrm{,} \nonumber \\
\big\|\B - \tilde{\B}_c^{\textrm{nys}}\big\|_*
& \geq & (m-c) (1-\alpha) \frac{1+c\alpha}{1 + c\alpha - \alpha} \textrm{.} \nonumber
\end{eqnarray}
Furthermore, the matrix $(\B - \tilde{\B}_c^{\textrm{nys}})$ is SPSD.
\end{theorem}

\begin{proof}
The matrix $\B$ is partitioned as in (\ref{eq:construction_bad_nystrom}).
The residual of the \nystrom approximation is
\begin{equation} \label{eq:nystrom_error1}
 \| \B - \tilde{\B}_c^{\textrm{nys}} \|_\xi
\;=\; \| \B_{2 2} - \B_{2 1} \W^\dag \B_{2 1}^T \|_\xi \textrm{,}
\end{equation}
where $\xi = 2$, $F$, or $*$.
Since $\W = (1-\alpha) \I_{c} + \alpha \1_{c} \1_{c}^T$ is nonsingular when $\alpha \in [0,1)$,
so $\W^\dag = \W^{-1}$. We apply the Sherman-Morrison-Woodbury formula
\[
(\A + \B \C \D)^{-1} \; = \; \A^{-1} - \A^{-1} \B (\C^{-1} + \D \A^{-1} \B)^{-1} \D \A^{-1}
\]
to compute $\W^{\dag}$, yielding
\[
\W^\dag \; = \; \frac{1}{1-\alpha} \I_c - \frac{\alpha}{(1-\alpha)(1-\alpha+c\alpha)} \1_{c} \1_{c}^T \textrm{.}
\]
According to the construction,
$\B_{2 1}$ is an $(m{-}c)\times c$ matrix with all entries equal to $\alpha$,
it follows that $\B_{2 1} \W^\dag \B_{2 1}^T$ is an $(m{-}c) {\times} (m{-}c)$ matrix
with all entries equal to
\begin{equation} \label{eq:all_xi_matrix}
\eta
\;\triangleq \; \alpha^2 \1_{c}^T \W^\dag \1_{c}
\; = \; \frac{c\alpha^2}{1-\alpha + c\alpha} \textrm{.}
\end{equation}
Then we obtain that
\begin{eqnarray} \label{eq:nystrom_error1_new}
\B_{2 2} - \B_{2 1} \W^\dag \B_{2 1}^T
& = &
(1{-}\alpha)\I_{m{-}c} + (\alpha -\eta) \1_{m{-}c} \1_{m{-}c}^T \textrm{.}
\end{eqnarray}
It is easy to check that $\eta \leq \alpha \leq 1$, thus the matrix $(1{-}\alpha)\I_{m{-}c} + (\alpha -\eta) \1_{m{-}c} \1_{m{-}c}^T$
is SPSD, and so is $(\B - \tilde{\B}_c^{\textrm{nys}})$.

Combining (\ref{eq:nystrom_error1}) and (\ref{eq:nystrom_error1_new}),
we have that
\begin{eqnarray} \label{eq:nystrom_error2}
\| \B - \tilde{\B}_c^{\textrm{nys}} \|_F^2
&=& \big\| (1{-}\alpha)\I_{m{-}c} + (\alpha -\eta) \1_{m{-}c} \1_{m{-}c}^T \big\|_F^2    \nonumber \\
&=& (m{-}c) \big( 1 {-} \eta \big)^2 + \Big((m{-}c)^2-(m{-}c)\Big) \big( \alpha {-} \eta \big)^2 \nonumber \\
&=& (m{-}c) (1 {-} \alpha)^2 \Big(1 + \frac{ \alpha^2 (m{+}c) + 2(\alpha {-} \alpha^2) }{(1 {-} \alpha {+} c\alpha)^2} \Big) \nonumber \\
&=& (m{-}c) (1{-}\alpha)^2 \Big(1 + \frac{ m{+}c + \frac{2}{\alpha} - 2 }{( c + \frac{1 {-} \alpha}{\alpha} )^2} \Big) \textrm{,} \label{eq:residual_nystrom_fro}
\end{eqnarray}
which proves the Frobenius norm of the residual.

Now we compute the spectral norm of the residual.
Based on the results above we have that
\[
\big\| \B - \tilde{\B}_c^{\textrm{nys}} \big\|_2
\;=\; \big\| (1{-}\alpha)\I_{m{-}c} + (\alpha{-}\eta) \1_{m{-}c} \1_{m{-}c}^T \big\|_2 \textrm{.}
\]
Similar to the proof of Lemma~\ref{lem:nystrom_residual},
it is easily obtained that $\frac{1}{\sqrt{m-c}} \1_{m-c}$ is the top singular vector of the SPSD matrix
$(1-\alpha)\I_{m{-}c} + (\alpha-\eta)\1_{m{-}c} \1_{m{-}c}^T$,
so the top singular value is
\begin{equation}
\sigma_1 \big( \B - \tilde{\B}_c^{\textrm{nys}} \big)
\;=\; (m-c)(\alpha-\eta)+1-\alpha
\;=\; \frac{(1-\alpha) \Big(m + \frac{1-\alpha}{\alpha}\Big)}{c + \frac{1-\alpha}{\alpha}} \textrm{,}\label{eq:residual_nystrom_spe}
\end{equation}
which proves the spectral norm bound because $\|\B - \tilde{\B}_c^{\textrm{nys}}\|_2 = \sigma_1 \big( \B - \tilde{\B}_c^{\textrm{nys}} \big)$.

It is also easy to show the rest singular values obey
\begin{eqnarray}
\sigma_2 \big( \B - \tilde{\B}_c^{\textrm{nys}} \big)
&=& \cdots
\quad=\quad \sigma_{m-c} \big( \B - \tilde{\B}_c^{\textrm{nys}} \big)
\quad\geq\quad 0  \textrm{,} \nonumber\\
\sigma_{m-c+1} \big( \B - \tilde{\B}_c^{\textrm{nys}} \big)
&=& \cdots
\quad=\quad \sigma_{m} \big( \B - \tilde{\B}_c^{\textrm{nys}} \big)
\quad=\quad 0 \textrm{.} \nonumber
\end{eqnarray}
Thus we have, for $i = 2,\cdots, m-c$,
\[
\sigma_i^2 \big( \B - \tilde{\B}_c^{\textrm{nys}} \big)
\; = \;  \frac{\| \B - \tilde{\B}_c^{\textrm{nys}} \|_F^2 - \sigma_1^2 \big( \B - \tilde{\B}_c^{\textrm{nys}} \big)}{m-c-1}
\; = \; (1 - \alpha)^2 \textrm{.}
\]
The nuclear norm of the residual $\big( \B - \tilde{\B}_c^{\textrm{nys}} \big)$ is
\begin{eqnarray}
\|\B - \tilde{\B}_c^{\textrm{nys}}\|_*
& = & \sum_{i=1}^m \sigma \big( \B - \tilde{\B}_c^{\textrm{nys}} \big) \nonumber\\
& = & \sigma_1 \big( \B - \tilde{\B}_c^{\textrm{nys}} \big) + (m-c-1)\, \sigma_2 \big( \B - \tilde{\B}_c^{\textrm{nys}} \big) \nonumber\\
& = & (m-c)(1-\eta) \nonumber\\
& = & (m-c) (1-\alpha) \Big( 1 + \frac{1}{c+\frac{1-\alpha}{\alpha}} \Big)  \textrm{.}\label{eq:residual_nystrom_nuclear}
\end{eqnarray}
The theorem follows from equalities (\ref{eq:residual_nystrom_fro}), (\ref{eq:residual_nystrom_spe}), and (\ref{eq:residual_nystrom_nuclear}).
\end{proof}

Now we use the matrix $\A$ constructed in (\ref{eq:construction_bad_nystrom_blk}) to show the Frobenius norm and nuclear norm lower bound.
The bound is stronger than the one in Theorem~\ref{thm:nystrom_error} by a factor of $k$.

\begin{theorem} \label{thm:lower_bound_conventional}
For the $m\times m$ SPSD matrix $\A$ defined in (\ref{eq:construction_bad_nystrom_blk}),
the approximation error incurred by the standard \nystrom method is lower bounded by
\begin{eqnarray}
{\big\| \A - \C \W^\dag \C^T \big\|_F}
& \geq &
(1-\alpha)\sqrt{m - c - k + \frac{k(m + \frac{1-\alpha}{\alpha}k)^2}{(c + \frac{1-\alpha}{\alpha}k )^2} }
\textrm{,} \nonumber \\
{\big\| \A - \C \W^\dag \C^T \big\|_*}
& \geq &
(1-\alpha) (m - c)\Big(1 + \frac{k}{c  + \frac{1-\alpha}{\alpha}k } \Big) \textrm{,}\nonumber
\end{eqnarray}
where $k < m$ is an arbitrary positive integer.
\end{theorem}

\begin{proof}
Let $\C$ consist of $c$ column sampled from $\A$ and $\hat{\C}_i$ consist of $c_i$ columns sampled from the $i$-th block diagonal matrix in $\A$.
Without loss of generality, we assume $\hat{\C}_i$ consists of the first $c_i$ columns of $\B$,
and accordingly $\hat{\W}_i$ consists of the top left $c_i \times c_i$ block of $\B$.
Thus $\C = \blkdiag\big(\hat{\C}_1 , \cdots , \hat{\C}_k \big)$ and $\W= \blkdiag\big(\hat{\W}_1 , \cdots , \hat{\W}_k \big)$.
\begin{eqnarray}
\tilde{\A}^{\textrm{nys}}_{c}
\;=\; \C \W^\dag \C
& = & \left[
          \begin{array}{ccc}
            \hat{\C}_1 &  & \0 \\
             & \ddots &  \\
            \0 &  & \hat{\C}_k \\
          \end{array}
        \right]
        \left[
          \begin{array}{ccc}
            \hat{\W}_1 &  & \0 \\
             & \ddots &  \\
            \0 &  & \hat{\W}_k \\
          \end{array}
        \right]^\dag
        \left[
          \begin{array}{ccc}
            \hat{\C}_1^T &  & \0 \\
             & \ddots &  \\
            \0 &  & \hat{\C}_k^T \\
          \end{array}
        \right]  \nonumber \\
& = & \left[
          \begin{array}{ccc}
            \hat{\C}_1 &  & \0 \\
             & \ddots &  \\
            \0 &  & \hat{\C}_k \\
          \end{array}
        \right]
        \left[
          \begin{array}{ccc}
            \hat{\W}_1^\dag &  & \0 \\
             & \ddots &  \\
            \0 &  & \hat{\W}_k^\dag \\
          \end{array}
        \right]
        \left[
          \begin{array}{ccc}
            \hat{\C}_1^T &  & \0 \\
             & \ddots &  \\
            \0 &  & \hat{\C}_k^T \\
          \end{array}
        \right]  \nonumber\\
& = & \left[
          \begin{array}{ccc}
           \hat{\C}_1 \hat{\W}_1^\dag \hat{\C}_1^T &  & \0 \\
             & \ddots &  \\
            \0 &  & \hat{\C}_k \hat{\W}_k^\dag \hat{\C}_k^T\\
          \end{array}
        \right] \label{eq:thm:lower_bound_conventional:1}
  \textrm{.}
\end{eqnarray}
Then it follows from Theorem~\ref{thm:nystrom_error} that
\begin{eqnarray}
\big\| \A - \tilde{\A}^{\textrm{nys}}_{c} \big\|_F^2
& = & \sum_{i=1}^k \big\| \B -  \hat{\C}_i \hat{\W}_i^\dag \hat{\C}_i^T \big\|_F^2 \nonumber \\
& = & \sum_{i=1}^k (p-c_i) (1-\alpha)^2 \Big( 1 + \frac{p+c_i + 2\frac{1-\alpha}{\alpha}}{(c_i+\frac{1-\alpha}{\alpha})^2} \Big)\nonumber \\
& = & (1-\alpha)^2 \sum_{i=1}^k (\hat{p}-\hat{c}_i) \Big( 1 + \frac{\hat{p}+\hat{c}_i}{\hat{c}_i^2} \Big)\nonumber\\
& = & (1-\alpha)^2 \Big(m - c - k + \hat{p}^2 \sum_{i=1}^k \hat{c}_i^{-2} \Big)\textrm{,} \nonumber
\end{eqnarray}
where $\hat{p} = p + \frac{1-\alpha}{\alpha}$ and $\hat{c_i} = c_i + \frac{1-\alpha}{\alpha}$.
Since $\sum_{i=1}^k \hat{c}_i = c + \frac{1-\alpha}{\alpha} k \triangleq \hat{c}$,
the term $\sum_{i=1}^k \hat{c}_i^{-2}$ is minimized when $\hat{c}_1 = \cdots = \hat{c}_k$.
Thus $\sum_{i=1}^k \hat{c}_i^{-2} = k \frac{k^2}{\hat{c}^{2}} = k^3 \hat{c}^{-2}$.
Finally we have that
\begin{eqnarray}
\big\| \A - \tilde{\A}^{\textrm{nys}}_{c} \big\|_F^2
& = & (1-\alpha)^2 \Big(m - c - k + \hat{p}^2 \sum_{i=1}^k \hat{c}_i^{-2} \Big) \nonumber \\
& \geq & (1-\alpha)^2 \Big(m - c - k + \frac{k(m + \frac{1-\alpha}{\alpha}k)^2}{(c + \frac{1-\alpha}{\alpha}k )^2} \Big)\textrm{,} \nonumber
\end{eqnarray}
by which the Frobenius norm bound follows.

Since the matrices $\B -  \hat{\C}_i \hat{\W}_i^\dag \hat{\C}_i^T$ are all SPSD by Theorem~\ref{thm:nystrom_error},
so the matrix $(\A - \tilde{\A}^{\textrm{nys}}_{c})$ is also SPSD.
We have that
\begin{eqnarray}
\big\| \A - \tilde{\A}^{\textrm{nys}}_{c} \big\|_*
& = & \sum_{i=1}^k \big\| \B -  \hat{\C}_i \hat{\W}_i^\dag \hat{\C}_i^T \big\|_* \nonumber \\
& \geq & (1-\alpha) \sum_{i=1}^k (p-c_i)\Big(1 + \frac{1}{c_i + \frac{1-\alpha}{\alpha} } \Big)\nonumber \\
& \geq & (1-\alpha) \, k \, (\frac{m}{k} - \frac{c}{k})\Big(1 + \frac{1}{c/k + \frac{1-\alpha}{\alpha} } \Big) \nonumber \\
& = & (1-\alpha) (m - c)\Big(1 + \frac{k}{c  + \frac{1-\alpha}{\alpha}k } \Big) \textrm{,} \nonumber
\end{eqnarray}
where the former inequality follows from Theorem~\ref{thm:nystrom_error},
and the latter inequality follows by minimizing w.r.t. $c_1, \cdots , c_k$ subjecting to $c_1+ \cdots + c_k = c$.
\end{proof}

\begin{theorem} \label{thm:lower_bound}
There exists an $m{\times}m$ SPSD matrix $\A$ such that
the approximation error incurred by the standard \nystrom method is lower bounded by
\begin{eqnarray}
\frac{\big\| \A - \C \W^\dag \C^T \big\|_F}{\big\| \A - \A_k \big\|_F}
& \geq & \sqrt{ 1 + \frac{ m^2 k - c^3 }{c^2 (m-k)} } \textrm{,} \nonumber \\
\frac{\| \A - \C \W^\dag \C^T\|_2}{\| \A - \A_k\|_2}
& \geq & \frac{m}{c} \textrm{,} \nonumber \\
\frac{\| \A - \C \W^\dag \C^T\|_*}{\| \A - \A_k\|_*}
& \geq & \frac{m - c}{m-k}\Big( 1 + \frac{k}{c  } \Big) \textrm{,} \nonumber
\end{eqnarray}
where $k<m$ is an arbitrary positive integer.
\end{theorem}

\begin{proof}
For the spectral norm bound
we use the matrix $\A$ constructed in (\ref{eq:construction_bad_nystrom}) and set $\alpha \rightarrow 1$,
then it follows directly from Lemma~\ref{lem:nystrom_residual} and Theorem~\ref{thm:nystrom_error}.
For the Frobenius norm and nuclear norm bounds,
we use the matrix $\A$ constructed in (\ref{eq:construction_bad_nystrom_blk}) and set $\alpha \rightarrow 1$,
then it follows directly from Lemma~\ref{lem:nystrom_residual_new} and Theorem~\ref{thm:lower_bound_conventional}.
\end{proof}

\subsection{Lower Bounds of the Ensemble \nystrom Method} \label{sec:lower_bound_ensemble}

The ensemble \nystrom method \citep{kumar2009ensemble} is previously defined in (\ref{eq:ensemble_nystrom_approx}).
To derive lower bounds of the ensemble \nystrom method, we assume that the $t$ samples are non-overlapping.
According to the construction of the matrix $\B$ in (\ref{eq:construction_bad_nystrom}),
each of the $t$ non-overlapping samples are equally ``important'',
so without loss of generality we set the $t$ samples with equal weights: $\mu^{(1)}=\cdots=\mu^{(t)}=\frac{1}{t}$.

\begin{lemma}\label{lem:lower_bound_ensemble_additive}
Assume that the ensemble \nystrom method selects a collection of $t$ samples,
each sample ${\C^{(i)}}$ ($i=1, \cdots, t$) contains $c$ columns of $\B$ without overlapping.
For an $m\times m$ matrix $\B$ with all diagonal entries equal to one and off-diagonal entries equal to $\alpha\in [0,1)$,
the approximation error incurred by the ensemble \nystrom method is lower bounded by
\begin{eqnarray}
\big\|\B - \tilde{\B}_{t, c}^{\textrm{ens}} \big\|_F
&\geq & (1-\alpha) \sqrt{ \Big(m - 2c + \frac{c}{t}\Big)
        \Big( 1+ \frac{m+\frac{c}{t}+ \frac{2}{\alpha} - 2}{(c + \frac{1-\alpha}{\alpha})^2}  \Big)} \textrm{,} \nonumber\\
\big\|\B - \tilde{\B}_{t, c}^{\textrm{ens}} \big\|_*
&\geq & (1-\alpha) (m-c) \frac{c + \frac{1}{\alpha}}{c + \frac{1-\alpha}{\alpha}}  \textrm{.} \nonumber
\end{eqnarray}
where $\tilde{\B}_{t, c}^{\textrm{ens}} = \frac{1}{t} \sum_{i=1}^t \C^{(i)} {\W^{(i)}}^\dag {\C^{(i)}}^T$.
Furthermore, the matrix $(\B - \tilde{\B}_{t, c}^{\textrm{ens}})$ is SPSD.
\end{lemma}

\begin{proof}
We use the matrix $\B$ constructed in (\ref{eq:construction_bad_nystrom}).
It is easy to check that ${\W^{(1)}} = \cdots = {\W^{(t)}}$,
so we use the notation ${\W}$ instead.
We assume that the samples contain the firs $tc$ columns of $\B$ and each sample contains neighboring columns,
that is,
$$
\B \; = \; \big[\C^{(1)},\cdots, \C^{(t)}, \; \B_{(tc+1) : m} \big].
$$

If a sample ${\C}$ contains the first $c$ columns of $\B$,
then
\begin{small}
\[
\C \W^\dag {\C}^T
\;=\; \left[
    \begin{array}{cc}
      \W & \B_{2 1}^T \\
      \B_{2 1} & \B_{2 1} \W^\dag \B_{2 1}^T \\
    \end{array}
  \right] \quad
  \textrm{ and }\quad
\B - \C \W^\dag {\C}^T
\;=\; \left[
    \begin{array}{cc}
      \0 & \0 \\
      \0 & \B_{2 2} - \B_{2 1} \W^\dag \B_{2 1}^T \\
    \end{array}
  \right]
  \textrm{;}
\]
\end{small}
otherwise, after permuting the rows and columns of $\B - \C \W^\dag {\C}^T$,
we get the same result:
\begin{small}
\[
\Pii \big( \B - \C \W^\dag {\C}^T \big)  \Pii^T
\;=\; \B - \Pii \big( \C \W^\dag {\C}^T \big) \Pii^T
\;=\; \left[
    \begin{array}{cc}
      \0 & \0 \\
      \0 & \B_{2 2} - \B_{2 1} \W^\dag \B_{2 1}^T \\
    \end{array}
  \right]
  \textrm{,}
\]
\end{small}
where $\Pii$ is a permutation matrix.
As was shown in Equation~(\ref{eq:all_xi_matrix}),
$\B_{2 1} \W^\dag \B_{2 1}^T$ is an $(m{-}c) {\times} (m{-}c)$ matrix with all entries equal to
\[
\eta\;=\; \frac{c\alpha^2}{1-\alpha + c\alpha} \textrm{.}
\]

\begin{figure*}
\subfigtopskip = 0pt
\begin{center}
\centering
\includegraphics[width=70mm]{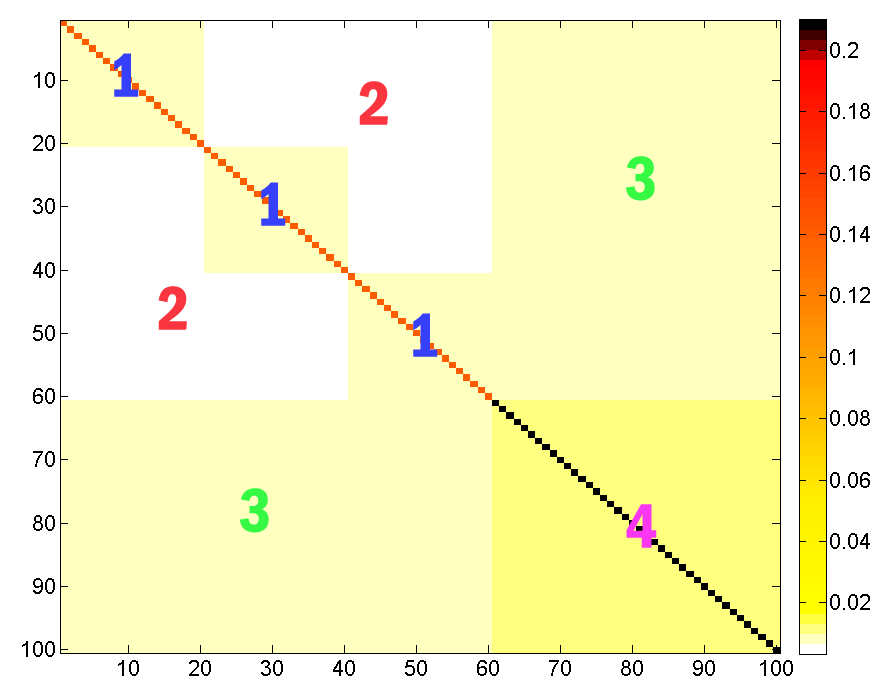}
\end{center}
   \caption{An illustration of the matrix $\B - \B_{t, c}^{\textrm{ens}}$ for the ensemble \nystrom method
            where $\B$ is defined in (\ref{eq:construction_bad_nystrom}).
            Here we set $m=100$, $c=20$, $\alpha=0.8$, and $t=3$.
            For the ensemble \nystrom method without overlapping, the matrix $\B - \B_{t, c}^{\textrm{ens}}$
            can always be partitioned into four regions as annotated.}
\label{fig:ensemble_resudial}
\end{figure*}

Based on the properties of the matrix $\B - \C^{(i)} {\W^{(i)}}^\dag {\C^{(i)}}^T$,
we study the values of the entries of $\B - \tilde{\B}_{t, c}^{\textrm{ens}}$.
We can express it as
\begin{equation}\label{eq:lem:lower_bound_ensemble_additive:expand}
\B - \tilde{\B}_{t, c}^{\textrm{ens}}
= \B - \frac{1}{t} \sum_{i=1}^t \C^{(i)} {\W^{(i)}}^\dag {\C^{(i)}}^T
= \frac{1}{t} \sum_{i=1}^t \Big( \B - \C^{(i)} {\W}^\dag {\C^{(i)}}^T\Big) ,
\end{equation}
and then a discreet examination reveals that $\B - \tilde{\B}_{t, c}^{\textrm{ens}}$
can be partitioned into four kinds of regions as illustrated in Figure~\ref{fig:ensemble_resudial}.
We annotate the regions in the figure and summarize the values of entries in each region in the table below.
(Region~1 and 4 are further partitioned into diagonal entries and off-diagonal entries.)

\begin{footnotesize}
\begin{tabular}{c| cccccc}
  \hline
  Region    &  1 (diag) &  1 (off-diag) &       2       &       3       &  4 (diag) &   4 (off-diag)    \\
  \hline
  \#Entries&   $tc$    &   $tc^2-tc$   & $(tc)^2-tc^2$ &   $2tc(m-tc)$ &  $m-tc$   & $(m-tc)^2-(m-tc)$ \\
  Value     &$\frac{t-1}{t}(1-\eta)$ &$\frac{t-1}{t}(\alpha-\eta)$&$\frac{t-2}{t}(\alpha-\eta)$&$\frac{t-1}{t}(\alpha-\eta)$& $1-\eta$ & $\alpha-\eta$ \\
  \hline
\end{tabular}
\end{footnotesize}


Now we do  summation over the entries of $\B - \tilde{\B}_{t, c}^{\textrm{ens}}$ to compute its squared Frobenius norm:
\begin{small}
\begin{eqnarray}
\big\|\B - \tilde{\B}_{t, c}^{\textrm{ens}} \big\|^2_F
&=& tc \Big[ \frac{t-1}{t}(1-\eta) \Big]^2 
    + \cdots + \big[(m-tc)^2-(m-tc)\big](\alpha-\eta)^2  \nonumber \\
&=& (1-\alpha) (1+\alpha-2\eta) (m - 2c + \frac{c}{t} ) +
    (\alpha - \eta)^2 \Big( 4c^2 - 4cm + m^2 + \frac{2cm - 3c^2}{t} \Big) \nonumber \\
&=& (1-\alpha)^2 \Big(m - 2c + \frac{c}{t}\Big)
    + \frac{(1-\alpha)^2}{(c + \frac{1-\alpha}{\alpha})^2}
    \Big[ (m-2c+\frac{c}{t})\big(\frac{2}{\alpha} - 2 + m\big) + \frac{c(m-c)}{t}\Big]  \nonumber \\
&\geq & (1-\alpha)^2 \Big(m - 2c + \frac{c}{t}\Big)
        \Big( 1+ \frac{m+ \frac{c}{t} + \frac{2}{\alpha} - 2}{(c + \frac{1-\alpha}{\alpha})^2}  \Big) \textrm{,} \nonumber
\end{eqnarray}
\end{small}
where the last inequality follows from $\frac{c(m-c)}{t}
= \frac{c}{t}\Big( (m-2c+\frac{c}{t}) + (c - \frac{c}{t}) \Big)
\geq \frac{c}{t} \Big(m-2c +\frac{c}{t}\Big)$.

Furthermore, since the matrices $\B - \C^{(i)} {\W}^\dag {\C^{(i)}}^T$ are all SPSD by Theorem~\ref{thm:nystrom_error},
so their sum is also SPSD.
Then the SPSD property of $(\B - \tilde{\B}_{t, c}^{\textrm{ens}})$ follows from (\ref{eq:lem:lower_bound_ensemble_additive:expand}).
Therefore, the nuclear norm of $(\B - \tilde{\B}_{t, c}^{\textrm{ens}})$ equals to the matrix trace, that is,
\begin{eqnarray}
\big\|\B - \tilde{\B}_{t, c}^{\textrm{ens}} \big\|_*
&=& \tr\big( \B - \tilde{\B}_{t, c}^{\textrm{ens}} \big) \nonumber \\
&=& t c \cdot \frac{t-1}{t} (1-\eta) + (m-t c) \cdot (1-\eta) \nonumber \\
&=& (1-\alpha) (m-c) \frac{c + \frac{1}{\alpha}}{c + \frac{1-\alpha}{\alpha}} \textrm{,} \nonumber
\end{eqnarray}
which proves the nuclear norm bound in the lemma.
\end{proof}

\begin{theorem} \label{thm:lower_bound_ensemble_additive}
Assume that the ensemble \nystrom method selects a collection of $t$ samples,
each sample ${\C^{(i)}}$ ($i=1, \cdots, t$) contains $c$ columns of $\A$ without overlapping.
For a the matrix $\A$ defined in (\ref{eq:construction_bad_nystrom_blk}),
the approximation error incurred by the ensemble \nystrom method is lower bounded by
\begin{eqnarray}
\big\|\A - \tilde{\A}_{t, c}^{\textrm{ens}} \big\|_F
&\geq & (1-\alpha) \sqrt{ \Big(m-2c + \frac{c}{t} -k\Big)  +
        k \bigg( \frac{ m - c + \frac{c}{t} + k\frac{1-\alpha}{\alpha} }{ c + k\frac{1-\alpha}{\alpha} } \bigg)^2} \textrm{,} \nonumber\\
\big\|\A - \tilde{\A}_{t, c}^{\textrm{ens}} \big\|_*
&\geq & (1-\alpha) (m-c) \frac{c + \frac{1}{\alpha}k}{c + \frac{1-\alpha}{\alpha}k}\textrm{,} \nonumber
\end{eqnarray}
where $\tilde{\A}_{t, c}^{\textrm{ens}} = \frac{1}{t} \sum_{i=1}^t \C^{(i)} {\W^{(i)}}^\dag {\C^{(i)}}^T$.
\end{theorem}

\begin{proof}
According to the construction of $\A$ in (\ref{eq:construction_bad_nystrom_blk}),
the $i$-th sample $\C^{(i)}$ is also block diagonal.
We denote it by $\C^{(i)} = \blkdiag\big(\hat{\C}^{(i)}_1 , \cdots , \hat{\C}^{(i)}_k \big)$.
Akin to (\ref{eq:thm:lower_bound_conventional:1}),
we have
\[
\tilde{\A}_{t, c}^{\textrm{ens}}
\; = \; \left[
          \begin{array}{ccc}
            \frac{1}{t} \sum_{i=1}^t \hat{\C}^{(i)}_1 \hat{\W}_1^\dag \big(\hat{\C}^{(i)}_1\big)^T &  & \0 \\
             & \ddots &  \\
            \0 &  & \frac{1}{t} \sum_{i=1}^t \hat{\C}^{(i)}_k \hat{\W}_k^\dag \big(\hat{\C}^{(i)}_k\big)^T\\
          \end{array}
        \right] \textrm{.}
\]
Thus the approximation error of the ensemble \nystrom method is
\begin{eqnarray}
\Big\|\A - \tilde{\A}_{t, c}^{\textrm{ens}} \Big\|_F^2
& = & \sum_{j=1}^k \Big\|\B - \frac{1}{t} \sum_{i=1}^t \hat{\C}_j^{(i)} \hat{\W}_j^\dag \big(\hat{\C}^{(i)}_j \big)^T \Big\|_F^2  \nonumber\\
&\geq & (1-\alpha)^2 \sum_{j=1}^k  \Big(p - 2{c_j} + \frac{{c_j}}{t}\Big)
        \Big( 1+ \frac{p+\frac{{c_j}}{t}+ \frac{2}{\alpha} - 2}{({c_j} + \frac{1-\alpha}{\alpha})^2}  \Big)  \nonumber \\
& = & (1-\alpha)^2 \bigg[ \Big(m-2c + \frac{c}{t}\Big)  + \sum_{j=1}^k \Big( p - 2 c_j + \frac{c_j}{t}  \Big)
               \frac{p + \frac{c_j}{t} + \frac{2(1-\alpha)}{\alpha}}{(c_j + \frac{1-\alpha}{\alpha})^2} \bigg]\textrm{,} \nonumber
\end{eqnarray}
where the inequality follows from Lemma~\ref{lem:lower_bound_ensemble_additive},
and the last equality follows from $\sum_{j=1}^k c_j = c$ and $k p = m$.
The summation in the last equality equals to
\begin{align}
&\!\!\!\!\!\!\!\!\!\!\!\!\!\!\!\!\!\!\!\!
    \sum_{j=1}^k \bigg[ \Big(p + \frac{c_j}{t} + \frac{2(1-\alpha)}{\alpha}\Big) - 2\Big( c_j + \frac{1-\alpha}{\alpha}\Big) \bigg]
               \frac{p + \frac{c_j}{t} + \frac{2(1-\alpha)}{\alpha}}{(c_j + \frac{1-\alpha}{\alpha})^2} \nonumber \\
&\qquad = \;  -k + \sum_{j=1}^k \bigg( \frac{p + \frac{c_j}{t} + \frac{2(1-\alpha)}{\alpha}}{c_j + \frac{1-\alpha}{\alpha}} -1 \bigg)^2 \nonumber \\
&\qquad \geq \; -k + k \bigg( \frac{ m - c + \frac{c}{t} + k\frac{1-\alpha}{\alpha} }{ c + k\frac{1-\alpha}{\alpha} } \bigg)^2 \textrm{.}\nonumber
\end{align}
Here the inequality holds because the function is minimized when $c_1 = \cdots = c_k = c/k$.
Finally we have that
\[
\Big\|\A - \tilde{\A}_{t, c}^{\textrm{ens}} \Big\|_F^2
\; \geq \;
(1-\alpha)^2 \bigg[ \Big(m-2c + \frac{c}{t} -k\Big)  +
        k \bigg( \frac{ m - c + \frac{c}{t} + k\frac{1-\alpha}{\alpha} }{ c + k\frac{1-\alpha}{\alpha} } \bigg)^2  \bigg]\textrm{,}
\]
which proves the Frobenius norm bound in the theorem.

Furthermore, since the matrix $\B - \frac{1}{t} \sum_{i=1}^t \hat{\C}_j^{(i)} \hat{\W}_j^\dag \big(\hat{\C}^{(i)}_j \big)^T$
is SPSD by Lemma~\ref{lem:lower_bound_ensemble_additive},
so the block diagonal matrix $(\A - \tilde{\A}_{t, c}^{\textrm{ens}})$ is also SPSD.
Thus we have
\begin{eqnarray}
\big\|\A - \tilde{\A}_{t, c}^{\textrm{ens}} \big\|_*
&=& (1-\alpha) \sum_{i=1} (p - c_i) \frac{c_i + \frac{1}{\alpha}}{c_i + \frac{1-\alpha}{\alpha}}
\;\geq \; (1-\alpha) (m-c) \, \Big(1+ \frac{k}{c + \frac{1-\alpha}{\alpha}k} \Big) \textrm{,} \nonumber
\end{eqnarray}
which proves the nuclear norm bound in the theorem.
\end{proof}

\begin{theorem} \label{thm:lower_bound_ensemble}
Assume that the ensemble \nystrom method selects a collection of $t$ samples,
each sample ${\C^{(i)}}$ ($i=1, \cdots, t$) contains $c$ columns of $\A$ without overlapping. Then there exists an $m{\times}m$
SPSD matrix $\A$ such that
the relative-error ratio of the ensemble \nystrom method is lower bounded by
\begin{eqnarray}
\frac{\| \A - \tilde{\A}_{t, c}^{\textrm{ens}} \|_F}{\| \A - \A_k \|_F}
& \geq & \sqrt{\frac{m - 2c + {c}/{t} - k}{m-k}
        \Big( 1+ \frac{k(m - 2c + c/t)}{c^2}  \Big)} \textrm{,} \nonumber \\
\frac{\| \A - \tilde{\A}_{t, c}^{\textrm{ens}} \|_*}{\| \A - \A_k \|_*}
& \geq & \frac{m-c}{m-k} \: \Big( 1 + \frac{k}{c}\Big) \textrm{,} \nonumber
\end{eqnarray}
where $\tilde{\A}_{t, c}^{\textrm{ens}} = \frac{1}{t} \sum_{i=1}^t \C^{(i)} {\W^{(i)}}^\dag {\C^{(i)}}^T$.
\end{theorem}

\begin{proof}
The theorem follows directly from Theorem~\ref{thm:lower_bound_ensemble_additive} and Lemma~\ref{lem:nystrom_residual_new}
by setting $\alpha \to 1$.
\end{proof}

\bibliography{cur}

\end{document}